\documentclass{article}

\PassOptionsToPackage{numbers, compress}{natbib}
\usepackage[eandd, final]{neurips_2026}

\usepackage[utf8]{inputenc} % allow utf-8 input
\usepackage[T1]{fontenc}    % use 8-bit T1 fonts
\usepackage{hyperref}       % hyperlinks
\usepackage{url}            % simple URL typesetting
\usepackage{booktabs}       % professional-quality tables
\usepackage{amsfonts}       % blackboard math symbols
\usepackage{nicefrac}       % compact symbols for 1/2, etc.
\usepackage{microtype}      % microtypography
\usepackage{xcolor} 
\usepackage{hyperref}
\title{\timewarp: Evaluating Web Agents by\\ Revisiting the Past}

\author{%
  Md Farhan Ishmam and Kenneth Marino \\
  University of Utah, USA\\
  \texttt{\{farhan.ishmam, kenneth.marino\}@utah.edu} \\\\
  \url{https://timewarp-web.github.io}
  \vspace{-0.5cm}
}

\usepackage{amsmath}
\usepackage{amssymb}
\usepackage{mathtools}
\usepackage{amsthm}
\usepackage{booktabs}
\usepackage{multirow}
\usepackage{graphicx}
\usepackage{etoc}
\usepackage{amsfonts}
\usepackage{xcolor}
\usepackage[table]{xcolor} 
\usepackage{hyperref}
\usepackage{pifont}
\usepackage{pgfplots}
\usepackage{listings}
\usepackage{tabularx}
\usepackage{algorithm}
\usepackage{algorithmic}
\usepackage{amsmath} % Required for \text{...} in math mode
\usepackage{float}

\pgfplotsset{compat=1.8}
\usepackage{caption}
\usepackage[capitalize,noabbrev]{cleveref}
\newcommand{\cmark}{\textcolor{green!60!black}{\ding{51}}}
\newcommand{\xmark}{\textcolor{red!70!black}{\ding{55}}}
\newcommand{\TF}[1]{\ifnum#1=1 \cmark \else \xmark \fi}

\newcommand{\farhan}[1]{}
\newcommand{\kenny}[1]{}

\newcolumntype{L}[1]{>{\raggedright\arraybackslash}p{#1}}
\newcolumntype{C}[1]{>{\centering\arraybackslash}p{#1}}

\newcommand{\mwrapCenter}[2]{%
  \multirow[c]{#1}{*}{\parbox{\linewidth}{\centering #2}}%
}

\theoremstyle{plain}
\newtheorem{theorem}{Theorem}[section]
\newtheorem{proposition}[theorem]{Proposition}

\theoremstyle{definition}
\newtheorem{definition}[theorem]{Definition}
\newtheorem{assumption}[theorem]{Assumption}
\theoremstyle{remark}
\newtheorem{remark}[theorem]{Remark}

\usepackage[disable,textsize=tiny]{todonotes}
\newcommand{\E}{\mathbb{E}}

\newcommand{\timewarp}{\textsc{TimeWarp}}
\newcommand{\timetraj}{\textsc{TimeTraj}}

\newcommand{\html}{{HTML}}
\newcommand{\axt}{{AXT}}
\newcommand{\som}{{SoM}}
\newcommand{\sshot}{{SS}}
\newcommand{\bc}{{BC}}
\newcommand{\bct}{{BCTv$_{1:6}$}}
\newcommand{\tw}{{TW}}

\newcommand{\mwrap}[2]{%
  \multirow[c]{#1}{*}{\parbox{\linewidth}{\raggedright #2}}%
}

\providecommand{\sectionvspace}{\vspace{0cm}}

\usetikzlibrary{patterns,patterns.meta}

\usepackage[most]{tcolorbox}
\tcbuselibrary{breakable}

\usepackage{microtype}
\usepackage{graphicx}
\usepackage{subfigure}
\usepackage{booktabs}
\usepackage{hyperref}
\usepackage{diagbox}
\usepackage{makecell}

\usepackage{amsfonts}
\usepackage{xcolor}
\usepackage[table]{xcolor} 
\usepackage{hyperref}
\usepackage{pifont}
\usepackage{pgfplots}
\usepackage{listings}

\pgfplotsset{compat=1.8}
\usepackage{etoc}
\usepackage{titletoc}
\newcommand\DoToC{%
  \startcontents
  \printcontents{}{1}{\noindent \textbf{\Large{Table of Contents in Appendix}}\vskip3pt\vskip5pt}
  \vskip3pt\vskip5pt
}
\begin{document}

\maketitle

% [STORY] Web agents are benchmarked on static sites but deployed on a web that keeps changing. TimeWarp rebuilds three sites in six historical UI versions to measure robustness to change in a controlled way, and TimeTraj plus TimeWarp-BC turn one human-refined plan per goal into multi-version training data that roughly doubles success and generalizes to unseen versions.
\begin{abstract}
%% Maybe we change the first line a bit?
% [REFINE: plain] "it raises the question" has no antecedent for "it".
% ORIG: As web agents close the gap with humans on benchmarks, it raises the question: \textit{Do today's agents perform just as well on tomorrow's web?}
{As web agents close the gap with humans on benchmarks, we ask: \textit{Do today's agents perform just as well on tomorrow's web?}}
We introduce \timewarp, a benchmark that emulates the evolving web. 
\timewarp\ consists of three web environments, each with six UI versions that differ in UI design, frontend code, and workflows from different eras of the internet.
We pair \timewarp\ with a set of complex, realistic tasks covering different forms of web navigation. 
Our experiments reveal that vision-based agents are vulnerable to changes, while text-based agents become brittle once fine-tuned on a single version.
To address this, we propose \timetraj, a new annotation method that uses plan distillation to collect trajectories across multiple versions. By training agents on teacher rollouts using our BC-variant, we achieve substantial performance gains: $20.4\%\rightarrow37.7\%$ for Qwen-3 4B and $0.7\%\rightarrow27.0\%$ for Llama-3.1 8B models.
Our work helps study generalization across web designs and opens a new paradigm for collecting plans rather than trajectories to improve the robustness of web agents. % Chop off the orphan line
\end{abstract}
\vspace{-12pt}
\begin{figure}[ht]
    \centering
    \includegraphics[width=\linewidth]{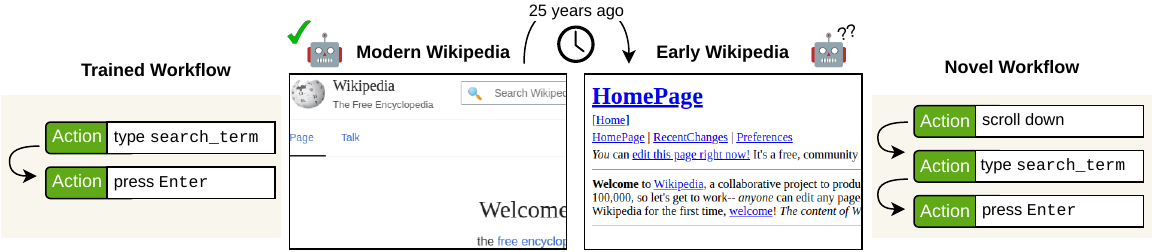}
    \caption{\textbf{Agents \textit{vs.} the Dynamic Web.} Websites change visually and functionally over time, resulting in workflow changes. Web agents can be vulnerable to novel workflows.}
    \label{fig:teaser}
\end{figure}
\vspace{-8pt}
\section{Introduction}
\label{sec:introduction}
\vspace{-8pt}
\sectionvspace

% \kenny{Why, webs change over time. We intro env. Shows that pertrained models vary perf based on version, esp. vision.  
% When training, overfits to version. We develop techinique where if you do this over multiple version (e.g. do as live web changes) then improves robustness with minimal additional labeling cost. Vanilla BC fails, so we introduce auto method which automatically collects for new versions. Also talk about how tasks are hard and need memory and planning and we also introduce this into BC training. Show that being mindful of changing web can make more robust agents. Also hit how this is primarily looking at how these agents are trained so this is why we didn't benchmark the closed models}

\noindent

\begin{figure*}[t]
    \centering

    % ---- FIGURE ----
    \begin{minipage}{\linewidth}
        \centering
        \includegraphics[width=\linewidth]{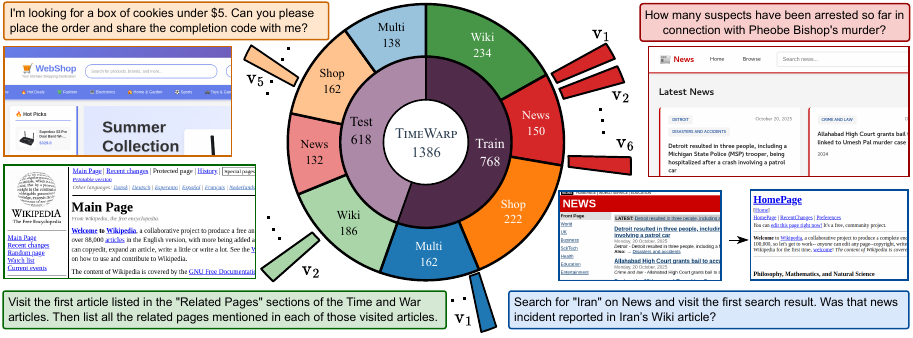}
    \end{minipage}

    % \vspace{0.2cm}

    % ---- TABLE ----
    \begin{minipage}{\linewidth}
        \centering
        \small
        \begin{tabular}{lccccccc}

\toprule

\textbf{Env} & \textbf{Reference Sites} & \textbf{v$_1$} & \textbf{v$_2$} & \textbf{v$_3$} & \textbf{v$_4$} & \textbf{v$_5$} & \textbf{v$_6$} \\

\midrule

\cellcolor{teal!08}\textbf{Wiki} & Wikipedia & 2001 & 2002-03 & 2003-04 & 2005-22 & 2023-25 & Minimal \\

\cellcolor{red!08}\textbf{News} & BBC News & 1998-01 & 2002-07 & 2008-15 & 2016-22 & 2023-25 & Minimal \\

\cellcolor{orange!13}\textbf{Shop} & Amazon (Az), Temu & Az'99-04 & Az'05-11 & Az'12-14 & Az'15-25 & Temu'25 & WebShop \\

\bottomrule

\end{tabular}
    \end{minipage}
    \caption{\textbf{Overview of \timewarp}: environments, versions, version to year mapping, UI examples, task distribution, and goal examples. The benchmark contains 231 goals $\times$ 6 versions $=$ 1386 tasks.}
    \vspace{-0.35cm}
    \label{fig:datasetStatistics}
\end{figure*}

Today's web agents are rapidly improving on benchmarks \cite{Song2025BEARCUBSABU}, with researchers showing that fine-tuning vastly outperforms training-free methods for open-source models~\cite{qi2024webrl,wei-etal-2025-webagent,vattikonda2025how}. But these gains may not persist when the web changes. The web is dynamic and ever-changing by nature, with its UI, layout, and content evolving rapidly as user preferences shift and technological capabilities develop. Will the web agents we train today work just as well when the web changes?

Existing benchmarks are not designed to study this problem, making it difficult to evaluate how web agents generalize to changes.
%However, existing web benchmarks fall short in evaluating an agent's generalization capabilities and vulnerability to web changes. 
Generally, these benchmarks feature either simulated environments \cite{zhou2023webarena,yao2022webshop, Garg2025REALBAA} or the live web \cite{Song2025BEARCUBSABU,He2024WebVoyagerBAAF,pan2024webcanvas}. Simulated environments have stable behavior and prioritize reproducibility, but do not capture the dynamic nature of the internet, often leading to poor generalization \cite{li2024websuite}. Live web benchmarks partially address this as they, by necessity, expose agents to the real, dynamic web. However, the live web can change unpredictably and cannot be studied in a controlled setting.
%Today's web agents are achieving unprecedented numbers in modern benchmarks \cite{Song2025BEARCUBSABU}, strengthening the case for integrating them into real human workflows. But how well do these benchmark gains translate into reliable real-world performance? A good indicator here is the web agent's ability to generalize across website variations. Given that the internet is dynamic, with UI, layout, and content changing over time, generalization across these changes should reflect an agent's robustness during deployment.

%They are often set up in a way that makes it difficult to adapt them to study these kinds of changes. 

%partially address this, but are prone to anti-bot measures, such as CAPTCHA, and have poor reproducibility due to unpredictable content modifications \cite{Garg2025REALBAA}. 
% Trained web agents are vulnerable to data distribution shifts
We introduce \timewarp, a new web environment that provides containerized UI versions across different eras of the internet
for multiple websites (Fig. \ref{fig:datasetStatistics}, \S\ref{sec:timewarpBenchmark}). \timewarp\ enables evaluation of web changes in a \textit{controlled} and \textit{reproducible} manner. The environment is complemented by a set of realistic and complex tasks spanning multiple categories to evaluate different forms of web navigation (\S\ref{sec:dataset}). 
Each task goal in our dataset can be tested on different versions of the web that simulate changes over time.
\timewarp\ is tightly integrated with the BrowserGym ecosystem \cite{tmlr25dechezellesBrowserGym} and provides easy expandability of its environments, versions, and tasks (\S\ref{subsec:timeWarpExpandability}).

Armed with this new benchmark, we experiment with open-source agents to evaluate how performance varies across versions (\S\ref{sec:Results}). % [REFINE: plain] Parallel clauses ("to be", "failing", "unable") made parallel.
% ORIG: We find vision-based agents to be more vulnerable to web changes, standard behavior cloning failing on our complex and realistic tasks, and trained agents unable to generalize to new versions \& benchmarks.
We find that vision-based agents are more vulnerable to web changes, standard behavior cloning fails on our complex and realistic tasks, and trained agents do not generalize to new versions and benchmarks.
While training across multiple versions of the environment can be a potential solution, trajectories cannot be reused across versions, and collecting human trajectories across these versions is time-consuming and resource-intensive. A more abstract representation, such as a plan, is more robust to temporal variations than trajectories. Using this principle, we develop \timetraj, a method for scalably collecting trajectories across multiple versions (\S\ref{sec:trajectoryCollectionTimeTraj}). Humans refine a high-level, checkpoint-based execution plan for a single version, which a teacher agent uses to collect low-level, version-specific trajectories with no additional human involvement. 

{Using \timetraj, we also obtain non-action tokens from the teacher agent, such as thinking, planning, and memory tokens, which human annotation cannot provide. This enables a new variant of behavior cloning, \timewarp-BC (\S\ref{sec:timewarpBC}), that trains web agents on these non-action tokens. It not only improves performance in \timewarp, but also generalizes better to unseen versions and benchmarks. Our findings suggest that by collecting plans once, researchers can increase both the amount of training data and the robustness of their agents. They can either create multiple versions of their websites or wait for websites to evolve naturally, and then scalably collect new trajectories (\S\ref{subsec:broaderImpactAnnotation}).}

The major contributions of our work include:
% \vspace{-0.4cm}
\begin{itemize}
    \item \timewarp: A web environment and dataset containing tasks spanning multiple versions, evaluating how agentic performance changes as websites develop visually and functionally.
    % \vspace{-0.2cm}
    % [REFINE: plain] "especially using visual observations or fine-tuned" lacked a noun; made explicit (matches the abstract).
    % ORIG: \item Empirical findings on the vulnerabilities of web agents against changes (especially using visual observations or fine-tuned), and generalizability across versions and benchmarks.
    \item {Empirical findings on the vulnerability of web agents to changes (especially agents using visual observations or fine-tuned on a single version), and on generalization across versions and benchmarks.}
    \item \timetraj: A method for scalably collecting training trajectories on new versions of the web without requiring additional human involvement.
    % \vspace{-0.2cm}
    % ORIG: \item \timewarp-BC: An improved behavior cloning for web agents that uses \timetraj~to train agents better on tasks requiring memory and planning.
    \item {\timewarp-BC: An improved behavior cloning method for web agents that uses \timetraj\ to train agents more effectively on tasks requiring memory and planning.}
\end{itemize}
\section{Related Work}
\label{sec:relatedWork}
% \sectionvspace
\textbf{Web Environments} fall broadly into two categories: simulated environments and live web \cite{sager2025comprehensive}. Simulated environments, such as WebShop \cite{yao2022webshop}, WebArena \cite{zhou2023webarena}, and REAL \cite{Garg2025REALBAA}, use a limited number of containerized websites that replicate the real ones. While reproducible, they fail to capture the internet’s dynamic complexity. Live web benchmarks, \textit{e.g.}, WebVoyager \cite{He2024WebVoyagerBAAF}, WebCanvas \cite{pan2024webcanvas}, and Online-Mind2Web \cite{Xue2025AnIOAH}, evaluate agents directly on the real internet, but cannot support systematic study of web evolution due to slow, uncontrolled, real-world change. We take a hybrid approach by using containerized versions of websites to emulate dynamic web behavior and study agents' responses to change in real time. 
To our knowledge, no environment has created multiple versions of the same websites/tasks to simulate the changing nature of the web. This makes \timewarp~unique and useful for studying robustness to the changing web.

\textbf{Web Agent Training} provides substantial performance gains to open-source models, with the best results achieved by LLMs aligned via RL, \textit{e.g.}, PPO \cite{qi2024webrl}, or GRPO \cite{wei-etal-2025-webagent}. These alignment methods, however, share a common behavior cloning (BC) phase, deemed a prerequisite to any form of web agent training \cite{wei-etal-2025-webagent}. BC is generally applied to the action tokens, with recent works finding further gains by including thinking tokens \cite{hu-etal-2025-webcot}. % ORIG: As task complexity increased, additional tokens for planning and memory have been introduced at inference time \cite{Drouin2024WorkArenaHCAW}.
{As task complexity has increased, additional tokens for planning and memory have been introduced at inference time \cite{Drouin2024WorkArenaHCAW}.} We build on these works and include the action, thinking, planning, and memory tokens from a teacher model during the BC phase. 
%We do not perform alignment, as these are orthogonal to our work and can be applied in parallel.

\textbf{Trajectory Collection} for behavior cloning is generally performed by humans \cite{yao2022webshop}, with recent works shifting towards automation \cite{pahuja-etal-2025-explorer, xu2024agenttrek, trabucco2025insta}. However, automatically generated trajectories are relatively less complex and not specific to the training tasks of a dataset (Tab. {\ref{tab:trajectoryCollectionComparison}}). Given the complexity of \timewarp's task, we simply require trajectory generation for the training split, but acknowledge the need for some form of automation, as collecting trajectories across multiple versions is tedious and resource-intensive. % [REFINE: plain] "making humans refine ... and use" was not parallel.
% ORIG: Our proposed \timetraj\ takes a hybrid approach by making humans refine checkpoint-based execution plans on a single version, and use an executor/teacher web agent to collect trajectories across multiple versions.
{Our proposed \timetraj\ takes a hybrid approach. Humans refine checkpoint-based execution plans on a single version, and an executor (teacher) web agent then collects trajectories across multiple versions.}

\section{Problem Formulation}
\label{sec:problem}
% \vspace{-8pt}
\sectionvspace
\timewarp\ consists of multiple versions $\{\mathcal{E}_1,\dots,\mathcal{E}_n\}$, where each version $\mathcal{E}_v$ is formalized as a POMDP \cite{spaan2012partially} %, kurniawati2022partially} 
and can be represented as a 7-tuple $\langle\mathcal{S},\mathcal{A},\mathcal{O}, \mathbb{T}, \mathbb{T}_0,\mathbb{O},R\rangle$. %%%%%%%%%% This representation is literally from Wikipedia, if reviewers ask about it, simply cite Wikipedia.

\noindent\textbf{State and Observation.} $s\in\mathcal{S}$ denotes the underlying, unobserved state of the environment.
%with $s\in\mathcal{S}$ being the underlying, unobserved state of the environment. 
% [REFINE: symbol] $o\in\mathcal{O}$ was typeset twice in the same sentence (a leftover of the commented-out clause).
% ORIG: $o\in\mathcal{O}$ is the observation received by the agent, $o\in\mathcal{O}$, in the form of HTML, accessibility trees (AXT), UI screenshots, or UI set-of-marks \cite{yang2023setofMarks}.
{$o\in\mathcal{O}$ is the observation received by the agent, in the form of HTML, accessibility trees (AXT), UI screenshots, or UI set-of-marks \cite{yang2023setofMarks}.}
$\mathbb{O}:\mathcal{S}\rightarrow\mathcal{O}$ is the observation function that maps states to observations.
%underlying state to the rendered observation. 
$\mathbb{T}:\mathcal{S} \times \mathcal{A}\rightarrow{}\mathcal{S}$ is the state transition, which can be considered deterministic.
%represents the state transition function, which can be treated as deterministic for modeling purposes. 
Finally, the initial state distribution $\mathbb{T}_0$ is uniform over all \timewarp\ tasks.

\noindent\textbf{Action.} The action space $\mathcal{A}$ consists of browser-level operations, and is a subset of BrowserGym's \cite{tmlr25dechezellesBrowserGym} action space. \timewarp’s action space (\S\ref{sec:actionSpace}) includes general actions (\textit{e.g.}, $\texttt{click}$), web navigation (\textit{e.g.}, $\texttt{goto}$), tab operations (\textit{e.g.}, $\texttt{tab\_focus}$), and user interactions (\textit{e.g.}, $\texttt{send\_msg\_to\_user}$). At each step $t$, the generated action $a_t \in \mathcal{A}$ is passed to the parser for execution.

\noindent\textbf{Reward.} \timewarp\ uses a sparse terminal reward $R\in\{0=\text{Failure},1=\text{Success}\}$, which is assigned only at the final step by the evaluator (\S\ref{sec:evaluation}).

\noindent\textbf{Agent.} Modern web agents generally follow an autoregressive policy $\pi_\theta$ \cite{christianos2023pangu} that conditions on the interaction history $h_t=(o_{1:t}, y_{1:t-1})\in\mathcal{H}$, \textit{i.e.}, observations and responses prior to time step $t$, and the current observation at step $t$. Since conditioning on multiple high-dimensional observations can exceed an agent's context window or confuse it \cite{He2024WebVoyagerBAAF}, the history is approximated using the most recent observation and previous responses \cite{putta2024agent}. The response tokens are generated at step $t$ as $y_t\sim \pi_\theta(\cdot\mid h_t),$
from which the browser-level action is deterministically parsed and executed. 
%Formally, an optimal policy in a partially observable environment may condition on the complete observation history.

% the current observation $o_t$ and all previously generated outputs, $y_{1:t-1}$, up to time step $t$.

%%%%%%% Response history is part of the observation, only mentioned about passing the observation and the response which the polcy maps
% In theory, the complete observation history $o_{1:t}$ should be passed, but because it would overload the agent's context, we only pass the current observation along with the response history. 

% response. At each step, the agent produces a structured 4-tuple $(a_t, p_t, c_t, m_t)$ comprising an environment action $a_t$, a plan $p_t$, an internal reasoning trace $c_t$, and a memory state $m_t$. The agent’s input at time $t$ is the observation $o_t$ together with the interaction history $(o_{1:t}, r_{1:t-1})$, where $r_t$ denotes the emitted 4-tuple. A trajectory is thus $\tau={(o_t, r_t)}_{t=1}^T$, with task horizon $T$. We consider a sparse terminal reward: the agent receives $R(\tau)\in{0,1}$ only at the final step $t=T$, indicating task failure or success, and $0$ otherwise.

% $\mathcal{O}$, and $\mathcal{A}$, represent the state, observation, and action spaces respectively, $\mathcal{T}$ denotes the t

% A trajectory (episode) is a sequence:
% \begin{equation}
%     \tau = (s_1, a_1, s_2, a_2, \dots, s_T, a_T),
% \end{equation}
% generated by an unknown expert policy interacting with the environment, which is typically a human. 

\sectionvspace
% \vspace{-8pt}
\section{The \timewarp\ Benchmark}
\label{sec:timewarpBenchmark}
% \vspace{-8pt}
\sectionvspace
The goal of \timewarp\ is to evaluate agents against changes to the web. To simulate these changes in a controlled and reproducible setting, we introduce multiple versions with varying front-end content across three containerized web environments. At its core, \timewarp\ is designed to be \textit{modular} and \textit{expandable}, achieved by adopting a lightweight Python-based Flask backend with modular frontends, each corresponding to a unique web version. Frontends can be easily added, swapped, or modified to create new versions and websites with minimal overhead.

Our design contrasts with previous web environments that used dockerized deployments \cite{zhou2023webarena,Garg2025REALBAA}, where individual websites require substantial storage\footnote{ \textit{e.g.}, WebArena's Wikipedia takes 89 GB disk storage \cite{webarena-environment-docker-2023}}, and maintaining multiple versions incurs high additional cost. \timewarp\ is tightly integrated with the BrowserGym framework \cite{tmlr25dechezellesBrowserGym}, enabling standardized agent interaction via Playwright-based browser actions and parallel execution of tasks.
%%%%%%%%%%% This is the main draft version

%% The goal behind creating timewarp is to benchmark agents against changes in the web. To simulate change without relying on the live web, we introduce six versions of varying frontend content across three new web environments.
%At its core, our benchmark aims to be \textit{accessible} and \textit{expandible}. We achieve this by, first, creating a lightweight system that uses a simple environment structure with a Python-based \textit{flask}\footnote{https://flask.palletsprojects.com/en/stable/} backend, with each frontend serving as a varying version. Frontends can be added, swapped, or modified, thereby easily creating new versions of the environment. The environment is similarly expandable, with newer websites easily being added to the Flask backend. (Maybe refer to a section in broader impact to tease miniweb) This is in sharp contrast to current benchmarks that used dockarized containers \cite{youKKnowHow}, one such container can easily take 100+ GB, with multiple versions of such taking many times of that. Secondly, the benchmark is tightly integrating the whole benchmark with the popular BrowserGym framework \cite{tmlr25dechezellesBrowserGym}, enabling easy and standardized agent interaction through Playwright-based browser automation.
\vspace{-4pt}
\sectionvspace
\subsection{Web Environments}
\vspace{-4pt}
\sectionvspace
\label{sec:webEnvironment}
We select three web environments: \textbf{Wiki}, \textbf{News}, and \textbf{Shop} to span a broad range of tasks, with each environment serving a distinct purpose and allowing evaluation of different aspects of web interaction (Fig. \ref{fig:datasetStatistics}). Each environment (and in some cases, its versions) features a unique search mechanism, \textit{e.g.}, Wiki's stricter string-based search aligns better with realistic encyclopedic searches, while News' flexible content-based search aligns with article searching (\S\ref{sec:searchAlgorithms}).
Searching successfully without backtracking is unlikely and encourages agents to adapt to each website’s search behavior. Below, we present a concise overview of each environment, with additional details provided in \S\ref{sec:envDetails}.

\noindent\textbf{Wiki} is a Wikipedia-style environment constructed from articles in the SimpleWiki dump \cite{WikimediaFoundationSimpleWiki}. It introduces encyclopedic content-retrieval tasks that evaluate an agent’s ability to search and navigate large bodies of information, while also serving as a knowledge base for other tasks. These tasks require reasoning over article sections, traversing linked pages, and synthesizing information across multiple articles, while adapting to a stricter search mechanism.

\noindent\textbf{News} is a BBC-style environment constructed from articles in the EnWikiNews dump \cite{WikimediaFoundationEnWikinews}. It introduces time-sensitive information retrieval tasks that evaluate an agent’s ability to prioritize and retrieve relevant content from multiple search results. Unlike Wiki, News tasks require agents to reason over metadata, \textit{e.g.}, the article title and publication date. The environment uses a relevance-based search algorithm based on term frequency, which resembles real-world news search. While this algorithm is more flexible than Wiki’s, imprecise search queries can produce many results, making it harder to retrieve the most relevant articles.

\noindent\textbf{Shop}'s backend is adapted from the WebShop environment \cite{yao2022webshop}, which contains $\sim1.18$ million real-world products scraped from Amazon, each with multiple variants. We modified the item-ordering mechanism so that a unique confirmation code is generated after the order is completed.
Shop tasks require agents to make decisions based on product attributes and features, often involving visual queries that necessitate inspecting product images. Shop's search mechanism is similar to News, resembling that of realistic e-commerce sites.
% We modify the item ordering mechanism such that, upon order completion, the environment generates a confirmation code unique to the selected product and its variants. This design enables easy verification of shopping tasks with minimal evaluation overhead. 
\vspace{-4pt}
\subsection{Environment Versions}
\label{sec:version}
% \vspace{-4pt}
\textbf{Version Selection.} For each environment, we select a set of six frontend versions $\{\text{v}_1,\cdots,\text{v}_6\}$, corresponding to distinct eras of web UI design (Fig.~\ref{fig:datasetStatistics}).
To identify representative design eras, we draw inspiration from the historical web page snapshots in the Web Design Museum \cite{WebDesignMuseum}.
After identifying these eras, we use webpage snapshots from the Internet Archive \cite{InternetArchive} to recreate version interfaces appropriate for that era. The associated front-end code is designed to be representative of each era, \textit{e.g.}, older websites use earlier HTML versions, such as 4.01 Transitional, and modern websites use HTML 5.0, thereby {also introducing} temporal variations in the HTML/AXT observation spaces. Workflows can vary across search behavior (Fig.~\ref{fig:searchDifferences}), interaction workflows (Fig.~\ref{fig:newsSearch}), dynamic page behavior (Fig.~\ref{fig:popUpAdsErrorQualAnalysis}), feature availability (Fig.~\ref{fig:tableofcontent}), and the structure of the observation space itself. Tasks and data are held fixed across versions, so that performance shifts are attributable solely to the version change. We provide additional details (\S\ref{sec:versionCreation}, \ref{sec:variationsVersions},\ref{sec:versionValidation}) and UI images for each version (\S\ref{sec:versionExamples}).
%The v$_1$ and v$_2$ resembles the early stages of the internet from late 90s to early 2000s. The websites have dated UI that is often harder to navigate visually (\textit{e.g.}, Wiki search is located at the bottom (\red{refer to an appendix pic}), articles lack table of contents). However, the HTML and AXTree are can be easier to navigate due to the simplicity. v$_3$ and v$_4$ are websites are mostly from the middle 2000s to the 2010s. Websites are visually easier to navigate stylistically, one common placements of search bars and easy navigation buttons (\textit{e.g.}, table of contents). v$_5$ shows the modern websites that are visually more complex often for aesthetic purposes, \textit{e.g.}, search icons instead of text box (news). The HTML and accessibility is also more complex and websites also have modern bloatware such as popup ads (refer shop v5). Modern websites have several quality of life improvements, for instance, wiki v5 and v6 have substring match results shown while typing, but the other versions don't feature this capability.

\begin{figure*}[ht]
    \centering
    \includegraphics[width=0.92\linewidth]{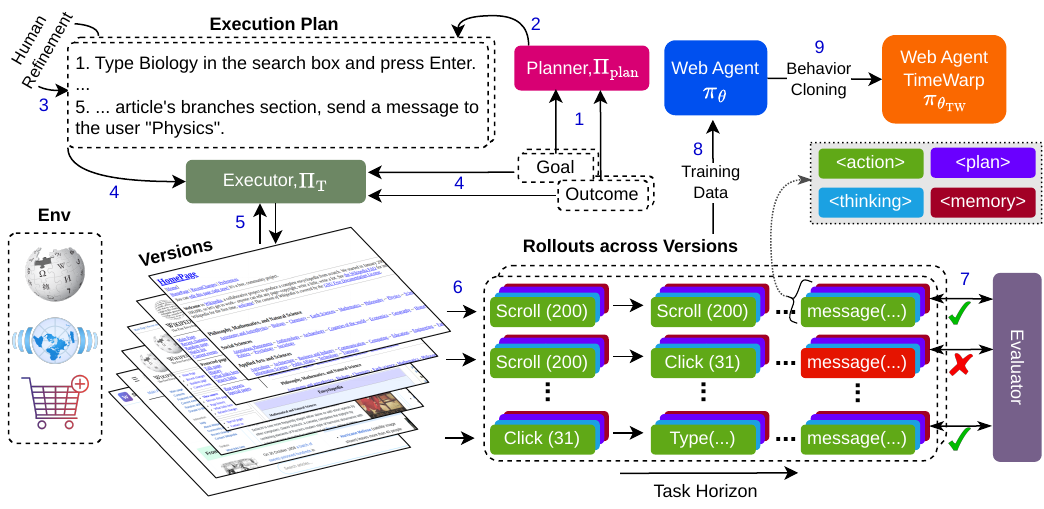}
    \vspace{-0.1cm}
    \caption{\textbf{Overview of \timetraj~and \textsc{TimeWarp-BC}}: (1) The goals and desired outcomes of the tasks are passed to the planner, $\Pi_\text{plan}$, which (2) produces draft execution plans. (3) The plans are refined by humans and (4) passed to the executor (teacher), $\Pi_\text{T}$, along with the goals and desired outcomes. (5) The executor generates trajectories across different versions of the \timewarp\ environments: Wiki, News, and Shop. (6) Rollouts of the trajectories consist of observations and a 4-tuple of response tokens ($\langle$action$\rangle$, $\langle$thinking$\rangle$, $\langle$memory$\rangle$, $\langle$plan$\rangle$), at each time step. (7) Trajectories are evaluated by the evaluator, (8) the filtered rollouts across versions form the training data for a web agent $\pi_{\theta}$, (9) which uses behavior cloning to produce a \timewarp\ agent, $\pi_{\theta_{\text{TW}}}$.}
    \vspace{-0.5cm}
    \label{fig:timewarpMain} 
\end{figure*}

\textbf{Older Versions} v$_1$ and v$_2$ showcase the early era of the internet with dated interfaces that are often visually harder to navigate, \textit{e.g.}, search boxes placed at the bottom of the page (Fig. \ref{fig:searchDifferences}a), and long articles without a table of contents (Fig. \ref{fig:tableofcontent}a, \ref{fig:tableofcontent}b). However, the underlying HTML structure and AXT tend to be simpler, which can often make web navigation for text-based web agents easier. % [REFINE: plain] Literal phrase instead of the "Middle Ages" metaphor.
% ORIG: Versions v$_3$ and v$_4$ are mostly from the internet's Middle Ages.
{Versions v$_3$ and v$_4$ mostly cover the intermediate era between the early and the modern web.} These interfaces are visually more organized, with consistent placement of UI elements, \textit{e.g.}, top nav bars, and table of contents (similar to Fig. \ref{fig:tableofcontent}c, \ref{fig:tableofcontent}d), improving usability for both humans and visual agents.

\textbf{Modern Versions.} v$_5$ showcases today's web design, where websites are visually more complex and often prioritize aesthetics through visual elements, such as icon-based search controls (Fig \ref{fig:newsSearch}). The HTML and AXT are also structurally more complex and may include dynamic elements and bloat, such as pop-up advertisements (Fig. \ref{fig:popUpAdsErrorQualAnalysis}). These websites also offer quality-of-life improvements, \textit{e.g.,} drop-down suggestions while typing in the search box (similar to Fig. \ref{fig:searchDifferences}b), which were absent in earlier versions. Modern LLMs and VLM-based web agents are also more likely to have been trained on websites with similar UI. v$_6$ is unique and represents a clean, minimal interface for the Wiki and News environments, while the Shop v$_6$ uses the base UI from the WebShop environment \cite{yao2022webshop}.

% Frontends are recreated using resources from WebArenaMuseum, with each environment version adopting UI structures and accessibility representations characteristic of its era.
% In particular, each version exposes HTML content and corresponding accessibility (AX) trees that reflect historical design conventions, enabling agents to interact with interfaces that vary meaningfully across time.

 % Each versions of the webenvironment has its unique frontend versions that is representative of an era of the internet's design. The associated front end code that is also representative of that era's code, for instance an older version will use an earlier version of such as HTML HTML 4.01 Transitional, and a modern version will use HTML 5.0 or modern version DOuble check this information.  

% For each environment, our goal is to select a set of versions representing a particular era of the internet's UI design (Tab. \ref{tab:theme_mapping}). To identify the historical eras, we rely on the webpage snapshots from Web Design Museum\footnote{\url{https://www.webdesignmuseum.org/}}. After identifying the eras, we refer to the Internet Archive\footnote{\url{https://archive.org/}} to view snapshots of the different web pages for reproducing them.

% The forntend are recreated from WebArenamuseem. Each environment versions uses the HTML and AX Tree of their own era. For historial accuracy, the versions have been taken from Web Design Museum \cite{WebDesignMuseum_wikipedia_search}

\vspace{-4pt}
\subsection{Dataset Tasks and Coverage}
\label{sec:dataset}
\sectionvspace

The \timewarp\ benchmark comprises 1,386 tasks (231 goals × 6 versions) split into training and test sets (Fig. \ref{fig:datasetStatistics}) and spans a wide range of categories (Tab. \ref{tab:datasetTasks}). Formally, a task is a pair $(g, \mathcal{E}_\text{v})$ of a version-invariant goal $g$, together with its reference answer and refined plan, and an environment version $\mathcal{E}_\text{v}$.
The task goals are similar in nature to other popular benchmarks, such as AssistantBench \cite{yoran2407assistantbench}, and BrowseComp \cite{wei2025browsecomp}. All tasks are manually written by authors and verified externally. Tasks are \textit{not} template-based, as such approaches tend to introduce redundant and incremental variations. Our dataset includes goals that require both visual and textual modalities (\S\ref{sec:visualInfoAnnotation}).

% For simplicity, we avoid evaluation of internal state changes \cite{zhou2023webarena}, which helps expandability of the benchmark in adding tasks. 

% -> Shop tasks are new and not from WebShop. We also have multi-evenironmentts tasks

% Talk about the timewarp tasks being really complex and require reasoning, decomposition of planning, 

% \textit{[Placeholder: Insert details regarding the dataset size (e.g., number of trajectories), task diversity (informational vs. transactional), and specific breakdown of tasks extracted from \texttt{test.raw.json}.]}

\vspace{-4pt}
\subsection{Evaluation}
\label{sec:evaluation}
\sectionvspace
% For consistent evaluation 
We evaluate trajectory success using deterministic verifiers, which provide a binary reward $R\in\{0,1\}$. Similar to the evaluation functions of WebArena \cite{zhou2023webarena}, each task specifies a verifier composed of function classes. The verifiers reach a 99.87\% agreement rate with the GPT-5.1 judge on the teacher trajectories, while being cheap, fully reproducible, and mostly free of LLM dependence (\S\ref{sec:EvaluationDetails}).

\vspace{-4pt}
\section{Methodology}
\label{sec:methodology}
\vspace{-4pt}

\sectionvspace
\subsection{\timetraj}
\label{sec:trajectoryCollectionTimeTraj}
\sectionvspace

\timetraj\ collects trajectories across different versions of \timewarp's environment using a planner $\Pi_\text{plan}$, and a teacher executor $\Pi_T$ module to (i) generate draft execution plans that are refined by humans (\S\ref{sec:planDistillationMethod}), and (ii) produce the teacher rollouts across versions (\S\ref{sec:teacherRolloutsMethod}). The prompt for each module is reported in \S\ref{sec:prompt}, with additional details on \timetraj\ provided in the Appendix \S\ref{sec:methodologyDetails}.
\sectionvspace

\begin{figure}[ht]
    \centering
    \includegraphics[width=\linewidth]{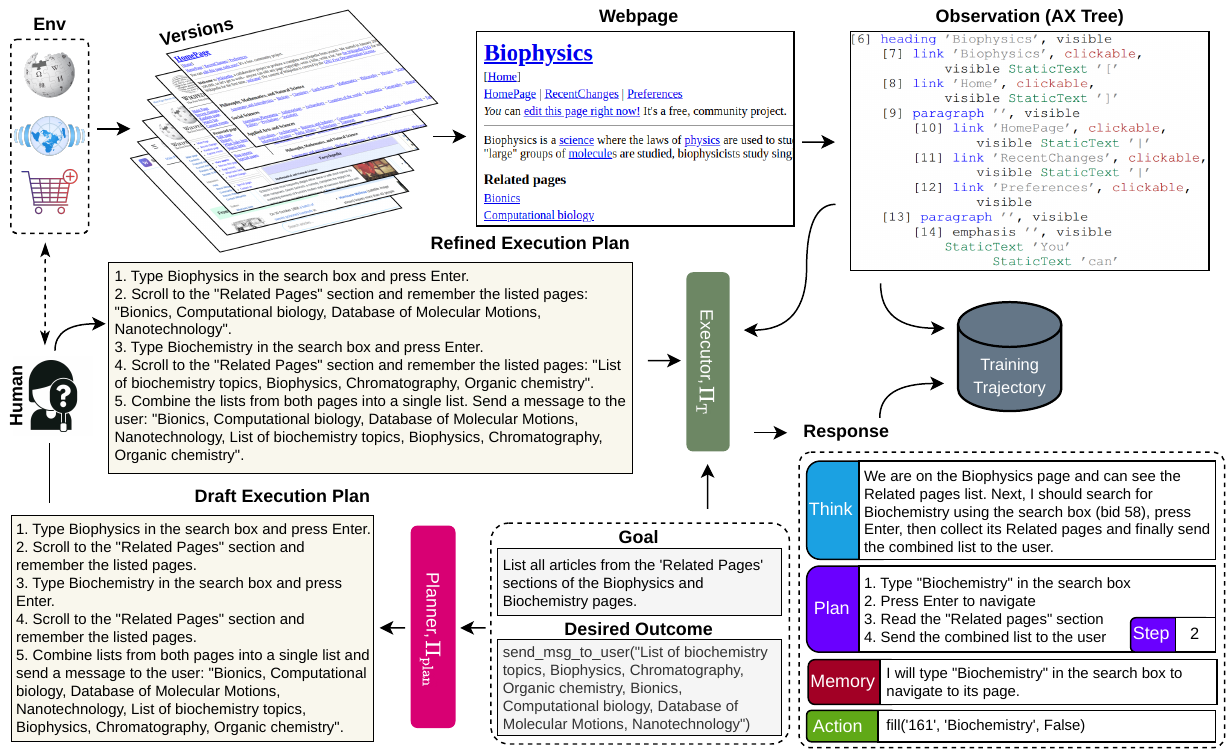}
    \caption{Execution of a single step of a positive (accepted) trajectory collected by \timetraj.}
    \label{fig:trajectoryexpanded}
    \vspace{-10pt}
\end{figure}

\subsubsection{Human-in-the-Loop Plan Distillation}
\label{sec:planDistillationMethod}
\sectionvspace

\timewarp's goal dataset, $\mathcal{D}_\text{goal}$, consists of human-annotated and version-invariant (i) task goals $g$, which are natural language queries describing what the agent needs to do, and (ii) their desired outcomes $a$, which are generally answer strings as outcomes of the given goals. The planner module $\Pi_\text{plan}$ generates draft execution plans $\hat{p}$ based on an autoregressive policy conditioned on $g$ and $a$. The plans are subsequently refined by humans who interact with only a single version of the environment and generate the human-refined plans $p^*$, which are \textit{strictly} version-independent (\S\ref{sec:TimeTrajHumanEffort}). Humans provide additional details in the execution plan, such as task execution checkpoints (\S\ref{sec:singleStepExecution}), making it easier for an agent to complete the tasks. The task goals and plans are aggregated to form the planning dataset, $\mathcal{D}_\text{plan}$. The procedure is formally described in Algorithm \ref{alg:HiTLPlanDistillation}.

\vspace{-4pt}
\subsubsection{Teacher Rollouts}
\vspace{-4pt}
\label{sec:teacherRolloutsMethod}
\sectionvspace
The teacher policy $\Pi_T$ collects teacher rollouts via online interactions with each version. It iterates over each goal and plan in the planning dataset, $\mathcal{D}_\text{plan}$, and executes them in each environment version. The resulting trajectory is a sequence of observation histories and responses, $\tau=\{(h_t, y_{t})\}_{t=1}^T$ with task horizon $T$. The agent's response at time step $t$ is denoted as a 4-tuple, $y_t=\langle a_t, c_t,p_t, m_t\rangle$, % ORIG: where $a_t$ is the action tokens, $c_t$ is the chain-of-thought thinking tokens, $p_t$ is the planning tokens, and $m_t$ is the memory tokens.
{where $a_t$, $c_t$, $p_t$, and $m_t$ denote the action, chain-of-thought thinking, planning, and memory tokens, respectively.} The observation history includes the response history $y_{1:t-1}$. Each trajectory is evaluated by the judge $J_\phi$, which assigns binary rewards based on success and failure. Successful trajectories are filtered and appended to the trajectory dataset $\mathcal{D}_\tau$, with each trajectory represented as observation history and response pairs,
% \vspace{-0.2cm}
\begin{equation}
\label{eq:trajectory}
    \mathcal{D}_{\tau} = \{\tau_i\}_{i=1}^N
      = \big\{ (h_{i,1}, y_{i,1}, \dots, h_{i,T_i}, y_{i,T_i}) \big\}_{i=1}^N,
% \vspace{-0.2cm}
\end{equation}
where $T_i$ is the number of steps taken in the task or the task horizon. The overall procedure is described in Algorithm \ref{alg:teacherRollouts}.  
Using the trajectory collection algorithm, we collected 757 training trajectories against the 768 training tasks, as 11 unsuccessful trajectories were filtered out by the judge evaluator.

\begin{figure}[t]
\begin{tcolorbox}[colback=gray!08, colframe=gray!8, boxrule=0pt, arc=5pt,
                  left=2pt, right=4pt, top=6pt, bottom=4pt]
\begin{algorithmic}[1]
\REQUIRE student policy $\pi_\theta$, environments $\{\mathcal{E}_1, \ldots, \mathcal{E}_k\}$, goal dataset $\mathcal{D}_\text{goal}$
\ENSURE trained student policy $\pi_\theta$
\STATE \textbf{Init:} $\theta \gets \theta_{\mathrm{pre}}$ (pretrained parameters)
\STATE \textbf{Plan:} $\mathcal{D}_\text{plan} \gets \textsc{PlanDistill}(\mathcal{D}_\text{goal}, \mathcal{E}_p)$, where $\mathcal{E}_p \in \{\mathcal{E}_1, \ldots, \mathcal{E}_k\}$ (Alg.~\ref{alg:HiTLPlanDistillation})
\STATE \textbf{Rollout:} $\mathcal{D}_\tau \gets \textsc{TeacherRollout}(\mathcal{D}_\text{plan})$, the teacher trajectories (Alg.~\ref{alg:teacherRollouts})
\STATE \textbf{Train:} fine-tune $\pi_\theta$ by supervised learning on observation histories $h$ and responses $y$ in $\mathcal{D}_\tau$, minimizing $\mathcal{L}_\text{TW-BC} = -\E_{(h,y)\sim\mathcal{D}_\tau}[\log \pi_\theta(y \mid h)]$ (Eq.~\ref{eq:timewarpBC})
\RETURN $\pi_\theta$
\end{algorithmic}
\vspace{2pt}\hrule\vspace{6pt}
\captionof{algorithm}{{\timewarp\ behavior cloning: imitating plan-guided teacher rollouts}}
\label{alg:timewarpBC}
\end{tcolorbox}
\end{figure}

\subsection{\timewarp-BC}
\label{sec:timewarpBC}

Building on the standard BC framework, we train web agents on the full teacher-agent response rather than only on action tokens. % [REFINE: plain] Reworded for grammar ("use the options such as", "more in distribution to"); no change in claims.
% ORIG: This has two key advantages: (i) it allows the agent to use the options such as thinking tokens needed for reasoning,  planning for solving complex tasks, and memory required for remembering information from certain tasks, which are often required for the tasks in the \timewarp\ dataset, (ii) the training data will be more in distribution to the test data.
{This has two key advantages: (i) it lets the agent produce the thinking tokens needed for reasoning, the planning tokens needed for complex tasks, and the memory tokens needed to retain information across steps, all of which \timewarp\ tasks often require, and (ii) the training data are closer in distribution to the test data.} The model is trained on the full response $(h_t,y_t)$, with $y_t$ containing action, thinking, planning, and memory tokens as mentioned in \S\ref{sec:teacherRolloutsMethod} (details in \S\ref{subsec:timewarpBCdetails}). % [REFINE: connection] Eq.~\ref{eq:vanillaBCLoss} is defined in the appendix, so point the reader there.
% ORIG: The BC loss from Eq. \ref{eq:vanillaBCLoss} can be reformulated as:
% [PROOF-CHECK] Not a reformulation. Standard BC trains the model to emit the parsed action $a=\phi(y)$ as its whole response, as Eq.~\ref{eq:timewarpBCATMP} makes explicit, while \timewarp-BC trains it to emit the full response $y$. The two losses therefore score different target sequences, and neither is a rewriting of the other. An earlier version of this note claimed $\mathcal{L}_{\text{BC}}\le\mathcal{L}_{\text{TW-BC}}$. That inequality holds only if $\pi_\theta(a\mid h)$ is read as the marginal of $\pi_\theta(y\mid h)$ over the fibre of $\phi$, which the paper does not define, so it is not asserted here.
% ORIG: {The standard BC loss (Eq.~\ref{eq:vanillaBCLoss}, \S\ref{sec:vanillaStandardBC}) can be reformulated as:}
{The standard BC loss (Eq.~\ref{eq:vanillaBCLoss}, \S\ref{sec:vanillaStandardBC}) is replaced by a loss on the full response:}
% \vspace{-0.1cm}
\begin{equation}
\label{eq:timewarpBC}
    \mathcal{L}_{\text{TW-BC}}(\theta)
    = - \mathbb{E}_{(h,y) \sim D_\tau} \big[ \log \pi_\theta(y \mid h) \big].
% \vspace{-0.15cm}
\end{equation}
\vspace{-20pt}
\section{Experiments}
\vspace{-6pt}
\label{sec:Results}
% \vspace{-8pt}

% \sectionvspace
\subsection{Experimental Setup}\label{sec:experimentalSetup}
\sectionvspace
All experiments are conducted using the BrowserGym environment \cite{tmlr25dechezellesBrowserGym}, with details in \S\ref{sec:experimentDetailsAppendix}.
Results are averages over $3$ runs.
The training was performed using Llamafactory~\cite{zheng2024llamafactory} on the accessibility tree, with models hosted locally using vLLM~\cite{llmcompressor2024}. % ORIG: The agents are evaluated using reasoning, planning, and memory capabilities, as tasks require varying degrees of these skills.
{The agents are evaluated with reasoning, planning, and memory capabilities enabled, as tasks require varying degrees of these skills.}

%%% Final draft by Farhan
\textbf{Baselines.} For the prompting baselines, we evaluate a range of open-source models, including (i) LLMs: Qwen-3 4B base and thinking \cite{yang2025qwen3}, Llama-3.1 8B \cite{grattafiori2024llama}, and (ii) VLMs: Qwen-3 VL 8B base and thinking \cite{yang2025qwen3}, and Gemma-3 12B \cite{team2025gemma}. We focus on open-source models to assess the effect of training on a model's robustness. The LLMs receive either the \html~or Accessibility Tree (\axt), and VLMs receive either the same input, screenshots (\sshot), or a Set of Marks (\som). % [REFINE: plain] The four training settings were not parallel; each now names what it trains on.
% ORIG: The LLMs are also fine-tuned on the AXT using behavior cloning on single ({BCv$_6$}), multiple versions ({BCv$_{1:6}$}), with reasoning traces (\bct), and the \timewarp~variant (\tw).
{The LLMs are also fine-tuned on the \axt\ using behavior cloning on a single version ({BCv$_6$}), on multiple versions ({BCv$_{1:6}$}), on multiple versions with reasoning traces (\bct), and the \timewarp\ variant (\tw).}

%\input{Figures/graphs/sampleExperiment}

% \farhan{Execution plan has checkpoints. Humans add the checkpoints.}
% \farhan{Add Complexity -> Number of elements in visual and textual elements, number of tokens. Add it in the version figure. Right now it is a bit unclear what the versions represent visually. Adding a new diagram on the versions (3rowsx2columns) will be pretty cool too}
% \farhan{Examples of tasks in the main table. Average length of trajectories for the GPT trajectories, and average length of tasks.  Remember, this is a dataset paper; talk a bit more about the dataset statistics.}

\vspace{-4pt}
\subsection{Results and Analysis}

% ORIG: Following Tab.\ref{tab:mainBenchmarkVisual} and \ref{tab:mainBenchmarkTextual}, \textsc{TimeWarp-BC} (\tw) models consistently outperform the other zero-shot (ZS) and trained baselines, notably, Qwen 3 4B \tw\ achieving the highest success rate across all environments and versions.
{Following Tab.~\ref{tab:mainBenchmarkVisual} and \ref{tab:mainBenchmarkTextual}, \textsc{TimeWarp-BC} (\tw) models consistently outperform the other zero-shot (ZS) and trained baselines, with Qwen-3 4B \tw\ achieving the highest success rate across all environments and versions.} Performance gains are more pronounced on multi-site tasks and remain consistent across all versions. We provide additional empirical results (\S\ref{sec:additionalResultAnalysis}) in the appendix.

\textbf{ZS visual agents are vulnerable to web changes.} Following Tab. \ref{tab:mainBenchmarkVisual}, % ORIG: performance of agents using SoM ranges span $[2.6,21.4]$ for Qwen-3 VL 8B
{the performance of agents using \som\ spans $[2.6,21.4]$ for Qwen-3 VL 8B}, $[1.3,12.3]$ for Qwen-3 VL 8B Thinking, and $[7.4,14.8]$ for Gemma-3 12B. The variation is less noticeable with \sshot, primarily because models perform worse in this setting. While Qwen-3 VL family models have poor \sshot\ performance, Gemma achieves a decent {$8.8\%$} overall, with a consistent $[8.1,10.7]$ across the versions. Notably, the same Qwen-3 VL 8B checkpoint shows a performance range ($\Delta$v) of $7.4/8.1$ under \html/\axt\ but $18.8$ under \som, \textit{i.e.}, with model capacity held constant, the additional variance is attributable to the visual observation modality rather than the model's base performance.

\begin{table*}[ht]
\centering
\footnotesize
\caption{Success rate (\%) of vision language models (VLMs) under various zero-shot observation settings: (i) Textual (\textcolor{orange!80!black}{\textbf{T}}): \html\ and Accessibility Tree (\axt), and (ii) Visual (\textcolor{cyan!80!black}{\textbf{V}}): UI Screenshot (\sshot) and Set of Marks (\som). $\Delta$v represents the difference between version-wise best and worst performance. Human indicates the human baseline reported in \S\ref{sec:humanBaseline}.}
\label{tab:mainBenchmarkVisual}
\setlength{\tabcolsep}{4pt}
\resizebox{\textwidth}{!}{%
\begin{tabular}{lcccccccccccccc}
\toprule
\multirow{2}{*}{\textbf{Model}} &
\multicolumn{2}{c}{\multirow{2}{*}{\textbf{Setting}}} &
\multicolumn{4}{c}{\textbf{Environment}} &
\multicolumn{7}{c}{\textbf{Version}} &
\multirow{2}{*}{\textbf{O/A}} \\
\cmidrule(lr){4-7}\cmidrule(lr){8-14}
& & &
\textbf{Wiki} & \textbf{News} & \textbf{Shop} & \textbf{Multi} &
\textbf{v1} & \textbf{v2} & \textbf{v3} & \textbf{v4} & \textbf{v5} & \textbf{v6} & \textbf{$\Delta$v}\textcolor{red!70!black}{$\downarrow$}
\\
\midrule
\multirow{4}{*}{Qwen-3 VL 8B}& \multirow{2}{*}{\textcolor{orange!80!black}{\textbf{T}}}
 & \multicolumn{1}{c|}{\cellcolor{yellow!15}HTML}
 & \textbf{20.8} & \textbf{33.6} &\textbf{20.8} & \multicolumn{1}{c|}{24.4}
 & \textbf{27.8} & \textbf{23.3} & 22.6 & 24.6 & \textbf{20.4} & \multicolumn{1}{c|}{\textbf{27.2}} & \multicolumn{1}{c|}{07.4}
 & \textbf{24.3} \\
 && \multicolumn{1}{c|}{\cellcolor{orange!15}AXT}
 & 19.0 & 32.8 & 18.1 & \multicolumn{1}{c|}{\textbf{24.9}}
 & 20.7 & 22.0 & \textbf{24.6} & \textbf{27.5} & 19.4 & \multicolumn{1}{c|}{23.9} & \multicolumn{1}{c|}{08.1}
 & 23.0 \\ \cmidrule{2-15}
& \multirow{2}{*}{\textcolor{cyan!80!black}{\textbf{V}}}  &\multicolumn{1}{c|}{\cellcolor{teal!10}SS}
 & 00.7 & 01.8 & 00.2 & \multicolumn{1}{c|}{00.0}
 & 01.3 & 00.6 & 00.6 & 00.6 & 00.3 & \multicolumn{1}{c|}{00.3} & \multicolumn{1}{c|}{01.0}
 & 00.6 \\
 && \multicolumn{1}{c|}{\cellcolor{lime!15}SoM}
 & 06.8 & 16.4 & 15.2 & \multicolumn{1}{c|}{10.6}
 & 02.6 & 04.5 & 07.8 & 21.4 & 14.2 & \multicolumn{1}{c|}{21.0} & \multicolumn{1}{c|}{18.8}
 & 11.9 \\

\midrule
\multirow{4}{*}{\begin{tabular}[l]{@{}l@{}}Qwen-3 VL 8B\\Thinking\end{tabular}}& \multirow{2}{*}{\textcolor{orange!80!black}{\textbf{T}}}
 & \multicolumn{1}{c|}{\cellcolor{yellow!15}HTML}
 & 08.8 & 26.5 & 09.3 & \multicolumn{1}{c|}{08.4}
 & 12.6 & 12.6 & 10.4 & 12.9 & 12.0 & \multicolumn{1}{c|}{15.2} & \multicolumn{1}{c|}{04.8} & 12.6 \\
 && \multicolumn{1}{c|}{\cellcolor{orange!15}AXT}
 & 11.1 & 26.9 & 08.6 & \multicolumn{1}{c|}{06.7}
 & 14.4 & 09.4 & 11.0 & 13.6 & 12.3 & \multicolumn{1}{c|}{16.5} & \multicolumn{1}{c|}{07.1}
 & 12.8 \\ \cmidrule{2-15}
 & \multirow{2}{*}{\textcolor{cyan!80!black}{\textbf{V}}} & \multicolumn{1}{c|}{\cellcolor{teal!10}SS}
 & 00.0 & 01.8 & 01.4 & \multicolumn{1}{c|}{00.0}
 & 01.3 & 00.3 & 01.0 & 01.0 & 00.0 & \multicolumn{1}{c|}{01.0} & \multicolumn{1}{c|}{01.3}
 & 00.8 \\
 && \multicolumn{1}{c|}{\cellcolor{lime!15}SoM}
 & 01.4 & 08.1 & 09.1 & \multicolumn{1}{c|}{01.2}
 & 01.9 & 01.3 & 03.5 & 06.5 & 03.2 & \multicolumn{1}{c|}{12.3} & \multicolumn{1}{c|}{11.0}
 & 04.8 \\

\midrule
\multirow{4}{*}{Gemma-3 12B}& \multirow{2}{*}{\textcolor{orange!80!black}{\textbf{T}}} 
 & \multicolumn{1}{c|}{\cellcolor{yellow!15}HTML}
 & 14.0 & 27.8 & 10.5 & \multicolumn{1}{c|}{10.6}
 & 13.3 & 12.6 & 14.6 & 18.4 & 12.9 & \multicolumn{1}{c|}{19.7} & \multicolumn{1}{c|}{07.2}
 & 15.3 \\
& & \multicolumn{1}{c|}{\cellcolor{orange!15}AXT}
 & 15.4 & 27.8 & 13.2 & \multicolumn{1}{c|}{19.1}
 & 19.7 & 15.2 & 19.1 & 18.8 & 15.8 & \multicolumn{1}{c|}{21.7} & \multicolumn{1}{c|}{06.5}
 & 18.3 \\ \cmidrule{2-15}
 & \multirow{2}{*}{\textcolor{cyan!80!black}{\textbf{V}}} & \multicolumn{1}{c|}{\cellcolor{teal!10}SS}
 & 07.1 & 15.2 & 08.6 & \multicolumn{1}{c|}{05.3}
 & 08.7 & 08.1 & 08.4 & 08.7 & 10.7 & \multicolumn{1}{c|}{08.1} & \multicolumn{1}{c|}{02.6}
 & 08.8 \\
& & \multicolumn{1}{c|}{\cellcolor{lime!15}SoM}
 & 09.9 & 15.4 & 11.7 & \multicolumn{1}{c|}{07.2}
 & 07.4 & 09.1 & 10.3 & 14.8 & 09.7 & \multicolumn{1}{c|}{14.5} & \multicolumn{1}{c|}{07.4}
 & 11.0 \\
\midrule
\multicolumn{3}{c|}{\cellcolor{gray!10}Human}
 & \cellcolor{gray!10}100 & \cellcolor{gray!10}95.4 & \cellcolor{gray!10}96.3 & \multicolumn{1}{c|}{\cellcolor{gray!10}95.6} &\cellcolor{gray!10}97.1&\cellcolor{gray!10}97.1&\cellcolor{gray!10}97.1&\cellcolor{gray!10}97.1&\cellcolor{gray!10}97.1& \multicolumn{1}{c|}{\cellcolor{gray!10}97.1} &\multicolumn{1}{c|}{\cellcolor{gray!10}00.0}
 & \cellcolor{gray!10}97.1 \\
\bottomrule
\end{tabular}
}
% \vspace{-0.3cm}
\end{table*}

\begin{table}[ht]
    \centering
    % First Minipage
    \begin{minipage}{0.49\textwidth}
        \centering
       \caption{Success rate (\%) across zero-shot \axt\ and \tw\ variants. \colorbox{purple!07}{Pink} indicates held-out versions \textit{i.e.}, unseen during training. $\mathcal{D}_{\tau,v}$ is the training data for version $v$. The worst performances in each setting are \underline{underlined}.}
       \vspace{4pt}
\label{tab:generalization}
        \setlength{\tabcolsep}{3pt}
        \resizebox{\columnwidth}{!}{%
\begin{tabular}{llcccccc}
\toprule
\textbf{Model} & ${\mathcal{D}_{\tau,v}}$ & \textbf{v1} & \textbf{v2} & \textbf{v3} & \textbf{v4} & \textbf{v5} & \textbf{v6} \\
\midrule

\multirow{4}{*}{\begin{tabular}[c]{@{}l@{}}Qwen-3 4B\\\end{tabular}}
 & -   & \cellcolor{purple!07}22.0 & \cellcolor{purple!07}\underline{18.4} & \cellcolor{purple!07}19.4 & \cellcolor{purple!07}20.7 & \cellcolor{purple!07}18.8 & \cellcolor{purple!07}23.0 \\
 & v$_1$   & 23.0 & \cellcolor{purple!07}19.1 & \cellcolor{purple!07}16.5 & \cellcolor{purple!07}15.2 & \cellcolor{purple!07}\underline{09.7} & \cellcolor{purple!07}12.9 \\
 & v$_5$   & \cellcolor{purple!07}\underline{13.6} & \cellcolor{purple!07}18.8 & \cellcolor{purple!07}15.5 & \cellcolor{purple!07}16.8 & 20.1 & \cellcolor{purple!07}17.8 \\
 & v$_{1:5}$ & 31.7 & 30.4 & 31.7 & 31.3 & 31.0 & \cellcolor{purple!07}\underline{29.7} \\
\midrule

\multirow{4}{*}{\begin{tabular}[c]{@{}l@{}}Qwen-3 4B\\Thinking\end{tabular}}
 & -   & \cellcolor{purple!07}29.1 & \cellcolor{purple!07}28.5 & \cellcolor{purple!07}28.8 & \cellcolor{purple!07}31.7 & \cellcolor{purple!07}\underline{24.9} & \cellcolor{purple!07}29.5 \\
 & v$_1$   & 14.6 & \cellcolor{purple!07}12.0 & \cellcolor{purple!07}12.0 & \cellcolor{purple!07}12.9 & \cellcolor{purple!07}\underline{09.7} & \cellcolor{purple!07}13.3 \\
 & v$_5$   & \cellcolor{purple!07}14.9 & \cellcolor{purple!07}14.2 & \cellcolor{purple!07}\underline{12.9} & \cellcolor{purple!07}15.2 & 18.4 & \cellcolor{purple!07}15.9 \\
 & v$_{1:5}$ & 33.0 & 30.7 & 35.6 & 28.5 & 32.7 & \cellcolor{purple!07}\underline{27.9} \\
\midrule

\multirow{4}{*}{\begin{tabular}[c]{@{}l@{}}Llama-3.1 8B\\\end{tabular}}
 & -  & \cellcolor{purple!07}\underline{00.4} & \cellcolor{purple!07}00.8 & \cellcolor{purple!07}\underline{00.4} & \cellcolor{purple!07}01.0 & \cellcolor{purple!07}00.8 & \cellcolor{purple!07}00.8 \\
& v$_1$   & 18.1 & \cellcolor{purple!07}15.9 & \cellcolor{purple!07}12.0 & \cellcolor{purple!07}17.8 & \cellcolor{purple!07}\underline{07.1} & \cellcolor{purple!07}11.7 \\
 & v$_5$   & \cellcolor{purple!07}\underline{13.6} & \cellcolor{purple!07}16.2 & \cellcolor{purple!07}14.9 & \cellcolor{purple!07}15.5 & 14.6 & \cellcolor{purple!07}16.5 \\
 & v$_{1:5}$ & 29.1 & 27.9 & 25.6 & 25.2 & \underline{24.3} & \cellcolor{purple!07}26.2 \\
\bottomrule
\end{tabular}
\vspace{-4pt}
}
    \end{minipage}
    \hfill % Adds spacing between the two minipages
    % Second Minipage
    \begin{minipage}{0.49\textwidth}
        \centering
        \caption{Success Rate (\%) of the Qwen-3 4B model on the v$_6$ environments, when trained under different combinations of Thinking (T), Memory (M), and Planning (P) tokens. Action tokens are included in all settings. The best performance is highlighted in \textbf{bold}.}
       \vspace{4pt}
        \setlength{\tabcolsep}{3pt}
\renewcommand{\arraystretch}{1.15}
\resizebox{0.98\columnwidth}{!}{%
\begin{tabular}{ccc|ccccc}
\toprule
\multicolumn{3}{c|}{\textbf{Response}} & \multicolumn{5}{c}{\textbf{Environment}} \\
\cmidrule(lr){1-3}\cmidrule(lr){4-8}
\textbf{T} & \textbf{M} & \textbf{P} &
\textbf{Wiki} & \textbf{News} & \textbf{Shop} & \textbf{Multi} & \textbf{O/A} \\
\midrule
\TF{0} & \TF{0} & \TF{0} & 16.1\% & 42.4\% & 17.3\% & 21.7\% & 23.3\% \\
\TF{1} & \TF{0} & \TF{0} & 19.4\% & 31.8\% & 16.0\% & 40.6\% & 25.9\% \\
\TF{0} & \TF{1} & \TF{0} & 18.3\% & 40.9\% & 11.1\% & 36.2\% & 25.2\% \\
\TF{0} & \TF{0} & \TF{1} & 17.2\% & 45.5\% & 25.9\% & 34.8\% & 29.5\% \\
\TF{1} & \TF{1} & \TF{0} & 22.6\% &\cellcolor{cyan!05}\textbf{ 47.0\%} & 22.2\% & 36.2\% & 30.7\% \\
\TF{1} & \TF{0} & \TF{1} & 25.8\% & 39.4\% & 18.5\% & \cellcolor{cyan!05}\textbf{50.7\%} & 32.4\% \\
\TF{0} & \TF{1} & \TF{1} & 21.5\% & \cellcolor{cyan!05}\textbf{47.0\%} & \cellcolor{cyan!05}\textbf{33.3\%} & 47.8\% & 35.9\% \\
\TF{1} & \TF{1} & \TF{1} & \cellcolor{cyan!05}\textbf{31.2\%} & \cellcolor{cyan!05}\textbf{47.0\%} & 30.9\% & 49.3\% & \cellcolor{cyan!05}\textbf{38.5\%} \\
\bottomrule
\end{tabular}
}
\label{tab:ablationTMP}
       \vspace{-10pt}
    \end{minipage}
\end{table}

% \begin{table}[t!]
% \centering
% \small

% % \vspace{-0.4cm}
% \end{table}

\begin{table*}[ht]
\centering
\footnotesize
\caption{Success rate (\%) of LLMs under various (i) Zero-Shot (\textcolor{red!80!black}{\textbf{ZS}}) settings: \html\ and Accessibility Tree (\axt) and (ii) Training (\textcolor{teal}{\textbf{T}}) settings: Behavior Cloning on a single version ({BCv$_{6}$}), all versions ({BCv$_{1:6}$}), using Thinking tokens ({BCTv$_{1:6}$}), and using \textsc{TimeWarp-BC} (\tw). $\Delta$v represents the difference between version-wise best and worst performance, \textbf{bold} indicates best performance.}
\label{tab:mainBenchmarkTextual}
\setlength{\tabcolsep}{4pt}
\resizebox{\textwidth}{!}{%
\begin{tabular}{lcccccccccccccc}
\toprule
\multirow{2}{*}{\textbf{Model}} &
\multicolumn{2}{c}{\multirow{2}{*}{\textbf{Setting}}} &
\multicolumn{4}{c}{\textbf{Environment}} &
\multicolumn{7}{c}{\textbf{Version}} &
\multirow{2}{*}{\textbf{O/A}} \\
\cmidrule(lr){4-7}\cmidrule(lr){8-14}
& & &
\textbf{Wiki} & \textbf{News} & \textbf{Shop} & \textbf{Multi} &
\textbf{v1} & \textbf{v2} & \textbf{v3} & \textbf{v4} & \textbf{v5} & \textbf{v6} & \textbf{$\Delta$v}\textcolor{red!70!black}{$\downarrow$}
\\
\midrule
\multirow{6}{*}{Qwen-3 4B}& \multirow{2}{*}{\textcolor{red!80!black}{\textbf{ZS}}} 
 & \multicolumn{1}{c|}{\cellcolor{yellow!15}HTML}
 & 19.0 & 26.0 & 18.5 & \multicolumn{1}{c|}{19.8}
 & 21.7 & 19.8 & 21.1 & 21.0 & 17.1 & \multicolumn{1}{c|}{22.6} & \multicolumn{1}{c|}{05.5}
 & 20.5 \\
 && \multicolumn{1}{c|}{\cellcolor{orange!15}AXT}
 & 20.3 & 25.3 & 18.1 & \multicolumn{1}{c|}{18.6}
 & 22.0 & 18.4 & 19.4 & 20.7 & 18.8 & \multicolumn{1}{c|}{23.0} & \multicolumn{1}{c|}{04.5}
 & 20.4 \\ \cmidrule{2-15}
 & \multirow{4}{*}{\textcolor{teal}{\textbf{T}}} & \multicolumn{1}{c|}{\cellcolor{red!10}BCv$_6$}
 & 11.5 & 38.9 & 20.2 & \multicolumn{1}{c|}{11.1}
 & 17.5 & 19.4 & 21.3 & 23.0 & 17.8 & \multicolumn{1}{c|}{18.1} & \multicolumn{1}{c|}{05.5}
 & 19.5 \\
 && \multicolumn{1}{c|}{\cellcolor{red!10}BCv$_{1:6}$}
 & 10.2 & 34.9 & 16.5 & \multicolumn{1}{c|}{16.6}
 & 16.8 & 17.1 & 17.5 & 21.4 & 22.0 & \multicolumn{1}{c|}{16.5} & \multicolumn{1}{c|}{05.5}
 & 18.5 \\

 && \multicolumn{1}{c|}{\cellcolor{violet!15}BCTv$_{1:6}$}
 & 14.2 & 36.6 & 25.5 & \multicolumn{1}{c|}{38.4}
 & 23.0 & 27.2 & 25.3 & 25.5 & 30.4 & \multicolumn{1}{c|}{32.7} & \multicolumn{1}{c|}{09.7}
 & 27.4 \\

 && \multicolumn{1}{c|}{\cellcolor{blue!10}TW}
 & \textbf{\cellcolor{blue!10}31.0} & \textbf{\cellcolor{blue!10}43.2} & \textbf{\cellcolor{blue!10}30.7} & \multicolumn{1}{c|}{\textbf{\cellcolor{blue!10}49.8}}
 & \textbf{\cellcolor{blue!10}39.8} & \textbf{\cellcolor{blue!10}35.6} & \textbf{\cellcolor{blue!10}40.8} & \textbf{\cellcolor{blue!10}36.6} & \textbf{\cellcolor{blue!10}34.9} & \multicolumn{1}{c|}{\textbf{\cellcolor{blue!10}38.5}} & \multicolumn{1}{c|}{\cellcolor{blue!10}05.8}
 & \textbf{\cellcolor{blue!10}37.7} \\
\midrule
\multirow{6}{*}{\begin{tabular}[l]{@{}l@{}}Qwen-3 4B\\Thinking\end{tabular}}& \multirow{2}{*}{\textcolor{red!80!black}{\textbf{ZS}}} 
 & \multicolumn{1}{c|}{\cellcolor{yellow!15}HTML}
 & 27.2 & 33.1 & 28.4 & \multicolumn{1}{c|}{29.2}
 & 28.8 & 28.8 & 29.5 & 32.0 & 26.2 & \multicolumn{1}{c|}{30.1} & \multicolumn{1}{c|}{05.8}
 & 29.2 \\
 && \multicolumn{1}{c|}{\cellcolor{orange!15}AXT}
 & 26.9 & 31.1 & 28.8 & \multicolumn{1}{c|}{29.0}
 & 29.1 & 28.5 & 28.8 & 31.7 & 24.9 & \multicolumn{1}{c|}{29.5} & \multicolumn{1}{c|}{06.8}
 & 28.8 \\ \cmidrule{2-15}
 & \multirow{4}{*}{\textcolor{teal}{\textbf{T}}} & \multicolumn{1}{c|}{\cellcolor{red!10}BCv$_6$}
 & 04.1 & 17.4 & 09.7 & \multicolumn{1}{c|}{05.1}
 & 06.5 & 05.2 & 09.0 & 10.7 & 09.4 & \multicolumn{1}{c|}{11.0} & \multicolumn{1}{c|}{05.8}
 & 08.6 \\
 && \multicolumn{1}{c|}{\cellcolor{red!10}BCv$_{1:6}$}
 & 04.3 & 31.1 & 16.0 & \multicolumn{1}{c|}{01.9}
 & 15.2 & 12.9 & 10.7 & 13.6 & 12.0 & \multicolumn{1}{c|}{11.0} & \multicolumn{1}{c|}{04.6}
 & 12.6 \\
  && \multicolumn{1}{c|}{\cellcolor{violet!15}BCTv$_{1:6}$}
 & 13.6 & 30.5 & 22.5 & \multicolumn{1}{c|}{32.4}
 & 24.3 & 24.0 & 26.5 & 23.0 & 25.2 & \multicolumn{1}{c|}{19.4} & \multicolumn{1}{c|}{07.1}
 & 23.7 \\
 && \multicolumn{1}{c|}{\cellcolor{blue!10}TW}
 & \cellcolor{blue!10}16.1 & \cellcolor{blue!10}37.6 & \cellcolor{blue!10}23.3 & \multicolumn{1}{c|}{\cellcolor{blue!10}37.5}
 & \cellcolor{blue!10}27.8 & \cellcolor{blue!10}28.2 & \cellcolor{blue!10}28.2 & \cellcolor{blue!10}26.2 & \cellcolor{blue!10}26.9 & \multicolumn{1}{c|}{\cellcolor{blue!10}26.9} & \multicolumn{1}{c|}{\cellcolor{blue!10}02.0}
 
 & \cellcolor{blue!10}27.4 \\
\midrule
\multirow{6}{*}{Llama-3.1 8B}& \multirow{2}{*}{\textcolor{red!80!black}{\textbf{ZS}}} 
 & \multicolumn{1}{c|}{\cellcolor{yellow!15}HTML}
 & 00.0 & 00.0 & 00.0 & \multicolumn{1}{c|}{00.0}
 & 00.0 & 00.0 & 00.0 & 00.0 & 00.0 & \multicolumn{1}{c|}{00.0} & \multicolumn{1}{c|}{00.0}
 & 00.0 \\
 && \multicolumn{1}{c|}{\cellcolor{orange!15}AXT}
 & 02.3 & 00.0 & 00.0 & \multicolumn{1}{c|}{00.0}
 & 00.4 & 00.8 & 00.4 & 01.0 & 00.8 & \multicolumn{1}{c|}{00.8} & \multicolumn{1}{c|}{00.6}
 & 00.7 \\ \cmidrule{2-15}
 & \multirow{4}{*}{\textcolor{teal}{\textbf{T}}}  & \multicolumn{1}{c|}{\cellcolor{red!10}BCv$_6$}
 & 00.5 & 09.9 & 03.7 & \multicolumn{1}{c|}{00.2}
 & 02.6 & 04.5 & 02.3 & 03.9 & 00.0 & \multicolumn{1}{c|}{06.5} & \multicolumn{1}{c|}{06.5}
 & 03.3 \\
 && \multicolumn{1}{c|}{\cellcolor{red!10}BCv$_{1:6}$}
 & 00.0 & 01.3 & 00.0 & \multicolumn{1}{c|}{00.0}
 & 00.0 & 01.0 & 00.3 & 00.0 & 00.0 & \multicolumn{1}{c|}{00.3} & \multicolumn{1}{c|}{01.0}
 & 00.3 \\
  && \multicolumn{1}{c|}{\cellcolor{violet!15}BCTv$_{1:6}$}
 & 12.2 & 27.8 & 29.8 & \multicolumn{1}{c|}{27.8}
 & 24.6 & 26.9 & 22.3 & 26.2 & 17.5 & \multicolumn{1}{c|}{24.3} & \multicolumn{1}{c|}{09.4}
 & 23.6 \\
 && \multicolumn{1}{c|}{\cellcolor{blue!10}TW}
 & \cellcolor{blue!10}12.7 & \cellcolor{blue!10}42.9 & \cellcolor{blue!10}24.3 & \multicolumn{1}{c|}{\cellcolor{blue!10}34.0}
 & \cellcolor{blue!10}28.8 & \cellcolor{blue!10}28.8 & \cellcolor{blue!10}27.2 & \cellcolor{blue!10}27.5 & \cellcolor{blue!10}22.0 & \multicolumn{1}{c|}{\cellcolor{blue!10}27.5} & \multicolumn{1}{c|}{\cellcolor{blue!10}06.8}
 & \cellcolor{blue!10}27.0 \\
\bottomrule
\end{tabular}
}
% \vspace{-0.3cm}
\end{table*}

\begin{figure}[htp]
    \centering
    {
\begin{tikzpicture}
\begin{axis}[
    ybar,
    width=0.33\columnwidth,
    height=3.35cm,
    enlarge y limits={upper},
    enlarge x limits=0.60,
    ylabel={SR (\%)},
    ymin=10,
    grid=both,
    major grid style={dashed, gray!55},
    minor grid style={gray!20},
    minor tick num=1,
    symbolic x coords={
        Qwen-3 4B,
        Qwen-3 4B Th,
        Llama-3.1 8B
    },
    legend image code/.code={
    \draw[#1, draw=black] (0cm,-0.12cm) rectangle (0.15cm,0.12cm);
},
    xtick=data,
    xticklabel style={
    rotate=35,
    anchor=east
},
    bar width=10pt,
    legend style={
        at={(0.5,1.1)},
        anchor=south,
        legend columns=2
    },
    tick label style={font=\small},
]

% Overall (first column)
\addplot[fill=magenta!30, draw=black, postaction={
        pattern=north east lines,
        pattern color=black!90
    }] coordinates {
    (Qwen-3 4B,25.2)
    % (-4B-Th,16.5)
    (Llama-3.1 8B,20.7)
};

% Overall (second column)
\addplot[fill=teal!30, draw=black] coordinates {
    (Qwen-3 4B,16.2)
    % (-4B-Th,20.4)
    (Llama-3.1 8B,17.4)
};

\legend{$\mathcal{D}_{\tau,6}$, $\mathcal{D}_{\tau,6\rightarrow1}$}
\end{axis}
\node[anchor=south east] at (current bounding box.south east) {(a)};
\end{tikzpicture}
    }
    \hfill
    {
\begin{tikzpicture}
    \begin{axis}[
        ybar,
        width=0.33\columnwidth,
        height=3.35cm,
        enlarge y limits={upper, value=0.1},
        enlarge x limits=0.60,
        ylabel={SR (\%)},
        ymin=0,
        grid=both,
        major grid style={dashed, gray!55},
        minor grid style={gray!20},
        minor tick num=1,
        symbolic x coords={
            Qwen-3 4B,
            Llama-3.1 8B
        },
        legend image code/.code={
            \draw[#1, draw=black] (0cm,-0.12cm) rectangle (0.15cm,0.12cm);
        },
        xtick=data,
        xticklabel style={
            rotate=35,
            anchor=east
        },
        bar width=10pt,
        legend style={
            at={(0.5,1.1)},
            anchor=south,
            legend columns=2,
            font=\small
        },
        tick label style={font=\small},
    ]
    \addplot[fill=blue!20, draw=black, postaction={
            pattern=north east lines,
            pattern color=black!90
        }] coordinates {
        (Qwen-3 4B, 4.2)
        (Llama-3.1 8B, 2.4)
    };
    \addplot[fill=orange!20, draw=black] coordinates {
        (Qwen-3 4B, 10.3)
        (Llama-3.1 8B, 6.7)
    };
    \legend{BCv$_{1:6}$, TW}
    \end{axis}
    \node[anchor=south east] at (current bounding box.south east) {(b)};
    \end{tikzpicture}
    }
    \hfill
    {
       \begin{tikzpicture}
\begin{axis}[
    ybar,
    width=0.33\columnwidth,
    height=3.35cm,
    enlarge y limits={upper},
    enlarge x limits=0.6,
    ylabel={SR (\%)},
    ymin=10,
    grid=both,
    major grid style={dashed, gray!55},
    minor grid style={gray!20},
    minor tick num=1,
    symbolic x coords={
        Qwen-3 4B,
        Qwen-3 4B Th,
        Llama-3.1 8B
    },
    xtick=data,
    xticklabel style={
    rotate=35,
    anchor=east
},
    bar width=10pt,
    legend style={
        at={(0.4,1.05)},
        anchor=south,
        legend columns=2
    },
    legend image code/.code={
    \draw[#1, draw=black] (0cm,-0.12cm) rectangle (0.15cm,0.12cm);
} 
    tick label style={font=\small},
]

% with WA
\addplot[fill=cyan!25, draw=black, postaction={
        pattern=north east lines,
        pattern color=black!90
    }] coordinates {
    (Qwen-3 4B,30.8)
    % (-4B-Th,29.4)
    (Llama-3.1 8B,15.9)
};

% without WA
\addplot[fill=violet!25, draw=black] coordinates {
    (Qwen-3 4B,38.5)
    % (-4B-Th,26.9)
    (Llama-3.1 8B,27.5)
};

\legend{with WA, w/o WA}

\end{axis}
% Add (b) label in bottom right
\node[anchor=south east] at (current bounding box.south east) {(c)};
\end{tikzpicture}
    }
    % \vspace{-7pt}
    \caption{Success rate (SR) of (a) agents trained on \timewarp\ v$_6$ trajectories $\mathcal{D}_{\tau,6}$ \textit{vs.} trained continually on v$_6$ then v$_1$ trajectories $\mathcal{D}_{\tau,6\rightarrow1}$, (b) cross-benchmark performance on WebArena-Lite \cite{Liu2024VisualAgentBenchTLBG}, (c) training with \textit{vs.} without WebArena-Lite training data and evaluated on the v$_6$ environments.}
    % \vspace{-0.6cm}
\label{fig:webArenaTrainingAndContinualLearning}
% \vspace{-08pt}
\end{figure}

\textbf{Text observations are more robust than visual ones.} From Tab. \ref{tab:mainBenchmarkVisual}, % [REFINE: punctuation] Colon replaced by parentheses.
% ORIG: both LLM and VLM agents have significantly lower performance variations when using textual observations: \html\ and \axt, instead of visual ones.
{VLM agents usually have significantly lower performance variations when using textual observations (\html\ and \axt) instead of visual ones.} Performance between \html\ and \axt\ settings is comparable with a ${\leq}3\%$ difference. In Tab.\ref{tab:mainBenchmarkTextual}, Qwen-3 4B Thinking with HTML achieves the highest ZS performance of $29.2\%$, with moderate robustness $[26.2,32.0]$. Most models similarly exhibit $\sim6\%$ difference across versions. Visual agents are more vulnerable to version changes, likely because UI shifts more significantly than the underlying HTML/AXT, \textit{e.g.}, while the code differences between v$_1$ (HTML 4.01 Transitional) and v$_5$ (HTML 5.0) are notable, they pale in comparison to the visual changes (\S\ref{sec:versionExamples}).

\textbf{Training helps, if done right.} Training LLM agents with \tw\ provides substantial performance gains: +17.3\% overall for Qwen-3 4B, and +26.3\% for Llama-3.1 8B, with similar consistency across versions. 
We highlight Llama-3.1 8B, improving from an unusable 0.7\% S.R. in untrained \axt\ to a functioning agent with 27.0\% S.R. in \tw. However, these gains do not extend to vanilla \bc, as the Qwen models see performance drops, while Llama shows incremental gains. Adding reasoning traces during training (\bct) partially mitigates this, but is outperformed by \timewarp-\textsc{BC}. These results suggest {that} the absence of additional tokens during training hinders the agent's ability to solve complex web tasks in our benchmark, strongly aligning with previous findings \cite{vattikonda2025how,hu-etal-2025-webcot}. 

\begin{figure*}[ht]
    \centering
    \includegraphics[width=0.98\linewidth]{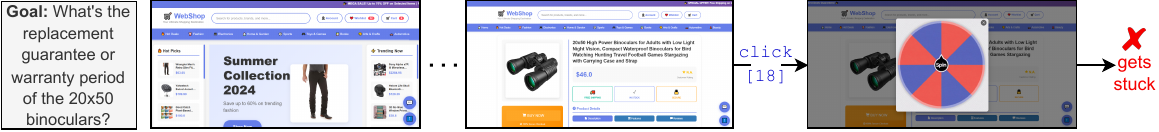}
    \caption{Popup Ads in Shop v$_5$ affecting a web agent's workflow and task success.}
    \label{fig:popUpAdsErrorQualAnalysis}
    % \vspace{-0.2cm}
    \vspace{-5pt}
\end{figure*}

\textbf{Thinking Models as Web Agents.}  We evaluate the thinking variants of  LLM and VLM agents. % [REFINE: story] Tab.~\ref{tab:mainBenchmarkVisual} gives 12.6/12.8/4.8 (Thinking) vs. 24.3/23.0/11.9 (non-thinking) under HTML/AXT/SoM, so the thinking variant underperforms. The following "In contrast" only reads correctly this way.
% ORIG: Qwen-3VL 8B thinking outperforms its non-thinking counterpart across all non-trivial settings. Our analysis shows that thinking models are producing more errors during visual reasoning (Fig. \ref{fig:overthikingQwen3VLthinking}).
{Qwen-3 VL 8B Thinking underperforms its non-thinking counterpart across all observation settings. Our analysis shows that thinking models produce more errors during visual reasoning (Fig.~\ref{fig:overthikingQwen3VLthinking}).} In contrast, the thinking variant of Qwen-3 4B outperforms the non-thinking model in the \html\ and \axt\ settings, but its performance degrades under \tw. Our analysis suggests that while training adds more structure to the thinking tokens (Fig.~\ref{fig:qwen4BThinkingImproved}), it can lead to mild performance degradation, potentially due to a mismatch between the teacher’s and student’s reasoning processes.

\textbf{Does multi-version training generalize better?} We evaluate generalization by holding out at least one version and training agents on the remaining version(s). As shown in Tab. \ref{tab:generalization}, models trained on a single version, either the older v$_1$ or the recent v$_5$, perform better on similar versions, but degrade on dissimilar ones. 
However, training on $5$ versions generalizes better to the held-out version. Degradation often scales with version distance, \textit{e.g.}, Qwen-3 4B trained on v$_1$ scores $23.0/19.1/16.5/15.2/9.7$ on v$_1$--v$_5$, \textit{i.e.}, failure tracks the amount of change, not just its presence. {\S\ref{sec:theory} gives a simple theoretical account of both trends in multi-version training.}

\textbf{Can training transfer across benchmarks?} We assess cross-benchmark generalizability by training on the \timewarp\ environments and evaluating on WebArena-Lite \cite{Liu2024VisualAgentBenchTLBG}, with Fig. \ref{fig:webArenaTrainingAndContinualLearning} (b) showing \tw\ agents significantly outperforming \bc.
We also assess whether integrating cross-benchmark data can improve performance by fine-tuning on WebArena-Lite and then \timewarp\ continually, with Fig. \ref{fig:webArenaTrainingAndContinualLearning} (c) showing the cross-dataset samples degrading the performance. 

\textbf{Can we train continually on versions?} We experiment by training Qwen-3 4B and Llama-3.1 8B using \timewarp-\textsc{BC} on (i) v$_6$ only for 1 epoch, and (ii) first v$_6$, then continually on v$_1$, for 1 epoch each. When evaluated on v$_6$ environments, performance drops after continual training (Fig. \ref{fig:webArenaTrainingAndContinualLearning} (a)), which could be due to catastrophic forgetting \cite{kirkpatrick2017overcoming}. 

% We treat this as a preliminary result and view continual learning on evolving web versions as a promising direction for future exploration.

\textbf{Does training sample count matter?} Following Fig. \ref{fig:trainingSamplesBCTWContextLength}(a), for Qwen-3 4B trained with \bc\ on a single version, performance saturates as the number of training samples increases. In contrast, the \tw\ performance scales approximately linearly with the number of training samples.

% and \tw trained on six versions, and find \bc's performance saturating while \tw's performance rises somewhat linearly with the number of training samples.

\textbf{Ablations.} We perform ablations using the Qwen-3 4B model on the v$_6$ environments to assess (i) training token combinations, (ii) training context window length, (iii) parameter-efficient fine-tuning via LoRA, and (iv) the number of training epochs. 
Tab. \ref{tab:ablationTMP} shows that training on thinking, memory, and planning tokens, along with the action tokens, substantially outperforms other training token combinations. We also observe performance gains from increasing the context window, specifically, using $64$k context and training for $3$ epochs (\S\ref{sec:ContextLengthAblation}, \ref{sec:epochResults}). While LoRA fine-tuning improves step-efficiency, it results in a significant performance drop compared to full fine-tuning (\S\ref{sec:LoRAResults}).

\textbf{Error Analysis.} We observe web agents failing in both older and newer web designs (details in \S\ref{sec:variationsVersions}). A common feature of older websites is the lack of quality-of-life features, such as a table of contents (Fig.~\ref{fig:tableofcontent}a), while newer websites introduce complex designs and bloat (Fig. \ref{fig:popUpAdsErrorQualAnalysis}). % [REFINE: plain] Literal phrase instead of "jarring pitfall".
% ORIG: However, the most jarring pitfall of agents is the misalignment between action and non-action tokens
{However, the most striking failure of agents is the misalignment between action and non-action tokens}, present in both trained and untrained agents (\S\ref{sec:actionMisalignmentError}, Fig. \ref{fig:actionMisalignment}). Thinking models also tend to generate more structured thinking tokens after training (Fig. \ref{fig:qwen4BThinkingImproved}) but often overanalyze the observation (\S\ref{sec:thinkingError}, Fig.~\ref{fig:overthikingQwen3VLthinking}).
\section{Discussion \& Conclusion}
% \vspace{-5pt}
As web agents move from research to real-world use, the ability to adapt to an ever-changing environment will matter as much as raw task performance. % [REFINE: plain] "scaffolding" metaphor replaced.
% ORIG: \timewarp\ provides the scaffolding needed to pursue that.
{\timewarp\ provides the testbed needed to study that.} Through evaluations spanning temporal changes, generalization, cross-benchmark transfer, and preliminary continual-learning experiments, we assess a range of open-source agents across modalities and training methods. % [REFINE: plain] "paint a consistent picture" replaced.
% ORIG: Our results paint a consistent picture: \textit{web agents can be brittle to temporal changes}.
{Our results consistently show that \textit{web agents can be brittle to temporal changes}.}

Temporal robustness is not simply a data or training problem. Collecting multi-version training data is expensive, as web interfaces change in ways that invalidate existing trajectories. % ORIG: \timetraj\ unlocks a new way to collect reusable high-level plans
{\timetraj\ provides a new way to collect reusable high-level plans}, and the resulting trajectories with detailed per-step responses can be distilled to further improve robustness using \timewarp-\textsc{BC}. We hope \timewarp\ serves as a foundation for agents that do not merely perform, but adapt over time.

\section*{Broader Impact}

Our work contributes to the methodological shift toward temporal evaluation in agentic AI, with implications for any domain where digital interfaces evolve continuously (\textit{e.g.}, mobile apps, documentation portals). The \textsc{TimeWarp} benchmark could thus serve as a general testbed for the long-term reliability of agents in any dynamic information retrieval system. We elaborate on the broader impact on robust web agents, trajectory annotation for agents, the expandability of \timewarp, and the continual and lifelong learning of agents in Appendix \S\ref{sec:broaderImpact}. 

% \vspace{-5pt}
\section*{Acknowledgements}
% \vspace{-5pt}
We thank Khadija Nejd for her careful annotation of the training tasks and for refining the execution plans. We also thank Dai-Jie Wu for the helpful feedback on an earlier draft of this paper.

\bibliographystyle{unsrtnat}
\bibliography{ref}

% \setlength{\subfigcapskip}{5pt}
% \tableofcontents

\newpage
\appendix

\begin{center}
\LARGE Supplementary Material of \timewarp: Evaluating Web Agents by Revisiting the Past
\end{center}

% 1. Force LaTeX to recognize sections in the ToC again

\DoToC

\newpage
% \onecolumn

\section{Broader Impact (Details)}
\label{sec:broaderImpact}

Below, we detail the broader impact of \timewarp\ across several aspects.

\subsection{Robust Web Agents}
\label{subsec:robustWebAgentsBroaderImpact}
Robust web agents have significant potential to transform how humans interact with the internet and computers. Evaluating and training systems that remain reliable as websites evolve in the real world % [REFINE: punctuation] Dash removed, and "On the other hand, though ... could have" was a sentence fragment.
% ORIG: could dramatically improve user-facing applications -- from assistive browsing tools to automated customer services. On the other hand, though greater automation of the web could have potentially negative consequences, such as increased hacking, security risks, or websites becoming more vulnerable to unauthorized bot scraping or scalping. However, we do not believe that this is an inherent limitation of \timewarp~but of agents, in general.
{could dramatically improve user-facing applications, from assistive browsing tools to automated customer services. On the other hand, greater automation of the web could have negative consequences, such as increased hacking, security risks, and websites becoming more vulnerable to unauthorized bot scraping or scalping. However, we do not believe that this is an inherent limitation of \timewarp\ but of agents in general.}

\subsection{A New Paradigm of Trajectory Annotation for Agents}
\label{subsec:broaderImpactAnnotation}
Through \timetraj, we unlock a new paradigm for scalably collecting data to train agents via distillation. Rather than spending time and resources on human annotators to collect new trajectories whenever a website changes, one can annotate a high-level plan and then use it to automatically generate teacher trajectories for the updated website. As websites naturally evolve, \timetraj\ can re-run a teacher trajectory on the human-refined plans, as these are usually robust to temporal changes. {Agents thus gain updated training data without additional annotation cost.} 

The principle of collecting a more abstract representation from humans and then using automated methods to specify the detailed trajectory can be applied to other domains, such as embodied agents and robotics. This is {especially} important for the web, as websites are constantly updated in their UI, functions, workflows, and data, and \timetraj\ helps keep the training data up to date.

\subsection{Expandability of \timewarp}
\label{subsec:timeWarpExpandability}
As \timewarp\ is tightly integrated with the BrowserGym \cite{tmlr25dechezellesBrowserGym} ecosystem, we can easily expand the environment and dataset across multiple dimensions. The expandability of \timewarp\ positions it as a generalized framework for evaluating the robustness of web agents.

\textbf{Tasks.}  Adding a new task requires only a task definition, a desired outcome, and a lightweight verifier specification composed from our five evaluation function classes (\S\ref{sec:deterministicVerifierDetails}). For tasks with ambiguous outcomes, multiple valid answers can be provided (see \S\ref{sec:multipleCorrectAnswers} for details), and subjective tasks can fall back to the generalized LLM judge, which requires no task-specific evaluation criterion.

\textbf{Versions.} \timewarp's design allows new front-end versions to be easily added, removed, or swapped. All front-end versions share the lightweight backend and data stores, while remaining fully independent without cross-version dependencies.

% . The repository has every version isolated, allowing each to work independently with no dependencies.

\textbf{Websites.} All web environments in \timewarp\ share a common architecture built on the Python-based Flask framework. Introducing a new website is as simple as implementing a new Flask backend with a data store, both of which are fully shared across versions.

\subsection{Continual and Lifelong Learning of Agents}
\label{sec:continualLearningAgentsBroaderImpact}
As websites evolve and new data becomes available, a natural question arises: can agents be continually trained on successive versions without degrading previous capabilities? \timewarp\ provides a controlled platform to study this challenge. Our results in Fig.~\ref{fig:webArenaTrainingAndContinualLearning} show that agents trained sequentially on a single version at a time suffer significant performance drops. This effect can be studied systematically across multiple temporal versions. We hope \timewarp~serves as a testbed for developing continual learning algorithms that help agents retain past capabilities while adapting to the ever-changing web.

\section{Limitations}
\label{sec:limitations}

\paragraph{Domain Coverage.} Each \timewarp\ domain costs roughly six times as much to create as a conventional web-benchmark domain, since every website ships six fully functional versions. As version breadth, rather than domain breadth, is our primary axis of evaluation, this naturally limited the number of environments we could include. The three environments were chosen to span the interaction primitives most affected by UI evolution: three distinct search paradigms, content- vs. metadata-centric retrieval, visual product queries, transactions, and multi-site queries. % [REFINE: punctuation] Semicolon list split into sentences.
% ORIG: Given more time, we wish to include categories such as social media platforms (\textit{e.g.}, Twitter/X, Reddit), which feature dynamic, user-generated content and infinite scrolls; web-based productivity tools (\textit{e.g.}, Google Docs, Notion), which involve stateful editing environments; entertainment/streaming platforms, which are media-heavy and interaction-light; and government and healthcare portals, which often present unique accessibility and navigation challenges.
{Given more time, we wish to include four further categories. Social media platforms (\textit{e.g.}, Twitter/X, Reddit) feature dynamic, user-generated content and infinite scrolling. Web-based productivity tools (\textit{e.g.}, Google Docs, Notion) involve stateful editing environments. Entertainment and streaming platforms are media-heavy and interaction-light. Government and healthcare portals often present unique accessibility and navigation challenges.} Our findings on the vulnerability of agents to web changes should therefore be read within the covered domains, and extending them to a broader range of web workflows remains future work.

\paragraph{Replication of the Earlier Web.} Reproducing historical frontend code is inherently challenging, as we are not practitioners who developed with the earlier web firsthand. We made best-effort attempts using coding agents and the Wayback Machine \cite{InternetArchive} as a reference, though incomplete archival coverage in the Wayback Machine meant we could not rely on it entirely. The reproduced interfaces may therefore not be exactly identical to their actual historical counterparts. However, each era's UI design was manually verified against existing archival snapshots \cite{WebDesignMuseum}, and we consider them close replicas of the actual ones. We further quantify how faithfully the reconstructed versions track real web evolution in \S\ref{sec:versionValidation}, which establishes that the versions vary along historically plausible axes, not that they exactly match the deployed pages of specific years.

% \paragraph{Version Sensitivity.} Version sensitivity is not uniform across agents: it is strongest for zero-shot agents operating on visual observations and for text agents trained on a single version (\S\ref{sec:Results}), whereas multi-version training with \timetraj\ and \timewarp-\textsc{BC} substantially mitigates both.

\paragraph{Partial Rewards.} Most of the web agent literature relies on sparse rewards assigned at the end of a trajectory \cite{Shi2017WorldOBN,yao2022webshop,zhou2023webarena,koh2024visualwebarena}, and we follow the same approach. While this is not a limitation unique to our work but shared across the web agent literature, partial rewards could benefit the field more broadly by enabling finer-grained evaluation and providing dense training signals for RL-based approaches. As it represents a research direction distinct from our core contributions, we do not explore it.

\paragraph{Reproducibility of Proprietary Models.} The teacher trajectories are collected with GPT-5, and the two subjective goals, as well as the validation of our verifiers, rely on a GPT-5.1 judge \cite{gu2024survey}, so parts of our pipeline depend on proprietary models whose exact checkpoints may not remain available. We mitigate the evaluation-side dependence in three ways: $99.1\%$ of the goals are scored by deterministic verifiers that require no LLM (\S\ref{sec:deterministicVerifierDetails}), an open-source Gemma-4 12B judge \cite{team2026gemma} is supported as a drop-in replacement with $99.74\%$ agreement (\S\ref{sec:judgeEvaluatorDetails}), and the judge itself was validated against human verdicts ($98.5\%$ agreement, \S\ref{sec:evaluationJudgeEval}) and across judge models (Tab.~\ref{tab:selfEvalJudge}). The collected teacher trajectories are also released, so training can be reproduced without re-querying the teacher model.

\paragraph{Post-Training.} Our work focuses on the behavior cloning phase, which is considered a prerequisite to any form of web agent training pipeline \cite{wei-etal-2025-webagent,hu-etal-2025-webcot}. Recent literature has explored a variety of post-training methods, including PPO \cite{qi2024webrl}, DPO \cite{putta2024agent}, and GRPO \cite{wei-etal-2025-webagent}. These approaches are orthogonal to our core contribution and can be applied on top of our training method for further gains.

% While these are somewhat orthogonal to our core contribution of establishing a benchmark for evaluating the evolving web, we believe a more methodology-focused paper can explore the effect of post-training methods in improving the temporal robustness of web agents.

% Our work does not explore these directions as we view them orthogonal to the core contribution of our work, but a different tangent can explore the impact of these post-training methods to improve the temporal robustness of web agents. 

%Mention RL as a potential training direction; however, frame it more as a future scope of the work. Mention training in the visual modality.

% Safety concerns of robust = more capable agents. Finding security vulnerabilities and also used to red team
% ARicles are from public data sources
% Visual impaired users

\section{Environment \& Benchmark Details}
\label{sec:EnvironmentBenchmarkDetails}

We complement the benchmark overview in \S\ref{sec:timewarpBenchmark} by providing additional details, examples, and comparisons to better understand the \timewarp\ benchmark.
% ORIG: We first look into the observation (\S\ref{sec:ObservationSpace}) and action spaces (\S\ref{sec:actionSpace}), details about how the \timewarp\ dataset has been curated (\S\ref{sec:datasetCreation}), examples of the task goals (\S\ref{sec:goalExamples}), additional details on the environment (\S\ref{sec:envDetails}), versions (\S\ref{sec:versionCreation}, \ref{sec:variationsVersions}), and evaluation (\S\ref{sec:EvaluationDetails}).
{We first describe the observation (\S\ref{sec:ObservationSpace}) and action spaces (\S\ref{sec:actionSpace}), how the \timewarp\ dataset was curated (\S\ref{sec:datasetCreation}), examples of the task goals (\S\ref{sec:goalExamples}), and which goals require visual information (\S\ref{sec:visualInfoAnnotation}). We then provide additional details on the environments (\S\ref{sec:envDetails}), the versions and their validation (\S\ref{sec:versionCreation}, \ref{sec:variationsVersions}, \ref{sec:versionValidation}), and the evaluation (\S\ref{sec:EvaluationDetails}), and close by discussing how version robustness relates to generalization (\S\ref{sec:generalizationDiscussion}).}

\subsection{Observation Space}
\label{sec:ObservationSpace}

As previously mentioned in \S\ref{sec:problem}, \timewarp's observation space can be one of the four types: HTML (\html), accessibility tree (\axt), UI screenshots (\sshot), and Set of Marks (\som) \cite{yang2023setofMarks}. Most of these observations have been used in previous well-established benchmarks \cite{zhou2023webarena,koh2024visualwebarena,Drouin2024WorkArenaHCAW}. In our benchmark, observations are taken from the BrowserGym environment \cite{tmlr25dechezellesBrowserGym}, where \html, \axt, and \som\ are assigned BrowserGym IDs (BIDs) for easier interaction by the web agent. Examples of each observation used in our benchmark have been provided in Fig. \ref{fig:observationExamples}.

\lstset{
  language=HTML,
  basicstyle=\ttfamily\small,
  keywordstyle=\color{blue},
  stringstyle=\color{teal},
  commentstyle=\color{gray},
  showstringspaces=false
}

\begin{figure}[h]
\centering

% =====================
% Row 1 — Listings
% =====================

\begin{minipage}[]{0.48\linewidth}
\textbf{(a) HTML (\html)}
\vspace{0.3em}

\begin{lstlisting}[language=HTML]
<head bid="1">
 <meta bid="2" content="text/html;
 charset=utf-8" http-equiv=
 "Content-Type"/>
 <title bid="3"> 
    HomePage-Wikipedia
 </title>
</head> 
<h1 bid="6" visible="">
 <a bid="7" href="/" visible="">
 HomePage</a> </h1>
<a bid="8" href="/" visible="">
Home </a>
<p bid="9" visible="">
\end{lstlisting}

\end{minipage}
\hfill
\begin{minipage}[]{0.48\linewidth}
\textbf{(b) Accessibility Tree (\axt)}
\vspace{0.3em}

\definecolor{axheading}{RGB}{0,90,150}      % blue-ish
\definecolor{axlink}{RGB}{180,40,40}        % red-ish
\definecolor{axstatic}{RGB}{0,120,60}       % green-ish

\lstdefinelanguage{AXTree}{
  morekeywords=[1]{heading,paragraph,link},
  morekeywords=[2]{clickable},
  morekeywords=[3]{StaticText},
  keywordstyle=[1]\color{blue},
  keywordstyle=[2]\color{axlink},
  keywordstyle=[3]\color{axstatic},
  morestring=[b]',
  stringstyle=\color{gray},
}

% \begin{lstlisting}[language=AXTree]
% [6] heading 'Biophysics', visible
%     [7] link 'Biophysics', clickable, 
%             visible StaticText '['
%     [8] link 'Home', clickable, 
%             visible StaticText ']'
%     [9] paragraph '', visible
%         [10] link 'HomePage', clickable,
%                 visible StaticText '|'
%         [11] link 'RecentChanges', clickable,
%                 visible StaticText '|'
%         [12] link 'Preferences', clickable, 
%                 visible
%     [13] paragraph '', visible
%         [14] emphasis '', visible
%             StaticText 'You'
%                   StaticText 'can'
% \end{lstlisting}

\begin{lstlisting}[language=AXTree]
[6] heading 'HomePage', visible
   [7] link 'HomePage', 
   clickable, visible 
   StaticText '['
 [8] link 'Home', clickable, 
 visible StaticText ']'
 [9] paragraph '', visible
   [10] link 'HomePage', 
   clickable, visible 
   StaticText '|'
   [11] link 'RecentChanges', 
   clickable StaticText '|'
   [12] link 'Preferences', 
   clickable
\end{lstlisting}

\end{minipage}

\vspace{1em}

% =====================
% Row 2 — Images
% =====================

\begin{minipage}[]{0.48\linewidth}

\textbf{(c) Screenshot (\sshot)}
\centering
\vspace{0.3em}

\fbox{\includegraphics[width=\linewidth]{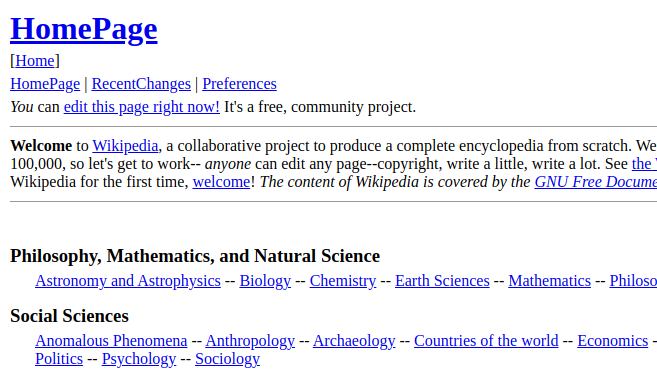}}

\end{minipage}
\hfill
\begin{minipage}[]{0.48\linewidth}
\textbf{(d) Set of Marks (\som)}
\centering

\vspace{0.3em}

\fbox{\includegraphics[width=0.995\linewidth]{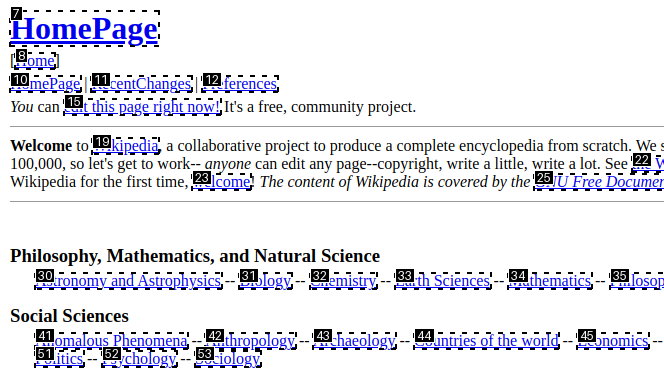}}

\end{minipage}

\caption{The textual (\html\ and \axt), and visual (\sshot\ and \som) observations in \timewarp.}
\label{fig:observationExamples}

\end{figure}

%% In the figure of SoM, use the one from the AgentXRay, which has more colors.

\subsection{Action Space}
\label{sec:actionSpace}
\timewarp's action space also follows previous well-established benchmarks \cite{zhou2023webarena,Drouin2024WorkArenaHCAW}. Table~\ref{tab:actionSpace} summarizes the complete set of actions available to our agents. % [REFINE: punctuation] Hyphens used as dashes replaced by commas and "since".
% ORIG: The actions include \textbf{general website interaction} -  a prerequisite to any web benchmark, \textbf{page and tab navigation} - as \timewarp\ tasks can span multiple websites and require backtracking, and \textbf{user interaction} - as the tasks require returning a final message to the user for evaluation.
{The actions include \emph{general website interaction}, a prerequisite for any web benchmark, \emph{page and tab navigation}, since \timewarp\ tasks can span multiple websites and require backtracking, and \emph{user interaction}, since the tasks require returning a final message to the user for evaluation.} The 
\texttt{report\_infeasible} action is included as a diagnostic action to identify unachievable tasks during dataset creation and can be removed during benchmarking. All actions use the default implementation provided by BrowserGym. The observation and action spaces can be coupled with their respective alignment or refinement strategies, such as AgentOccam \cite{yang2025agentoccam}.

\begin{table}[ht]
\centering
\caption{Action space of the \timewarp\  benchmark.}
\label{tab:actionSpace}
       \vspace{4pt}
\resizebox{0.9\textwidth}{!}{%
\begin{tabularx}{\textwidth}{@{} l l X @{}}
\toprule
\textbf{Type} & \textbf{Action} & \textbf{Description} \\
\midrule

\multirow{4}{*}{General}
& \texttt{scroll(dir)} &
Scroll up or down. \\

& \texttt{fill(id, value)} &
Fill the \texttt{id} field with a string \texttt{value}. \\

& \texttt{click(id, button)} &
Click \texttt{id} element with mouse option \texttt{button}. \\

& \texttt{press(id, key\_comb)} &
Press a key or key-combination (e.g., \texttt{Ctrl\ +\ a}) on \texttt{id} element. \\

\midrule
\multirow{2}{*}{Web Navigation}
& \texttt{go\_back()} &
Navigate to the previous page in history. \\

& \texttt{goto(url)} &
Navigate to a URL or target location. \\

\midrule
\multirow{2}{*}{User Interaction}
& \texttt{send\_msg\_to\_user(text)} &
Return the result by messaging the user and terminating. \\

& \texttt{report\_infeasible(text)} &
Mark the requested step/task as infeasible and provide a reason. \\

\midrule
\multirow{3}{*}{Tab Operations}
& \texttt{new\_tab()} &
Open a new tab and focus it. \\

& \texttt{tab\_close()} &
Close the current tab. \\

& \texttt{tab\_focus(index)} &
Switch focus to the \texttt{index} tab. \\
\bottomrule
\end{tabularx}
}
% \vspace{-0.85cm}
\end{table}

\subsection{Dataset Creation} 
\label{sec:datasetCreation}

Following \S\ref{sec:dataset}, all \timewarp\ tasks are manually annotated to evaluate
different aspects of agentic web navigation. % [REFINE: story] The environments are Wiki, News, and Shop; "web" was a slip. Order matched to the paper.
% ORIG: We first define task categories for each setting (\textit{web}, \textit{shop}, \textit{news}, and \textit{multi})
{We first define task categories for each setting (\textit{wiki}, \textit{news}, \textit{shop}, and \textit{multi})}, and draft example tasks for each category. For each task category, the annotators specify the task intent, including its relevance to evaluating robustness across multiple web versions and how that category can be challenging for an agent. These categories are iteratively refined through discussion before finalization. We then create training and test instances per category, ensuring broad coverage and similar task distributions across the splits. During instance creation, annotators interact with the web environment to verify the task's feasibility and annotate the desired outcome for evaluation. All annotations are subsequently verified, and erroneous or infeasible tasks are filtered or corrected. Distributional imbalances across train-test splits are also adjusted.

\subsection{Goal Examples}
\label{sec:goalExamples}
Following the discussion in \S\ref{sec:dataset}, \ref{sec:datasetCreation}, \timewarp\ comprises $1386$ task instances with $231$ unique goals. Tab. \ref{tab:datasetTasks} shows examples of task goals from some of the key categories. The tasks require comprehension of both visual and textual modalities. % ORIG: A good portion of the tasks require retrieval and navigation
{A large portion of the tasks requires retrieval and navigation}, which are quintessential to web agentic benchmarks. These tasks are complemented by specialized tasks that evaluate more fine-grained aspects of agents, such as hallucination traps \cite{li-etal-2023-halueval}, needle-in-a-haystack \cite{liu-etal-2024-lost}, and counting \cite{xu-ma-2025-llm-counting}. Multi-site tasks are also introduced to assess an agent's ability to select websites without explicit instructions and complete goals across multiple websites.

\begin{table*}[ht]
\centering
\caption{Example task goals and reference answers in the \timewarp\ dataset across environments and task categories. This is not a comprehensive list of task categories, but rather the key categories that we consider more important.}
\label{tab:datasetTasks}
\small
\setlength{\tabcolsep}{4pt}
\resizebox{\textwidth}{!}{%
\begin{tabular}{@{} 
L{1.5cm} 
L{3cm} 
L{9cm} 
L{5cm} 
@{}}
\toprule
\rowcolor{gray!15}\textbf{Env} &
\textbf{Category} &
\textbf{Example Task Goal} &
\textbf{Reference Answer} \\
\midrule

\rowcolor{teal!5}\mwrap{5}{Wiki} & \mwrap{5}{Multi-article Retrieval} &
\mwrap{5}{List all articles from the ``Related Pages'' sections of the Biophysics and Biochemistry pages. }&
List of biochemistry topics, Biophysics, Chromatography, Organic chemistry, Bionics, Computational biology, Database of Molecular Motions, Nanotechnology \\\midrule

\rowcolor{teal!5}\mwrap{2}{Wiki} & \mwrap{2}{Multi-step Navigation} &
Trace this path: Territory $\rightarrow$ first country mentioned $\rightarrow$ a country sharing its southern border. What is the last word of the first section? &
\mwrap{2}{world} \\\midrule

\rowcolor{teal!5}\mwrap{2}{Wiki} & \mwrap{2}{Recursive Navigation} &
Recursively follow the first link in the main content starting from the ``Technology'' page until a non-existent article is reached. &
\mwrap{2}{Technology, Skill, Experience, Learning, Knowledge, Fact}\\ \midrule

\rowcolor{teal!5}\mwrap{2}{Wiki} & \mwrap{2}{Fact Retrieval} &
Who is the current mayor of New York City, according to the article on NYC? &
\mwrap{2}{Eric Adams}\\\midrule

\rowcolor{teal!5}\mwrap{2}{Wiki} & \mwrap{2}{Needle in a Haystack} &
Which specialized tools are mentioned in the approaches section of the Physics article? &
\mwrap{2}{Particle Accelerators, Lasers }\\ \midrule

\rowcolor{red!5}\mwrap{2}{News} & \mwrap{2}{Search Count} &
How many articles were found on Richard Stallman (excluding redirect-only pages)? &
\mwrap{2}{Two} \\\midrule

\rowcolor{red!5}\mwrap{2}{News} & \mwrap{2}{Fact Verification} &
Is the publication year of the Tom Cruise psychiatry article consistent with the year of the event? &
\mwrap{2}{No} \\\midrule

\rowcolor{red!5}\mwrap{2}{News} & \mwrap{2}{Hallucination Trap} &
Name the article on the Indian prime minister testing positive for COVID-19. &
\mwrap{2}{No such article exists} \\\midrule

\rowcolor{red!5}\mwrap{2}{News} & \mwrap{2}{Comparison} &
Were the two specified football-related articles published on the same day? &
\mwrap{2}{Yes} \\ \midrule

\rowcolor{orange!5}\mwrap{1}{Shop} & \mwrap{1}{Product Verification} &
Is the Wii U Microphone compatible with the Wii system? &
\mwrap{1}{No} \\\midrule

\rowcolor{orange!5}Shop & Search Count &
How many sound bars are available under 150 USD? &
Two \\\midrule

\rowcolor{orange!5}\mwrap{2}{Shop} & \mwrap{2}{Order Placement} &
Find a small, portable, black music player for mini vinyl records and place an order if possible. &
\mwrap{2}{59EC38CAE8} \\\midrule

\rowcolor{orange!5}Shop & Visual Query &
What is the color of the TASYL USB lightning camera adapter? &
White \\\midrule

\rowcolor{orange!5}\mwrap{2}{Shop} & \mwrap{2}{Counting} &
How many times does ``baking soda'' appear on the Arm \& Hammer Toothpaste product page? &
\mwrap{2}{Eight times} \\ \midrule

\rowcolor{cyan!5}\mwrap{2}{Multi-Site} & \mwrap{2}{Multi-Site Query} &
Find the most recent sports event mentioned in Brazil’s Wikipedia article and a related news article. &
\mwrap{2}{Rio de Janeiro to host 2016 Olympics} \\\midrule

\rowcolor{cyan!5}\mwrap{2}{Multi-Site} & \mwrap{2}{Site Selection} &
Place an order for the cheapest available nuts and share the confirmation code. \textit{(starting website is Wiki)} &
\mwrap{2}{54F0DC47B1} \\\midrule

\rowcolor{cyan!5}\mwrap{2}{Multi-Site} & \mwrap{2}{Multi-Search Counting} &
Identify vegetables listed as sources of Vitamin K on Wiki and check their availability in the shop. &
\mwrap{2}{None} \\\midrule

\rowcolor{cyan!5}\mwrap{2}{Multi-Site} & \mwrap{2}{Chained Reasoning} &
Follow the links from the Wiki probability article, search the second link in the shop, then search its brand on Wiki. &
\mwrap{2}{Rusticware} \\

\bottomrule
\end{tabular}
}
% \vspace{-0.5cm}
\end{table*}

\subsection{Visual Information Requirements}
\label{sec:visualInfoAnnotation}
Web agents differ in whether they consume rendered screenshots or purely textual observations, \textit{i.e.}, HTML or the accessibility tree. To characterize how much of \timewarp\ genuinely depends on pixel-level information, % [REFINE: punctuation] Dashes removed.
% ORIG: we annotate each of the 231 goals as (a) \textit{Relevant} -- the rendered visual content (\textit{e.g.}, product images, colors, icons, or layout geometry) is required or helpful for solving the task -- or (b) \textit{Irrelevant} -- the task is fully solvable from textual observations.
{we annotate each of the 231 goals as either (a) \textit{Relevant}, meaning that the rendered visual content (\textit{e.g.}, product images, colors, icons, or layout geometry) is required or helpful for solving the task, or (b) \textit{Irrelevant}, meaning that the task is fully solvable from textual observations.} Each goal was independently labeled by two annotators using different framings, with disagreements settled by an adjudicator. The two annotators agree on $211/231$ ($91.3\%$) of the goals (Cohen's $\kappa=0.771$). The labels were reviewed by the authors and are released alongside the dataset.

As shown in Table~\ref{tab:visualInfo}, visual information is relevant for $59/231$ ($25.5\%$) of the goals and is most prevalent in Shop ($37.5\%$), which contains the only goals that strictly require inspecting product images, \textit{e.g.}, identifying the color of a product or the text printed on it. The majority of the goals ($74.5\%$) are solvable from textual observations alone, which contextualizes why text-based agents remain competitive with visual agents on \timewarp\ (\S\ref{sec:Results}).

\begin{table}[ht]
\centering
\caption{\textbf{Visual information requirements across the 231 goals of \timewarp.} A goal is \textit{Relevant} if rendered visual content is required or helpful for solving it, and \textit{Irrelevant} if it is fully solvable from textual observations.}
\label{tab:visualInfo}
\vspace{4pt}
\small
\begin{tabular}{lccc}
\toprule
\textbf{Environment} & \textbf{\# Goals} & \textbf{Relevant} & \textbf{Irrelevant} \\
\midrule
Wiki  & 70 & 20 (28.6\%) & 50 (71.4\%) \\
News  & 47 & 8 (17.0\%)  & 39 (83.0\%) \\
Shop  & 64 & 24 (37.5\%) & 40 (62.5\%) \\
Multi & 50 & 7 (14.0\%)  & 43 (86.0\%) \\
\midrule
\textbf{Total} & 231 & 59 (25.5\%) & 172 (74.5\%) \\
\bottomrule
\end{tabular}
\end{table}

\subsection{Environment Details}
\label{sec:envDetails}
We extend the discussion in \S\ref{sec:webEnvironment} and provide additional details behind the rationale of the \timewarp\ web {environments} (\S\ref{sec:environmentRationale}, \ref{sec:liveWebRationale}, \ref{sec:waybackRationale}), implementation details (\S\ref{sec:envImplementationDetails}, \ref{sec:webshopEnvModifcation}), and search algorithms (\S\ref{sec:searchAlgorithms}, \ref{sec:searchRationale}).

\subsubsection{Why do we need new environments?}
\label{sec:environmentRationale}
A common counterargument to creating new web environments, \textit{i.e.}, Wiki \& News, and a new dataset for \timewarp\ is that we could have created multiple versions of the existing containerized benchmarks by simply modifying their web environments. However, this approach has two key limitations. First, the realistic, general-purpose containerized benchmarks, \textit{e.g.}, WebArena \cite{zhou2023webarena}, REAL \cite{Garg2025REALBAA}, are distributed as fully Dockerized containers that include all dependencies. This form of packaging limits the extensibility of the environment and makes any form of modification non-trivial. Second, as mentioned in \S\ref{sec:timewarpBenchmark}, these environments are resource-intensive and require substantial disk storage. For instance, the Wikipedia environment in WebArena takes $89$ GB disk storage \cite{webarena-environment-docker-2023}. Creating six such Dockerized versions of Wikipedia would hence require $\sim0.5$ TB of disk storage, significantly reducing the convenience and practical usability of our benchmark. Moreover, at the time of writing, the Docker images for REAL were not publicly accessible, further constraining the modifiability of the environment.

\subsubsection{Why not benchmark on the live web?}
\label{sec:liveWebRationale}

Live web benchmarks, such as Web Voyager \cite{He2024WebVoyagerBAAF} and Mind2Web-Live \cite{pan2024webcanvas}, capture the natural dynamism of the internet and can, in principle, evaluate the robustness of models to change. However, as discussed in \S\ref{sec:introduction}, \ref{sec:relatedWork}, they present a few limitations, particularly when the goal is to assess performance against change. 

First, meaningful web changes typically unfold over extended periods of time, requiring longitudinal, {repeated experiments}. Benchmarking in this way is time-intensive and, essentially, impractical. The underlying data can also change, making the task outcome, which is treated as ground truth, unreliable. Second, {despite recent efforts} to improve evaluation on the live web \cite{Xue2025AnIOAH,gou2025mindweb}, the general reproducibility of live web results remains inherently constrained by uncontrollable and non-deterministic content updates. While judge agents \cite{gou2025mindweb} attempt to address this issue, their general evaluation accuracy across multiple realistic benchmarks has yet to be assessed. Finally, anti-bot measures, \textit{e.g.}, CAPTCHA, interfere with automated evaluation and are likely to become more restrictive, further challenging the feasibility of live web benchmarking.

% Live web benchmarks offer the dynamism of the internet, which can plausibly be used to assess a model's performance against change. However, as mentioned in \S\ref{sec:introduction},\ref{sec:relatedWork}, this has three key limitations: (i) Significant changes on the live web can only be observed after a long period of time, as website changes don't occur rapidly. Hence, for someone to assess performance against change, the experiments need to be conducted periodically, which is tedious and time-consuming. (ii) While recent studies have somewhat addressed the evaluation concerns on live web \cite{pan2024webcanvas}, the reproducibility of the benchmarks remains questionable, as the changes occur in an uncontrollable manner. Hence, containerizing multiple versions can enable us to assess performance under change without waiting long, while remaining controllable and reproducible. (iii) Finally, the internet deploys several anti-bot measures, such as CAPTCHA, which can hinder benchmarking on the live web. We also believe that these measures might become more robust over time, questioning the possibility of benchmarking on the live web.

\subsubsection{Can we benchmark on the Wayback Machine?}
\label{sec:waybackRationale}
The Wayback Machine \cite{InternetArchive} provides static historical snapshots of websites and
is conceptually similar to the temporal versions of websites introduced in \timewarp. However, benchmarking web agents in the Wayback Machine has a few limitations. First, snapshots capture only a finite set of pre-recorded states of a website, and navigating to unarchived or novel states results in errors, making it unsuitable for evaluating web agents. Second, access to the Wayback Machine is rate-limited, preventing large-scale or parallel benchmarking without explicit authorization from the hosts. These constraints limit its practical utility for systematic and scalable evaluation.

% \subsubsection{Why major redesigns instead of incremental updates?}
% \label{sec:macroMicroDiscussion}
% Real websites often change incrementally, but the changes most likely to affect an agent's performance are the major redesigns that \timewarp\ replicates. Modeling incremental micro-versions is also infeasible at benchmark scale: replacing each era with ${\sim}20$ micro-versions would turn our $103\times6=618$ test tasks into ${\approx}12.4$k test tasks (${\approx}37$k with 3 seeded runs), which is impractical as a standard benchmark. For reference, other web benchmarks range from ${\sim}100$ to ${\sim}2$k test tasks, \textit{e.g.}, WebArena \cite{zhou2023webarena} has $812$ and Online Mind2Web \cite{Xue2025AnIOAH} has $300$. \timewarp\ is nonetheless informative about the continuity of change: performance degrades with version distance, \textit{e.g.}, Qwen-3 4B trained on v$_1$ scores $23.0/19.1/16.5/15.2/9.7$ on v$_1$--v$_5$ (Tab.~\ref{tab:generalization}), \textit{i.e.}, failure scales with the amount of change, not only with its presence.

% in essence, is similar to our web environment, which spans across different time versions. The Wayback Machine also has a few limitations. Firstly, the static snapshots are across a fixed number of states of the website, and any attempt to traverse any novel state would result in an error, making it unsuitable to test web agents. Secondly, requests on the Wayback Machine are rate-limited, \textit{i.e.}, there is no easy way to benchmark multiple web agents in parallel without explicitly receiving permission from the site host. 

\subsubsection{Environment Implementation Details}
\label{sec:envImplementationDetails}

As mentioned in \S\ref{sec:timewarpBenchmark} and following previous works \cite{yao2022webshop}, all environments are hosted locally using Python-based Flask\footnote{\url{https://flask.palletsprojects.com/}} backends to ensure convenience and low-latency interaction. As our data store is {large}, we use aggressive caching strategies to reduce latency during web interaction:
\begin{itemize}
    \item \textbf{Wiki \& News:} MediaWiki XML dumps are pre-parsed at indexing time and cached via \texttt{pickle} serialization, resulting in response times of $\sim50$ms per action.
    \item \textbf{Shop:} Uses Webshop's Apache Lucene via \texttt{pyserini}\footnote{\url{https://pypi.org/project/pyserini/}} for $<100$ms search latency across $1.18$ million items.
\end{itemize}

\subsubsection{Webshop Environment Modification}
\label{sec:webshopEnvModifcation}
% ORIG: The Shop environment is a modification of the popular WebShop environment \cite{yao2022webshop}, by implementing the following changes:
{The Shop environment modifies the popular WebShop environment \cite{yao2022webshop} in two ways:}

\begin{enumerate}

    \item \textbf{Removal of Oracle Feeds:} Unlike the original WebShop, which displayed rewards and task instructions as elements within the webpage, our implementation removes these fields entirely. % [REFINE: connection] Verifiers, not the judge, are now the primary evaluator (\S\ref{sec:EvaluationDetails}); tense made present.
    % ORIG: Agents will rely solely on the textual task instruction provided in the prompt, and the reward will be provided by the judge evaluator. 
    {Agents rely solely on the textual task instruction provided in the prompt, and the reward is assigned by the evaluator (\S\ref{sec:EvaluationDetails}).} 

    \item \textbf{Purchase Verification:} Upon completion of an order, the website generates a unique purchase confirmation code derived from a deterministic hash $H(p, o)$ of the product ID $p$ and selected option $o$. % ORIG: The agent will use the code from the confirmation page and send it to the user to confirm order completion.
    {The agent reads the code from the confirmation page and sends it to the user to confirm order completion.}

\end{enumerate}

\subsubsection{Search Algorithms}
\label{sec:searchAlgorithms}

Each environment exposes a different search backend designed to replicate the retrieval behavior of contemporary web systems. Concretely, we use (i) a substring-based lookup for Wiki, (ii) a BM25-style inverted index for News, and (iii) a Lucene/Pyserini BM25 product search for Shop. As previously mentioned in \S\ref{sec:version}, the substring-based search of Wikipedia is version-dependent, \textit{i.e.}, older versions do not show the substring results as drop-down options, essentially restricting the search to an exact match search algorithm, representative of the encyclopedic sites of that era.

\subsubsection{Why different search algorithms?} 
\label{sec:searchRationale}
Search functionality varies substantially across real-world websites and depends on the underlying use case. For instance, the search on an encyclopedic platform such as Wikipedia differs fundamentally from product search on e-commerce platforms like Amazon. The inclusion of different search algorithms enables a more realistic evaluation of a web agent's robustness and retrieval capabilities. 

\subsection{Version Creation Details}
\label{sec:versionCreation}
As discussed in \S\ref{sec:version} and illustrated in Fig.~\ref{fig:datasetStatistics}, each version is modeled after a distinct historical phase of the web. We identified key design inflection points for each reference site, \textit{e.g.}, 2001 for Wikipedia, 2016 for BBC News, using resources such as the Web Design Museum\footnote{\url{https://www.webdesignmuseum.org/}}. While replicating the UI, we referenced archival snapshots from the Internet Archive\footnote{\url{https://archive.org/}}.
A notable exception was the 2001 Wikipedia theme, as individual webpages were unavailable in the Wayback Machine. We reconstructed the article template by extrapolating the design language of Wikipedia's main landing page from 2001. Each version also {uses} era-specific front-end code, which introduces meaningful temporal variation in the textual space, \textit{i.e.}, with the HTML and accessibility tree.

\subsection{Variations across Versions}
\label{sec:variationsVersions}
Following \S\ref{sec:timewarpBenchmark}, the \timewarp\ environments incorporate feature variations across versions to reflect the evolution of web interaction. Early versions emphasize minimal functionality, lacking several quality-of-life features. For example, we observe in Fig.~\ref{fig:tableofcontent} that Wiki v$_1$ contains long, unstructured pages without a table of contents, whereas later versions (e.g., v$_6$) introduce a ToC for easier navigation. Similarly, search functionality also evolves: early Wiki search relies on exact string matching, while later versions provide drop-down suggestions based on substring matching (Fig.~\ref{fig:searchDifferences}).
Modern versions introduce additional interaction complexity. For instance, Fig.~\ref{fig:newsSearch} shows that News v$_5$ introduces icon-based search functionalities, requiring agents to first identify and click the search icon before typing the search query. Modern layouts also increase the UI clutter, introducing bloat, such as pop-up advertisements (Fig.~\ref{fig:popUpAdsErrorQualAnalysis}), further challenging agentic navigation.

\begin{figure}[h]
    \centering
    \includegraphics[width=0.995\linewidth]{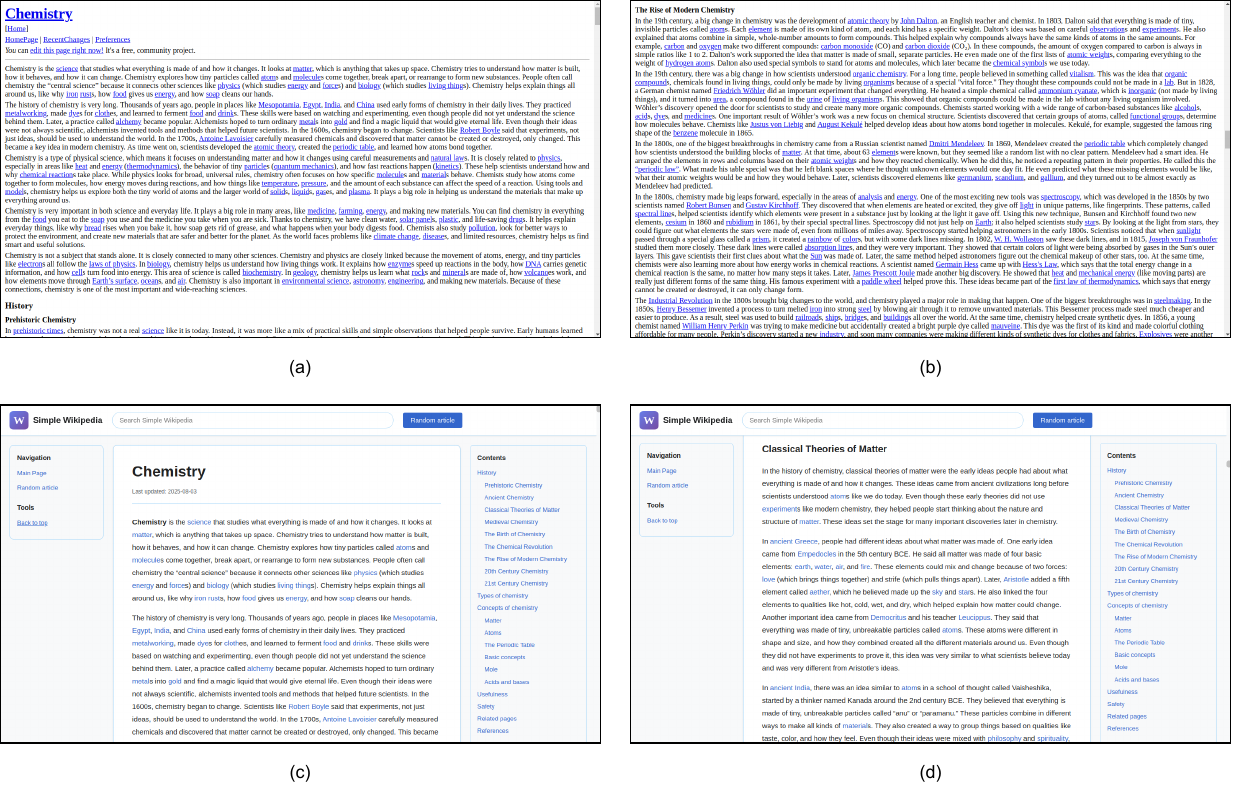}
    \caption{\textbf{Navigation features in legacy vs. modern environments.} (a,b) Wiki v$_1$ lacks a table of contents, requiring models to rely on scrolling to navigate long pages, whereas (c,d) Wiki v$_6$ provides a table of contents for efficient navigation. \textit{While the figure is not easily readable at this scale, full-size screenshots are provided to illustrate the volume of content and the resulting difficulty in navigation.}}
    \label{fig:tableofcontent}
\end{figure}

\begin{figure}[h]
    \centering
    \includegraphics[width=0.995\linewidth]{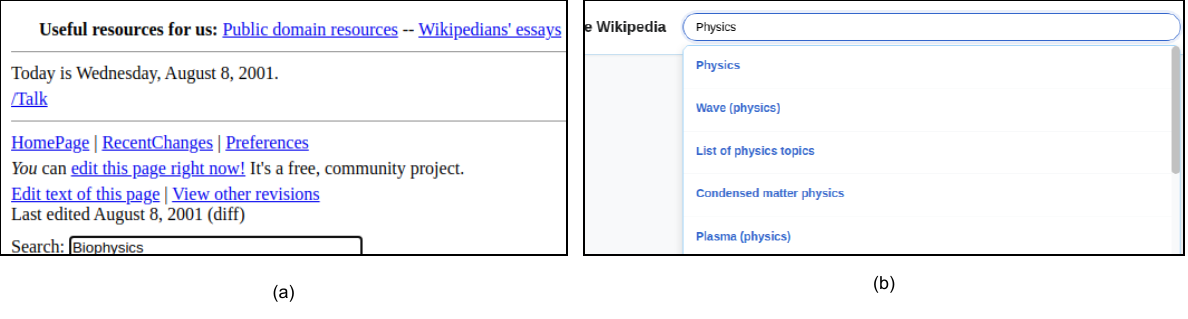}
    \caption{\textbf{Variations in Search across Versions.} (a) Wiki v$_1$'s search box is located at the bottom of the page and doesn't provide any drop-down suggestions, whereas (b) v$_6$'s search is at the standard top navigation bar and provides drop-down suggestions.}
    \label{fig:searchDifferences}
\end{figure}

\begin{figure}[h]
    \centering
    \includegraphics[width=0.995\linewidth]{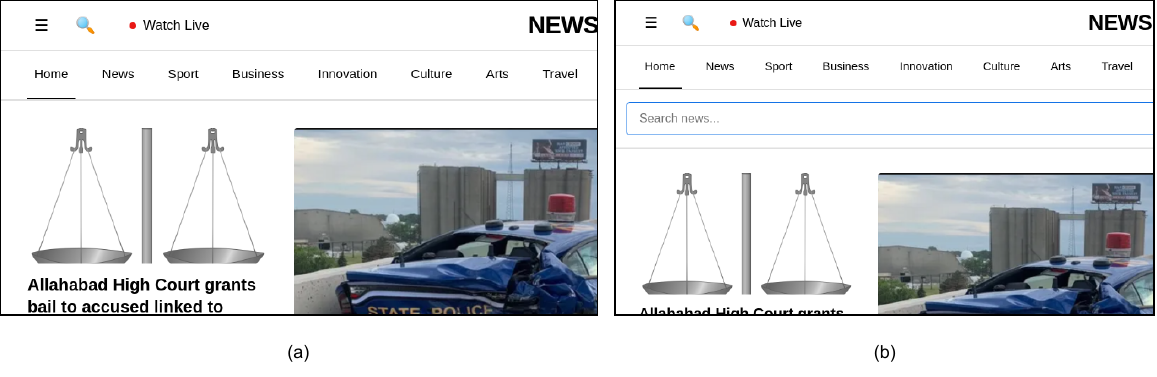}
    \caption{\textbf{Difficulties in the modern internet.} To search in News v$_5$, agents need to click on the search icon first, and then fill the search box. This is not as intuitive as directly filling the search box and often results in errors.}
    \label{fig:newsSearch}
\end{figure}

\subsection{Quantitative Validation of Version Realism}
\label{sec:versionValidation}
A natural concern is whether the six versions represent real temporal web evolution rather than arbitrary handcrafted UI variation. Because all versions of an environment render the same backend content through swapped front-end themes, measuring rendered pages across versions isolates the effect of the UI era in a controlled manner. We measure every version of every environment with headless Chromium at the benchmark's $1280\times720$ viewport, covering 246 page measurements (9, 13, and 19 pages per version for Wiki, News, and Shop, respectively) and multiple metrics per page, spanning DOM structure, the accessibility tree exactly as flattened for the agent's observation (with tokens counted using the \texttt{cl100k\_base} encoding), interactive elements, screenshot statistics, and web-technology markers, \textit{e.g.}, CSS rule counts, media queries, responsive viewport meta tags, and the share of table-based vs. flex/grid layout. As v$_6$ is a minimal baseline theme rather than a historical era (\S\ref{sec:version}), it is excluded from trend statistics and reported as a control.

\noindent\textbf{Versions are quantitatively distinct.} As shown in Table~\ref{tab:versionValidation} (top), the versions are far from cosmetic re-skins: over identical content, the agent's action space on Shop varies by $8.0\times$ (4.4 to 34.7 interactive elements) and its accessibility-tree observation by $4.0\times$ (769 to 3{,}069 tokens). Wiki is the deliberately low-variance environment ($1.1$--$1.4\times$), so per-environment robustness results should be interpreted per environment rather than averaged.

\noindent\textbf{Era markers track documented web history.} Table~\ref{tab:versionValidation} (bottom) and Fig.~\ref{fig:eraMarkers} show that the markers whose real-world history is well documented trend strongly and significantly with the era-ordered versions in all three environments (all $p<0.001$): CSS rule counts rise, responsive design (viewport meta tags and media queries) switches on at v$_4$ for Wiki and News and at v$_3$ for Shop, matching its real-world arrival, and layouts transition from purely table-based to flex/grid (layout modernity rising from 0.00 to 0.95--0.98). Overall, 8/17 (Wiki), 16/18 (News), and 17/19 (Shop) of the measured metrics exhibit significant ($p<0.05$) era trends.

\noindent\textbf{Environments diverge as their real counterparts did.} Shop grows denser and more visual over time (DOM elements 107 to 248, above-the-fold image area 1\% to 23\%), News becomes structurally lighter yet more interactive (DOM elements fall 220 to 127 while interactive elements rise 33 to 63), mirroring the real shift from dense table-based link portals to modern component-based designs, and Wiki's article structure remains stable while only its chrome modernizes, consistent with Wikipedia's conservative design history. Shop v$_1$--v$_3$ additionally render product details as separate server-rendered sub-pages, whereas v$_4$--v$_5$ collapse them into in-page tabs, reproducing the web's transition from server-rendered pages to single-page applications and materially changing the agent's task structure. Such coordinated, domain-specific divergence is unlikely to arise from arbitrary handcrafted variation. Full per-version breakdowns of the headline metrics are provided in Tables~\ref{tab:versionMetricsWiki}, \ref{tab:versionMetricsNews}, and \ref{tab:versionMetricsShop}, and the workflow (trajectory) length across versions is reported in Fig.~\ref{fig:trajectoryPerformanceDiagram}. We note that this analysis establishes that the versions vary along historically plausible axes, and we do not claim exact fidelity to specific years (\textit{cf.} \S\ref{sec:limitations}).

\begin{figure}[ht]
    \centering
    \includegraphics[width=\linewidth]{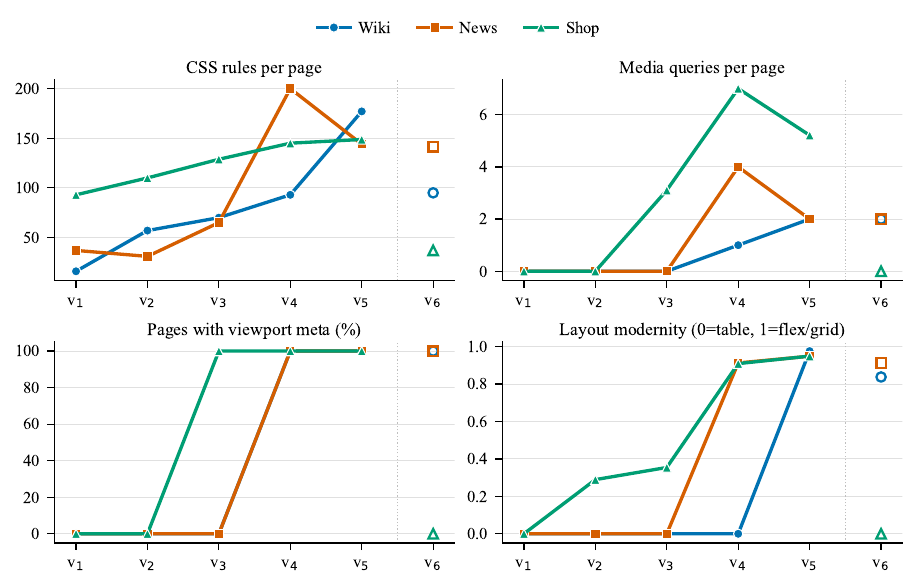}
    \caption{\textbf{Web-technology era markers across versions.} Per-page means of four markers with well-documented real-world trajectories, measured on the rendered pages of each environment version. All markers evolve in the historically expected direction in all three environments. v$_6$ (open markers) is the minimal baseline control, not a historical era.}
    \label{fig:eraMarkers}
\end{figure}

\begin{table}[ht]
\centering
\caption{\textbf{Quantitative validation of the \timewarp\ versions.} Top: variation of the agent-facing observation and action spaces across the six versions of each environment (per-page means over identical backend content). Bottom: Spearman correlation between the era-ordered versions v$_1$--v$_5$ and web-technology markers, computed over individual pages (all $p<0.001$).}
\label{tab:versionValidation}
\vspace{4pt}
\small
\setlength{\tabcolsep}{5pt}
\resizebox{\linewidth}{!}{%
\begin{tabular}{lccc}
\toprule
 & \textbf{Wiki} & \textbf{News} & \textbf{Shop} \\
\midrule
\multicolumn{4}{l}{\textit{Agent-facing variation (min $\rightarrow$ max across versions)}} \\
AX observation tokens & 21{,}837 $\rightarrow$ 24{,}663 ($1.1\times$) & 2{,}015 $\rightarrow$ 4{,}195 ($2.1\times$) & 769 $\rightarrow$ 3{,}069 ($4.0\times$) \\
Interactive elements  & 329 $\rightarrow$ 398 ($1.2\times$) & 19 $\rightarrow$ 63 ($3.3\times$) & 4.4 $\rightarrow$ 34.7 ($8.0\times$) \\
DOM elements          & 480 $\rightarrow$ 655 ($1.4\times$) & 106 $\rightarrow$ 263 ($2.5\times$) & 49 $\rightarrow$ 248 ($5.1\times$) \\
\midrule
\multicolumn{4}{l}{\textit{Era-marker trends (Spearman $\rho$, v$_1\rightarrow$v$_5$)}} \\
CSS rules            & 1.00 & 0.80 & 0.90 \\
Media queries        & 0.89 & 0.78 & 0.87 \\
Responsive viewport meta & 0.87 & 0.87 & 0.87 \\
Table $\rightarrow$ flex/grid layout & 0.70 & 0.87 & 0.87 \\
\bottomrule
\end{tabular}%
}
\end{table}

\begin{table}[ht]
\centering
\caption{\textbf{Per-version page metrics for the Wiki environment} (per-page means). v$_6$ is the minimal baseline control, not a historical era.}
\label{tab:versionMetricsWiki}
\vspace{4pt}
\small
\setlength{\tabcolsep}{4pt}
\resizebox{\linewidth}{!}{%
\begin{tabular}{lrrrrrr}
\toprule
\textbf{Metric} & \textbf{v$_1$ (2001)} & \textbf{v$_2$ (2002)} & \textbf{v$_3$ (2003--04)} & \textbf{v$_4$ (2005--22)} & \textbf{v$_5$ (2023--25)} & \textbf{v$_6$ (min.)} \\
\midrule
DOM elements & 479.7 & 571.3 & 594.7 & 644.2 & 654.9 & 509.2 \\
DOM nodes (+text) & 1,192 & 1,345 & 1,375 & 1,407 & 1,446 & 1,201 \\
DOM max depth & 3.11 & 10.00 & 10.67 & 8.67 & 8.67 & 8.89 \\
Distinct tags & 17.00 & 27.00 & 27.78 & 23.22 & 28.11 & 23.67 \\
HTML bytes & 37,725 & 42,992 & 45,354 & 46,606 & 48,193 & 41,247 \\
AXT nodes & 1,958 & 2,163 & 2,193 & 2,207 & 2,349 & 2,041 \\
AXT interactive nodes & 332.2 & 386.3 & 383.6 & 385.6 & 398.4 & 329.1 \\
AXT max depth & 4.11 & 10.00 & 11.44 & 9.44 & 8.22 & 7.67 \\
AXT observation tokens & 21,837 & 23,332 & 23,540 & 23,749 & 24,663 & 22,388 \\
Interactive elements & 332.2 & 366.4 & 384.4 & 385.6 & 398.4 & 329.1 \\
Links & 331.2 & 360.4 & 378.4 & 380.6 & 382.6 & 327.0 \\
Buttons & 0.00 & 3.00 & 4.00 & 4.00 & 6.89 & 1.00 \\
Forms & 1.00 & 3.00 & 2.00 & 1.00 & 1.00 & 0.00 \\
Ink fraction & 0.12 & 0.11 & 0.10 & 0.12 & 0.08 & 0.07 \\
Edge density & 14.97 & 13.65 & 6.93 & 7.01 & 6.67 & 5.71 \\
Palette colors & 65.89 & 290.6 & 255.3 & 351.1 & 217.3 & 225.8 \\
Image area fraction & 0.00 & 0.02 & 0.02 & 0.03 & 0.01 & 0.00 \\
Table elements & 0.00 & 2.00 & 1.44 & 0.56 & 0.00 & 0.00 \\
Layout modernity & 0.00 & 0.00 & 0.00 & 0.00 & 0.98 & 0.84 \\
CSS rules & 16.00 & 57.00 & 70.00 & 93.00 & 177.0 & 95.00 \\
\bottomrule
\end{tabular}%
}
\end{table}

\begin{table}[ht]
\centering
\caption{\textbf{Per-version page metrics for the News environment} (per-page means). v$_6$ is the minimal baseline control, not a historical era.}
\label{tab:versionMetricsNews}
\vspace{4pt}
\small
\setlength{\tabcolsep}{4pt}
\resizebox{\linewidth}{!}{%
\begin{tabular}{lrrrrrr}
\toprule
\textbf{Metric} & \textbf{v$_1$ (2000s)} & \textbf{v$_2$ (2004)} & \textbf{v$_3$ (2008)} & \textbf{v$_4$ (2016)} & \textbf{v$_5$ (2024)} & \textbf{v$_6$ (min.)} \\
\midrule
DOM elements & 220.3 & 258.8 & 263.1 & 127.8 & 127.5 & 106.2 \\
DOM nodes (+text) & 287.2 & 337.7 & 346.5 & 192.8 & 210.1 & 155.3 \\
DOM max depth & 15.38 & 12.31 & 15.54 & 7.31 & 7.15 & 8.69 \\
Distinct tags & 20.31 & 20.54 & 20.85 & 22.92 & 21.00 & 25.54 \\
HTML bytes & 12,453 & 16,473 & 18,249 & 9,783 & 9,860 & 11,225 \\
AXT nodes & 309.0 & 378.9 & 394.2 & 242.6 & 272.4 & 173.2 \\
AXT interactive nodes & 32.62 & 47.69 & 48.54 & 40.00 & 62.31 & 19.23 \\
AXT max depth & 12.31 & 10.15 & 12.31 & 6.77 & 6.62 & 7.31 \\
AXT observation tokens & 4,084 & 4,138 & 4,195 & 2,456 & 2,673 & 2,015 \\
Interactive elements & 32.62 & 47.69 & 48.23 & 41.00 & 63.31 & 19.23 \\
Links & 30.62 & 45.69 & 46.08 & 39.00 & 59.31 & 16.62 \\
Buttons & 1.00 & 1.00 & 1.08 & 0.00 & 3.00 & 1.62 \\
Forms & 1.00 & 1.00 & 1.00 & 0.00 & 0.00 & 0.00 \\
Ink fraction & 0.23 & 0.36 & 0.16 & 0.16 & 0.09 & 0.09 \\
Edge density & 9.18 & 8.29 & 7.52 & 5.76 & 5.50 & 5.15 \\
Palette colors & 200.1 & 356.7 & 365.2 & 349.4 & 380.8 & 133.2 \\
Image area fraction & 0.00 & 0.00 & 0.01 & 0.02 & 0.00 & 0.00 \\
Table elements & 11.77 & 14.46 & 15.62 & 0.00 & 0.00 & 0.00 \\
Layout modernity & 0.00 & 0.00 & 0.00 & 0.91 & 0.95 & 0.91 \\
CSS rules & 37.00 & 31.00 & 65.00 & 200.0 & 144.0 & 141.5 \\
\bottomrule
\end{tabular}%
}
\end{table}

\begin{table}[ht]
\centering
\caption{\textbf{Per-version page metrics for the Shop environment} (per-page means). v$_6$ is the minimal baseline control, not a historical era.}
\label{tab:versionMetricsShop}
\vspace{4pt}
\small
\setlength{\tabcolsep}{4pt}
\resizebox{\linewidth}{!}{%
\begin{tabular}{lrrrrrr}
\toprule
\textbf{Metric} & \textbf{v$_1$ (2000)} & \textbf{v$_2$ (2005)} & \textbf{v$_3$ (2010)} & \textbf{v$_4$ (2015)} & \textbf{v$_5$ (2025)} & \textbf{v$_6$ (classic)} \\
\midrule
DOM elements & 106.8 & 94.74 & 122.9 & 132.6 & 247.8 & 49.05 \\
DOM nodes (+text) & 155.6 & 126.8 & 169.4 & 196.4 & 339.0 & 60.32 \\
DOM max depth & 8.47 & 8.74 & 9.21 & 8.84 & 15.00 & 9.00 \\
Distinct tags & 17.74 & 21.42 & 24.84 & 23.37 & 28.16 & 10.95 \\
HTML bytes & 10,462 & 8,544 & 13,726 & 18,080 & 41,324 & 5,464 \\
AXT nodes & 179.7 & 142.7 & 173.9 & 182.1 & 303.8 & 62.58 \\
AXT interactive nodes & 23.79 & 20.74 & 22.05 & 28.05 & 34.89 & 4.37 \\
AXT max depth & 6.37 & 8.37 & 7.68 & 6.00 & 8.05 & 5.95 \\
AXT observation tokens & 1,928 & 1,515 & 1,858 & 2,055 & 3,069 & 769.4 \\
Interactive elements & 23.79 & 20.74 & 21.42 & 25.00 & 34.74 & 4.37 \\
Links & 14.05 & 13.95 & 7.74 & 12.79 & 6.16 & 1.58 \\
Buttons & 9.37 & 6.68 & 10.63 & 10.21 & 24.89 & 2.74 \\
Forms & 9.37 & 6.47 & 5.63 & 4.95 & 13.42 & 2.74 \\
Ink fraction & 0.11 & 0.24 & 0.13 & 0.27 & 0.41 & 0.06 \\
Edge density & 3.90 & 3.71 & 3.51 & 3.81 & 2.88 & 2.47 \\
Palette colors & 281.4 & 397.1 & 950.9 & 1,444 & 2,040 & 281.9 \\
Image area fraction & 0.01 & 0.02 & 0.06 & 0.20 & 0.23 & 0.05 \\
Table elements & 2.89 & 1.37 & 0.00 & 0.00 & 0.00 & 0.00 \\
Layout modernity & 0.00 & 0.29 & 0.35 & 0.91 & 0.95 & 0.00 \\
CSS rules & 93.00 & 110.0 & 128.8 & 145.0 & 148.7 & 37.00 \\
\bottomrule
\end{tabular}%
}
\end{table}

\subsection{Evaluation Details}
\label{sec:EvaluationDetails}
We expand upon the discussion in \S\ref{sec:evaluation} by providing additional details on \timewarp's evaluation framework, including details on the deterministic verifiers (\S\ref{sec:deterministicVerifierDetails}), their annotation and validation (\S\ref{sec:verifierAnnotation}), the judge evaluator (\S\ref{sec:judgeEvaluatorDetails}), assessment of multiple correct answers (\S\ref{sec:multipleCorrectAnswers}), evaluation of our judge evaluator (\S\ref{sec:evaluationJudgeEval}), and self-evaluation bias (\S\ref{sec:selfEvalBiasJudge}). We use success rate (\%) as our evaluation metric, which is standard across web benchmarks. Success rate is defined as the proportion of tasks completed successfully, and the criteria for success are determined by the deterministic verifiers, with the LLM judge evaluator serving as a fallback for subjective tasks.

\subsubsection{Deterministic Verifier Details}
\label{sec:deterministicVerifierDetails}
Each task in \timewarp\ specifies a deterministic verifier composed of five function classes, similar to WebArena's evaluation functions \cite{zhou2023webarena}: \texttt{exact\_match}, \texttt{must\_include}, and \texttt{must\_exclude} for string matching, \texttt{number\_match} for numeric answers, and \texttt{list\_match} for ordered and unordered lists. When a task combines multiple classes, the verifiers are conjoined (\texttt{AND} logic). As shown in Table~\ref{tab:evalTypes}, the deterministic verifiers cover $229/231=99.1\%$ of the goals. The remaining two goals require a comparative, subjective judgment and fall back to the LLM judge (\S\ref{sec:judgeEvaluatorDetails}).

\begin{table}[ht]
\centering
\caption{\textbf{Verifier coverage across the 231 goals of \timewarp.} Deterministic verifiers cover $229/231=99.1\%$ of the goals. Combined types conjoin multiple verifiers (\texttt{AND} logic). Only two subjective goals fall back to the LLM judge.}
\label{tab:evalTypes}
\vspace{4pt}
\small
\begin{tabular}{lr}
\toprule
\textbf{Evaluator} & \textbf{\# Goals} \\
\midrule
String match (\texttt{exact\_match}, \texttt{must\_include}, \texttt{must\_exclude}) & 156 \\
Number match & 39 \\
List match & 26 \\
String + number match & 5 \\
String + list match & 3 \\
LLM judge (subjective) & 2 \\
\midrule
\textbf{Total} & 231 \\
\bottomrule
\end{tabular}
\end{table}

All matching is performed on normalized answers: Unicode folding (accents, typographic quotes, and dashes), case folding, whitespace collapsing, and stripping of Markdown and edge punctuation make the verifiers robust to formatting variations in the agent's final answer. String containment is checked on word boundaries, preventing partial matches such as ``10'' matching within ``100''. Number matching parses formatted numerals, magnitude words, and spelled-out cardinals (\textit{e.g.}, ``57.7 million'', ``13th'', ``thirteen'') and requires exact equality unless the task explicitly specifies a tolerance. Verifier specifications additionally support alternative valid answers and regular-expression patterns (\S\ref{sec:multipleCorrectAnswers}), as well as a first-sentence scope that isolates yes/no verdicts from subsequent explanations. Malformed specifications and evaluator failures raise errors instead of silently scoring $0$, ensuring that infrastructure issues are never recorded as task failures.

To validate the deterministic verifiers, we score the collected teacher trajectories with both the verifiers and the GPT-5.1 judge, observing a \textbf{99.87\% agreement rate}. Each task also retains its original reference answer, so judge-based and verifier-based scoring remain directly comparable. Beyond reliability, deterministic verifiers are cheap, exactly reproducible, and free of API dependence and sampling variance, and can therefore provide a fast and reliable reward signal in training-heavy reinforcement learning pipelines such as GRPO \cite{shao2024deepseekmathpushinglimitsmathematical}.

\subsubsection{Verifier Annotation and Validation}
\label{sec:verifierAnnotation}
To author the verifier specifications, the 231 goals were split into 16 batches, stratified by answer shape and website, so that no batch contained a block of near-identical tasks. For each goal, an annotator drafts a specification from the goal, the gold answer, and the underlying environment content, together with three correct example answers of varying verbosity and formatting and three realistic near-miss incorrect answers. Every draft is then adversarially attacked by an independent critic, which attempts to construct answers that the specification falsely accepts or falsely rejects, and the specification is repaired accordingly. Specification drafting and critique were performed by LLM agents, with all flagged cases reviewed by the authors (\textit{cf.} \S\ref{sec:llmUsageDeclaration}).

Crucially, every specification is validated by execution rather than inspection. An automated self-check requires that (i) the gold answer scores $1$, (ii) all positive example answers score $1$, (iii) all near-miss negative answers score $0$, and (iv) the specification rejects the gold answers of other tasks, ensuring that it is discriminative rather than overly lenient. A complementary dataset-wide structural audit probes common failure classes, such as contrastive answers that mention an excluded phrase while still giving the correct answer, and specifications that merely echo tokens of the task goal. During this process, five gold answers were corrected against the environment content, and two goals were identified as genuinely subjective and left to the LLM judge (\S\ref{sec:judgeEvaluatorDetails}). Annotators assign a confidence level to each specification (131 high, 97 medium, 3 low), and every task retains free-text annotation notes documenting the evidence checked, keeping the full provenance of the evaluation auditable.

\subsubsection{Judge Evaluator Details}
\label{sec:judgeEvaluatorDetails}
For the two subjective goals not covered by the deterministic verifiers, and to validate the verifiers themselves, we use GPT-5.1 as an LLM judge \cite{son2024llm}, following recent benchmarks \cite{Xue2025AnIOAH,wei2025browsecomp}. The evaluator receives the task goal, a reference answer, and the agent's predicted answer, and outputs one of three labels: \textit{correct}, \textit{partially correct}, or \textit{incorrect}. Only \textit{correct} is assigned a reward of 1, and the remaining labels receive 0. The full prompt is provided in \S\ref{sec:prompt}. As LLM judges assess answers at a semantic level, they offer a robust complement to the deterministic verifiers for ambiguous or subjective answers.
We additionally support Gemma-4 12B \cite{team2026gemma} as an open-source judge, thereby removing dependence on a proprietary API. The open-source judge reaches a \textbf{99.74\% agreement rate} with the GPT-5.1 judge on the teacher trajectories, making it a reliable drop-in replacement.

\subsubsection{Multiple Correct Answers}
\label{sec:multipleCorrectAnswers}
Both evaluation paths support multiple valid answers. In deterministic verifier specifications, alternative acceptable answers are separated with an \texttt{|OR|} delimiter, following WebArena \cite{zhou2023webarena}, and individual entries may be anchored regular expressions to capture open-ended phrasings. \texttt{AND} logic is expressed by conjoining multiple verifier functions within a single task (\S\ref{sec:deterministicVerifierDetails}). For the LLM judge, the reference answer can be provided as a list of acceptable strings instead of a single string. % ORIG: This would prompt the judge evaluator to evaluate each reference answer independently
{The judge then evaluates each reference answer independently}, \textit{i.e.}, if there are $n$ reference answers, there will be $n$ calls. If at least one of the answers is deemed valid, the judge will return $1$, corresponding to the \texttt{OR} logic between the validity of the reference answers. For judge-scored tasks that require all conditions to be satisfied, \textit{i.e.}, \texttt{AND} logic, the references should be concatenated into a single string, which will be treated as a single reference answer by the judge and evaluated in a single call.

% . If the users want all the answers to match (\texttt{AND} logic), then they should provide all the reference answers in a single string, which can be evaluated by the judge in a single call. 

\begin{table*}[ht!]
\centering
\caption{Test cases across different categories to evaluate the Judge Evaluator. The human verdict indicates whether the human verifier agrees with the judge’s reward (\textbf{R}). The only disagreement occurs in a tricky case in which ordering is implied but not explicitly stated.}
\label{tab:judge_results}
\setlength{\tabcolsep}{4pt}
\small
\resizebox{\textwidth}{!}{%
\begin{tabular}{@{} L{1.5cm} L{5.7cm} L{2cm} L{2.5cm} C{0.2cm} C{1.1cm} @{}}
\toprule
\mwrap{2}{\textbf{Eval. Type}} & \mwrap{2}{\textbf{Example Task Goal}} & \textbf{Reference Answer} & \mwrap{2}{\textbf{Candidate Answer}} & \mwrap{2}{\textbf{R}} & \textbf{Human Verdict} \\
\midrule
\mwrap{5}{Exact String\\Matching} &
\mwrap{5}{Order a pair of jeans that are made of at least 60\% cotton. If there are multiple options, pick any. The size must be M. Once you complete the order, share the completion code.} &
\mwrap{5}{59EC38CAE8} &
C97A8FE47F & 1 & \cmark \\
\cmidrule{4-6}
 & & & Confirmation code: C97A8FE47F & \mwrapCenter{2}{1} & \mwrapCenter{2}{\cmark} \\
\cmidrule{4-6}
 & & & C97B8FE47F & 0 & \cmark \\
\midrule
\mwrap{2}{Number\\Matching} &
\mwrap{2}{How many sound bars are available under 150 USD?} &
\mwrap{2}{Two} &
2 & 1 & \cmark \\
\cmidrule{4-6}
 & & & 3 & 0 & \cmark \\
\midrule
\mwrap{5}{Estimate\\Matching} &
\mwrap{5}{Find the largest city listed on the Wiki pages for Canada and Australia. Then tell me the population difference between those two cities.} &
\mwrap{5}{Around 2.5 million} & Around 2.55 million & 1 & \cmark \\
\cmidrule{4-6}
 & & & 2,533,124 & 1 & \cmark \\
\cmidrule{4-6}
 & & & 2 million & 0 & \cmark \\
\midrule
\mwrap{3}{Logic\\Matching} &
\mwrap{3}{Does the article on Sociology discuss its relationship with Mathematics?} &
\mwrap{3}{No, it does not.} &
No & 1 & \cmark \\
\cmidrule{4-6}
 & & & Yes & 0 & \cmark \\
\midrule
\mwrap{7}{Intent\\Matching} &
\mwrap{7}{Does the ``Once Upon A Time Queen Born In 1982 T-Shirts'' product from the Shop contain polyester? If so, go to the Polyester article on Wiki and tell me the names of the clothing brands mentioned there. Otherwise, just tell me, ``T-shirt is polyester-free!!!''} &
\mwrap{7}{No clothing brands are mentioned in the Polyester article.} &
The Polyester article doesn't mention any clothing brand. & \mwrapCenter{3}{1} & \mwrapCenter{3}{\cmark} \\
\cmidrule{4-6}
 & & & T-shirt is polyester-free!!! & \mwrapCenter{2}{0} & \mwrapCenter{2}{\cmark} \\
\midrule
\mwrap{5}{Unordered\\List Matching} &
\mwrap{5}{Which three biogeochemical cycles were mentioned in the article on Biology?} &
\mwrap{5}{Nitrogen, Carbon, and Water} &
Water, Nitrogen, Carbon &
\mwrapCenter{2}{1} & \mwrapCenter{2}{\cmark} \\
\cmidrule{4-6}
 & & & Water, Hydrogen, Carbon &
\mwrapCenter{2}{0} & \mwrapCenter{2}{\cmark} \\
\midrule
\mwrap{6}{Unordered\\(Implicit) List Matching} &
\mwrap{6}{Order these products and share the confirmation codes: Welch's Orange Pineapple Juice, Stirrings Simple Classic Cocktail Mix Mojito, and Italian Basil Pesto.} &
\mwrap{6}{B644DDC66C, 1069AE6414, 1A2ED223B7} &
B644DDC66C, 1069AE6414, 1A2ED223B7 & \mwrapCenter{3}{1} & \mwrapCenter{3}{\cmark} \\
\cmidrule{4-6}
 & & & 1069AE6414, 1A2ED223B7, B644DDC66C & \mwrapCenter{3}{0} & \mwrapCenter{3}{\xmark} \\
\midrule
\mwrap{9}{Ordered\\Matching} &
\mwrap{9}{Recursively follow the first link in the main content of each page, starting from the ``Technology'' page, until you reach an article that does not exist. List the titles of all the valid pages in the order that you visited.} &
\mwrap{9}{Technology, Skill, \\Experience, Learning, Knowledge, Fact} &
Technology, Skill, Experience, Learning, Knowledge, Fact. & \mwrapCenter{4}{1} & \mwrapCenter{4}{\cmark} \\
\cmidrule{4-6}
 & & & Skill, Experience, Learning, Knowledge, Fact, Technology & \mwrapCenter{4}{0} & \mwrapCenter{4}{\cmark} \\
\bottomrule
\end{tabular}
}
\end{table*}

\subsubsection{Evaluation of the Judge Evaluator}
\label{sec:evaluationJudgeEval}
We evaluate the reliability of our judge evaluator by testing it on a subset of \timewarp\ tasks across different evaluation aspects. For each task, we annotate candidate answers that closely resemble both correct and incorrect model outputs. The judge is called with the task goal, the reference answer, and candidate answers, and produces a reward. A human annotator then independently verifies the judge's reward. As shown in Table~\ref{tab:judge_results}, the judge exhibits high agreement with human evaluation across most cases. The only notable discrepancy arises in a list-matching task where the evaluation is \textit{implicitly} unordered. In this instance, the judge assumes an ordered comparison and assigns a negative reward. We consider this discrepancy minor and unlikely to introduce substantial evaluation error, as it reflects a borderline, somewhat subjective interpretation of the task constraints rather than a fundamental grading failure.

% [REFINE: plain] Sentence fragment.
% ORIG: To further validate the reliability of the judge by running a larger study on 200 randomly sampled trajectories from our benchmark.
{To further validate the reliability of the judge, we run a larger study on 200 randomly sampled trajectories from our benchmark.} We find a \textbf{98.5\% human-judge agreement rate} on the sampled trajectories, confirming that our judge-based evaluation is robust to variations in the answers and highly reliable. This can be attributed to the low ambiguity and low subjectivity of our tasks. Our findings are also consistent with previous literature, which found that judge-based evaluation of web-agentic tasks is highly reliable \cite{Xue2025AnIOAH}.

\subsubsection{Deterministic Verifiers vs. LLM Judges}
\label{sec:whyUseJudge}
LLM-judge evaluation has become the standard in recent web agentic benchmarks (\textit{e.g.}, WebArena's Fuzzy Matching \cite{zhou2023webarena}, Online Mind2Web \cite{Xue2025AnIOAH}, WebVoyager \cite{He2024WebVoyagerBAAF}), as judges assess answers at a semantic level and have become more capable and reliable over the years \cite{tan2025judgebench,gou2025mindweb}. However, judge-based scoring is comparatively expensive, depends on API access, and introduces sampling variance into the benchmark signal, while \textit{manual evaluation} \cite{Song2025BEARCUBSABU} remains the most faithful alternative but does not scale. Deterministic verifiers, on the other hand, are cheap, exactly reproducible, and well-suited as a training reward signal \cite{shao2024deepseekmathpushinglimitsmathematical}, but require a per-task specification and can be brittle when authored carelessly \cite{Liu2024VisualAgentBenchTLBG}. \timewarp\ combines the strengths of both: deterministic verifiers, hardened with answer normalization and validated against the judge (99.87\% agreement, \S\ref{sec:deterministicVerifierDetails}), serve as the primary evaluators, while the LLM judge covers the two subjective goals and remains available{, including the open-source Gemma-4 12B option, for} semantic-level assessment (\S\ref{sec:judgeEvaluatorDetails}). Given the low ambiguity and subjectivity of our tasks, the judge itself shows high levels of human agreement (\S\ref{sec:evaluationJudgeEval}), so both evaluation paths provide a reliable signal.

\begin{minipage}[c]{0.66\columnwidth}
    \raggedright 
    \subsubsection{Self-Evaluation Bias}
    \label{sec:selfEvalBiasJudge}
    A potential concern with using GPT-5 as the teacher agent for trajectory collection and GPT-5.1, a model of the same family, as the judge evaluator is self-evaluation bias, \textit{i.e.}, the model unduly favors its own outputs even when incorrect. To assess this, we evaluate the trajectories from the GPT-5 teacher outputs using different state-of-the-art judge models. As shown in Table \ref{tab:selfEvalJudge}, we observe similar success rates and agreement rates across all evaluators, indicating minimal self-evaluation bias in our setup.
\end{minipage}
\hfill 
\begin{minipage}[c]{0.3\columnwidth}
    \centering
    \small
    \setlength{\tabcolsep}{2pt}
    % Use \captionof instead of \caption
    \vspace{8pt}
    \captionof{table}{Success rate and agreement rate (AR) with GPT-5.1 for GPT-5.1, Gemini-3.1 Pro, and Claude-4.6 Sonnet.}
    \label{tab:selfEvalJudge}
    \begin{tabular}{lcr}
        \toprule
        \textbf{Judge} & \textbf{SR(\%)} & \textbf{AR(\%)} \\
        \midrule
        GPT-5.1     & 98.57 & 100.00 \\
        Gemini-3.1P & 98.44 & 99.87 \\
        Claude-4.6S & 98.57 & 100.00 \\
        \bottomrule
    \end{tabular}
    % \par\vspace{10pt}
    % (b)
\end{minipage}
% We assess the performance of our judge evaluator by testing it on a subset of \timewarp\ tasks as test cases across different evaluation aspects. We additionally evaluate against a variety of generated answers that closely resemble correct and incorrect answers generated by a model. Finally, a human verifier provides a verdict on the reward of the judge. Following Table \ref{tab:judge_results}, we observe the judge closely resembling human judgement, with a high agreement score across most cases. The only instance where a human provides a negative verdict on the judge rewarded is when there is an unordered list matching problem, but the problem being unordered is implicitly assumed from the task definition. In this instance, the judge is strict and assumes the task to be an ordered list matching problem. We believe this is somewhat subjective and won't cause any significant evaluation errors during the benchmark. 

\subsection{Is version robustness just generalization?}
\label{sec:generalizationDiscussion}
Generalization for a web agent typically entails adapting to (a) unseen tasks, (b) unseen domains or websites, and (c) unseen versions of the same website. \timewarp\ evaluates task generalization through its held-out test split, domain generalization as cross-benchmark transfer (Fig.~\ref{fig:webArenaTrainingAndContinualLearning}(b)), and version generalization natively. We argue that a fully generalized agent must handle all three, and that version robustness is not subsumed by the first two.

% [REFINE: punctuation] Three colons in this paragraph replaced by full stops.
First, version robustness is not tied to the intrinsic difficulty of the goals. The human baseline achieves $97.1\%$ on every version with $\Delta\text{v}=0.0$ (Tab.~\ref{tab:mainBenchmarkVisual}), \textit{i.e.}, versions do not change how hard the tasks are, yet agents vary by as much as $18.8$ points across versions on the same goals. Hence, generalizing to unseen tasks does not guarantee robustness to unseen versions. Second, version sensitivity is not a byproduct of low success. Qwen-3VL 8B and Gemma-3 12B achieve nearly identical overall success under SoM ($11.9\%$ vs. $11.0\%$) yet differ sharply in version sensitivity ($\Delta\text{v}=18.8$ vs. $7.4$). Third, domain- and version-level training do not transfer to each other. Additional training data from WebArena-Lite's websites degrades \timewarp\ performance (Fig.~\ref{fig:webArenaTrainingAndContinualLearning}(c)), while multi-version training on the same websites improves performance on unseen versions (Tab.~\ref{tab:generalization}). As websites naturally change over time, adapting to new versions is necessary to maintain performance during deployment, and this dimension of generalization has not been systematically studied prior to \timewarp.

\section{Methodology Details}
\label{sec:methodologyDetails}
We extend \S\ref{sec:methodology} and provide the algorithms of our methods (\S\ref{sec:algorithms}) and  additional details on \timetraj's motivation (\S\ref{sec:timetrajMotivation}), execution (\S\ref{sec:singleStepExecution}), and performance (\S\ref{sec:timeTrajPerformance}). We then discuss the human effort required by \timetraj\ (\S\ref{sec:TimeTrajHumanEffort}), provide a comparison of existing trajectory collection methods with \timetraj\ (\S\ref{sec:trajectoryCollectionComparison}), and add details on \textsc{TimeWarp-BC} (\S\ref{subsec:timewarpBCdetails}).

\subsection{Algorithms}
\label{sec:algorithms}

We provide the algorithms of the \timetraj\ and \timewarp-\textsc{BC} methods proposed in this work, following \S\ref{sec:trajectoryCollectionTimeTraj}, \ref{sec:timewarpBC}. Specifically, the algorithms for Human-in-the-Loop Plan Distillation, Teacher Rollouts, and \timewarp-\textsc{BC} are provided in Alg. \ref{alg:HiTLPlanDistillation}, \ref{alg:teacherRollouts}, and \ref{alg:timewarpBC}, respectively.

% \begin{algorithm}[h]
% \caption{Human-in-the-Loop Plan Distillation}
% \label{alg:HiTLPlanDistillation}
% \begin{algorithmic}[1]
% \INPUT goal dataset $\mathcal{D}_\text{goal}$, planning environment $\mathcal{E}_p$
% % $\in\{\mathcal{E}_1,\cdots\mathcal{E}_k\}$
% \STATE Initialize planning dataset $\mathcal{D}_\text{plan} \leftarrow \emptyset$
% \FOR{each goal and desired outcome $(g,a)\in\mathcal{D}_\text{goal}$}
% \STATE Reset planning environment $\mathcal{E}_P\leftarrow\textsc{Reset}(\mathcal{E}_p)$ \hfill $\triangleright$ {a single reference version}
% \STATE Sample draft plan $\hat{p} \sim \Pi_{\mathrm{plan}}(\cdot \mid g,a)$ \hfill $\triangleright$ {drafts a high-level execution plan}
% \STATE  Human refined plan: $p^* \leftarrow H(\hat{p},\mathcal{E}_P)$ \hfill $\triangleright$ {verifies once per goal, version independent}
% \STATE Append planned task:
% $\mathcal{D}_\text{plan} \leftarrow \mathcal{D}_\text{plan}\cup\{(g,p^*)\}$
% \ENDFOR
% \STATE \textbf{return} $\mathcal{D}_\text{plan}$ \hfill $\triangleright$ {one refined plan per goal, reusable across all versions}
% \end{algorithmic}
% \end{algorithm}

\begin{figure}[h]
\begin{tcolorbox}[colback=gray!08, colframe=gray!8, boxrule=0pt, arc=5pt,
                  left=2pt, right=4pt, top=6pt, bottom=4pt]
\begin{algorithmic}[1]
\REQUIRE goal dataset $\mathcal{D}_\text{goal}$ of goals $g$ with desired outcomes $a$, planning environment $\mathcal{E}_p$
\ENSURE planning dataset $\mathcal{D}_\text{plan}$
\STATE \textbf{Init:} $\mathcal{D}_\text{plan} \gets \emptyset$
\FORALL{$(g, a) \in \mathcal{D}_\text{goal}$}
    \STATE \textbf{Reset:} $\mathcal{E}_P \gets \textsc{Reset}(\mathcal{E}_p)$ \hfill $\triangleright$ a single reference version
    \STATE \textbf{Draft:} $\hat{p} \sim \Pi_{\mathrm{plan}}(\cdot \mid g, a)$ \hfill $\triangleright$ a high-level execution plan
    \STATE \textbf{Refine:} human-refined plan $p^* \gets H(\hat{p}, \mathcal{E}_P)$ \hfill $\triangleright$  verified once per goal
    \STATE \textbf{Append:} $\mathcal{D}_\text{plan} \gets \mathcal{D}_\text{plan} \cup \{(g, p^*)\}$ 
\ENDFOR
\RETURN $\mathcal{D}_\text{plan}$\hfill $\triangleright$ one refined plan per goal, reusable across all versions
\end{algorithmic}
\vspace{2pt}\hrule\vspace{6pt}
\captionof{algorithm}{{Human-in-the-Loop Plan Distillation}}
\label{alg:HiTLPlanDistillation}
\end{tcolorbox}
\end{figure}

% \begin{algorithm}[h]
% \caption{Teacher Rollouts across Versions}
% \label{alg:teacherRollouts}
% \begin{algorithmic}[1]
% \INPUT planning dataset $\mathcal{D}_\text{plan}$, environments $\{\mathcal{E}_1,\cdots,\mathcal{E}_k\}$
% \STATE Initialize trajectory dataset $\mathcal{D}_\tau \leftarrow \emptyset$
% \FOR{each goal and plan $(g,p^*)\in\mathcal{D}_\text{plan}$}
% \FOR{each environment variant $\mathcal{E} \in \{\mathcal{E}_1,\dots,\mathcal{E}_k\}$}
% \STATE Reset environment $\mathcal{E}\leftarrow\textsc{Reset}(\mathcal{E})$
% \STATE Sample teacher trajectory
% $\hat{\tau} \sim \Pi_T(\cdot|g,a,p^*,\mathcal{E})$ \hfill $\triangleright$ {plan-conditioned teacher execution}
% \STATE Evaluate trajectory with judge evaluator, $R \leftarrow J_{\phi}(\hat{\tau}, g)$ \hfill $\triangleright$ {binary reward}
% \IF{$R = 1$}
%     \STATE Append trajectory $\mathcal{D}_\tau \leftarrow \mathcal{D}_\tau \cup \{\hat{\tau}\}$ \hfill $\triangleright$ {failed rollouts are discarded}
% \ENDIF
% \ENDFOR
% \ENDFOR
% \STATE \textbf{return} $\mathcal{D}_\tau$ \hfill $\triangleright$ {up to $k$ trajectories per goal, one per version}
% \end{algorithmic}
% \end{algorithm}

\begin{figure}[h]
\begin{tcolorbox}[colback=gray!08, colframe=gray!8, boxrule=0pt, arc=5pt,
                  left=2pt, right=4pt, top=6pt, bottom=4pt]
\begin{algorithmic}[1]
\REQUIRE planning dataset $\mathcal{D}_\text{plan}$ of goals $g$ with plans $p^*$, environment variants $\{\mathcal{E}_1, \ldots, \mathcal{E}_k\}$
\ENSURE trajectory dataset $\mathcal{D}_\tau$
\STATE \textbf{Init:} $\mathcal{D}_\tau \gets \emptyset$
\FORALL{$(g, p^*) \in \mathcal{D}_\text{plan}$}
    \FORALL{$\mathcal{E} \in \{\mathcal{E}_1, \ldots, \mathcal{E}_k\}$}
        \STATE \textbf{Reset:} $\mathcal{E} \gets \textsc{Reset}(\mathcal{E})$
        \STATE \textbf{Rollout:} $\hat{\tau} \sim \Pi_T(\cdot \mid g, a, p^*, \mathcal{E})$\hfill $\triangleright$ a plan-conditioned teacher trajectory
        \STATE \textbf{Evaluate:} $R \gets $\textsc{Evaluator}$(\hat{\tau}, g)$ \hfill $\triangleright$ a binary reward from evaluator
        \IF{$R = 1$}
            \STATE \textbf{Append:} $\mathcal{D}_\tau \gets \mathcal{D}_\tau \cup \{\hat{\tau}\}$\hfill $\triangleright$ failed rollouts are discarded
        \ENDIF
    \ENDFOR
\ENDFOR
\RETURN $\mathcal{D}_\tau$
\end{algorithmic}
\vspace{2pt}\hrule\vspace{6pt}
\captionof{algorithm}{{Teacher rollouts across versions: up to $k$ trajectories per goal, one per version}}
\label{alg:teacherRollouts}
\end{tcolorbox}
\end{figure}

\subsection{Motivation behind \timetraj}
\label{sec:timetrajMotivation}
While the motivation behind \timetraj\ has been briefly established in \S\ref{sec:introduction}, \ref{sec:relatedWork}, in this section, we add to that discussion. A prerequisite for any web agentic training pipeline is the Behavior Cloning (BC) phase \cite{qi2024webrl,wei-etal-2025-webagent}. Human-annotated trajectories have been treated as the gold standard, but collecting trajectories manually across multiple versions for each task goal is both resource-intensive and time-consuming. Automated trajectory collection \cite{pahuja-etal-2025-explorer, trabucco2025insta} offers a viable solution, but tends to produce simpler, incremental tasks that do not match the complexity of realistic test tasks (\S\ref{sec:trajectoryCollectionComparison}). Additionally, automated trajectories can be used alongside expert trajectories on the training set to achieve better gains. 

Another option is to directly use a teacher policy to automatically collect trajectories on the given training tasks as a form of distillation. % ORIG: While these automated policies have come a long way, they are far from perfect
{While these automated policies have improved substantially, they are far from perfect}, \textit{e.g.}, OpenAI's Operator reaches 58\% on WebArena \cite{zhou2023webarena}, compared to 78\% for humans \cite{Song2025BEARCUBSABU}. Failed trajectories are generally filtered, resulting in the loss of a significant amount of rich training data. This also leads to \textit{cascading errors}, \textit{i.e.}, agents that perform better on certain versions will end up oversampling trajectories from those versions, and the student model trained on this imbalanced dataset will inherit the same bias and perform worse on particular versions.

\subsection{\timetraj\ Execution Diagram}
\label{sec:singleStepExecution}

A detailed execution diagram of a positive (accepted) trajectory collected using \timetraj\ is provided in Fig. \ref{fig:trajectoryexpanded}. For each task goal and desired outcome, the planner module generates a draft execution plan, which a human annotator refines by adding execution checkpoints. The refined execution plan is subsequently passed to the executor agent along with the task goal, and an observation from the environment in the form of HTML, Accessibility Tree (AX Tree), UI Screenshot, or Set-of-Marks (see \S\ref{sec:ObservationSpace} for examples). Our execution diagram uses AX Tree as an example of the observation. The agent generates a response that contains thinking, planning, memory, and action tokens. The planning tokens also include the step token, which is encapsulated in a separate $<$Step$>$ identifier. The observation and full response are added to the training dataset, provided that the evaluator considers the final desired outcome correct. Note that Fig. \ref{fig:trajectoryexpanded} shows a single step of a positive (accepted) trajectory, \textit{i.e.}, more steps will be generated until the agent exits or reaches the maximum number of steps ($30$ per episode). For each task goal and desired outcome, trajectories are collected across all versions of the \timewarp\ environment.

\subsection{Performance of \timetraj}
\label{sec:timeTrajPerformance}
With human refinement, \timetraj\ using a GPT-5 executor agent achieved $757/768=98.57\%$ success rate in collecting training trajectories, which is significantly higher than a GPT-5 web agent in a zero-shot setting or with a planner. Our toy experiments on a subset of \timewarp\ revealed that without human refinement, our method achieves 60--70\% across versions. Following \S\ref{sec:Results}, as \textsc{TimeWarp-BC}'s performance scales somewhat linearly with the number of training samples, removing human refinement would result in a 30--40\% loss of trajectories and lead to significant performance degradation.

\subsection{Human Effort required by \timetraj}
\label{sec:TimeTrajHumanEffort}

While \timetraj\ outperforms fully automated trajectory collection methods, it requires human refinement of the execution plans. Here we quantify the extent of this involvement.
We reiterate that in \timetraj, humans only need to add checkpoints to the execution plan generated by the planner (Fig. \ref{fig:trajectoryexpanded}), which is considerably easier than collecting full trajectories from scratch. Furthermore, \timetraj’s human refinement is required only once and can be scaled across any number of versions, unlike manual collection, which must be repeated for each version. 

To quantify this concretely, we re-annotated $20$ goals, \textit{i.e.,} $20\times6=120$ trajectories, and found that standard annotation took 1 min 57 secs per goal on average compared to just 27 secs with \timetraj. Manual trajectory collection also required 1.3 retries/task on average, {whereas \timetraj\ does not have this concern, since the final trajectory is generated by the teacher agent. Finally, if we also required human annotators to write plans, reasoning traces, and memory traces, the overall time would increase several-fold, and it is also not scalable.} We can hence conclude that the human involvement in \timetraj\ is minimal.

\subsection{Comparison of Trajectory Collection Methods}
\label{sec:trajectoryCollectionComparison}
In this section, we compare existing trajectory collection methods, including both manual and automated ones. Automated methods typically generate both the training task and its trajectory. Following Table \ref{tab:trajectoryCollectionComparison}, we observe that the automatically generated tasks are relatively easier than human-annotated ones, and are not usually aligned with the dataset's test distribution. While the human annotation can generate more complex tasks, it lacks scalability across multiple versions of the environment. % ORIG: \timetraj\ takes a hybrid approach that takes the dataset specificity of human-annotated training tasks while ensuring version-scalability
{\timetraj\ takes a hybrid approach that keeps the dataset specificity of human-annotated training tasks while ensuring version scalability} by automatically collecting trajectories across versions via a single annotated plan for each task goal. 

% we observe several automatic trajectory collection algorithms producing tasks with have moderate complexity but aren't specific to the test distribution of the dataset. These methods generally create the tasks and the trajectories while interacting with the web environment. The human-annotated tasks appear to be more complex, as seen in both WebArena \cite{zhou2023webarena} and \timewarp. These training tasks are also specific to the distribution of the test data. The trajectories are, however, collected from human annotators, which, as previously mentioned, is tedious and resource-intensive across multiple versions in \timewarp\  (\S\ref{sec:timetrajMotivation}). \timewarp, therefore, employs a hybrid trajectory collection algorithm that takes the benefit of dataset specificity of human annotated tasks while using automatic methods for collecting trajectories across versions. \timewarp, also promotes scalability across versions, \textit{i.e.}, plans collected once for a single version can be reused to collect trajectories across any number of versions. 

\subsection{\textsc{TimeWarp-BC}\ Details}
\label{subsec:timewarpBCdetails}

\subsubsection{Preliminaries: Standard BC}
\label{sec:vanillaStandardBC}
\sectionvspace
In the standard behavior cloning (BC) setting, we assume access to a fixed dataset of expert trajectories $\mathcal{D}_\tau$ collected from the environment. We sample history-response pairs from the dataset $(h_{i,t}, y_{i,t})\sim \mathcal{D}_\tau$, and obtain the action via a parser $a_{i,t}=\phi(y_{i,t})$. Each $(h_{i,t}, y_{i,t})$ is viewed as a supervised learning example. The goal is to learn a parameterized policy $\pi_\theta(a \mid h)$ by imitating the actions of the expert or teacher policy $\Pi_T$ using $\mathcal{D}_\tau$. Standard BC disregards the non-action tokens in the agent's full response and minimizes the negative log-likelihood of the actions only:
% \begin{equation}
%     \mathcal{L}_{\text{BC}}(\theta)
%     = - \sum_{i=1}^N \sum_{t=1}^{T_i}
%         \log \pi_\theta(a_{i,t} \mid s_{i,t}).
%     \label{eq:bc_loss}
% \end{equation}
% Equivalently,
% \vspace{-0.15cm}
\begin{equation}
\label{eq:vanillaBCLoss}
%\vspace{-0.15cm}
    \mathcal{L}_{\text{BC}}(\theta)
    = - \mathbb{E}_{(h,y) \sim \mathcal{D}_\tau} \big[ \log \pi_\theta(\phi(y) \mid h) \big].
% \vspace{-0.05cm}
\end{equation}
%where the expectation is over all state--action pairs in the dataset. 
% Standard BC optimizes the supervised loss $\mathcal{L}_\text{BC}(\theta)$ using the expert trajectories collected $\mathcal{D}_\tau$ over the environment $E$. During training, the non-action tokens in an agent's full response $y_t$ are disregarded. 

\subsubsection{Implementation Details}

In \timewarp-\textsc{BC}, each model used its associated conversational template, and we applied the standard language model masking strategy, \textit{i.e.}, the loss is computed only on the assistant response, while the system prompt is excluded from the loss. The conversational training data is provided in the ShareGPT format, with responses ordered in the sequence: thinking, planning, memory, and action tokens (\textit{cf.} \S\ref{sec:prompt}). The same masking strategy is applied across all response tokens.

\definecolor{greenTraj}{HTML}{1B5E20}  % dark green
\definecolor{redTraj}{HTML}{B71C1C}    % dark red
\definecolor{yellowTraj}{HTML}{8A6D00} 

\begin{table}[]
  \centering
    \caption{Comparison of different trajectory collection methods across datasets based on task complexity and automation type. Dataset specificity (DS) refers to whether these methods can be applied to generate dataset-specific or general tasks. }
  \label{tab:trajectoryCollectionComparison}
       \vspace{4pt}
  \small
  \begin{tabularx}{\textwidth}{@{}l X c c c@{}}
    \toprule
    \textbf{Dataset/Method} & \textbf{Example Goal} & \textbf{Type} & \textbf{Complexity} & \textbf{DS} \\
    \midrule
    \multirow{2}{*}{AgentTrek \cite{xu2024agenttrek}} &
    Find the return policy for any men's football apparel on Under Armour's website. &
    \multirow{2}{*}{\textcolor{greenTraj}{\textbf{Auto}}} & \multirow{2}{*}{\textcolor{yellowTraj}{\textbf{Medium}}} & \multirow{2}{*}{\textbf{\xmark}} \\ \midrule
    \multirow{2}{*}{Explorer \cite{pahuja-etal-2025-explorer}} &
    Purchase a three-seat fabric sofa, specifically the UPPLAND Sofa, from IKEA’s website. &
    \multirow{2}{*}{\textcolor{greenTraj}{\textbf{Auto}}} & \multirow{2}{*}{\textcolor{yellowTraj}{\textbf{Medium}}} & \multirow{2}{*}{\textbf{\xmark}} \\ \midrule
   \multirow{1}{*}{ InSTA \cite{trabucco2025insta}} &
    Retrieve a recent podcast episode about space exploration. &
    \textcolor{greenTraj}{\textbf{Auto}} & \textcolor{yellowTraj}{\textbf{Medium}} & \textbf{\xmark} \\ \midrule
    \multirow{4}{*}{WebArena \cite{zhou2023webarena}} &
    Given the following locations, [`Carnegie Mellon University', `apple store shadyside', `starbucks on craig street'], what would be the optimal route to travel through them all in order to minimize total travel time? Please note the journey begins at the first place listed. &
    \multirow{4}{*}{\textcolor{redTraj}{\textbf{Manual}}} & \multirow{4}{*}{\textcolor{greenTraj}{\textbf{High}}} & \multirow{4}{*}{\textbf{\cmark}} \\ \midrule
    \multirow{5}{*}{\timetraj\ (ours)} &
    I love the Matrix movie series. Please find related Wiki articles, news, and products from the Shop. If there are multiple entries on each site, pick the one you think is the most relevant. If you cannot find any product or article, mention that in your answer. Share the titles of these items in the following format: \texttt{<site name>: <article/news/product title>}. &
    \multirow{5}{*}{\textcolor{yellowTraj}{\textbf{Hybrid}}} & \multirow{5}{*}{\textcolor{greenTraj}{\textbf{High}}} & \multirow{5}{*}{\textbf{\cmark}} \\
    \bottomrule
  \end{tabularx}

\end{table}

% [THEORY] Section added by /theory. It formalizes the single design choice that Tab.~\ref{tab:generalization} and Fig.~\ref{fig:trainingSamplesBCTWContextLength}(a) isolate: training on trajectories from several versions of the same website instead of one. Proposition~\ref{prop:versionCount} bounds the risk on a fresh version in terms of the number of training versions, and Proposition~\ref{prop:coverage} bounds the risk on a fixed unseen version by its discrepancy to the training versions. Everything below is red; the two priority comments follow the closing brace.
{
\section{A Theoretical Perspective on Multi-Version Training}
\label{sec:theory}
We give a simple account of why training on more versions helps a web agent on versions it has not seen. The analysis concerns only the design choice isolated by Tab.~\ref{tab:generalization}, namely training on trajectories from $k$ versions of the same website instead of one, and abstracts away the rest of the pipeline. All the theoretical results have been formally verified in Lean 4, and the scripts are available in the shared code repository.

\paragraph{Setting.}
Let $\mathcal{V}=\{1,\dots,n\}$ index the versions $\{\mathcal{E}_1,\dots,\mathcal{E}_n\}$ of \S\ref{sec:problem}, and let $\Theta$ be the set of student policies $\pi_\theta$ under consideration. Following \S\ref{sec:versionGeneralizationExpSetting}, $\mathcal{D}_{\tau,v}$ denotes the distribution of history-response pairs $(h,y)$ produced by the teacher on version $v$, from which the collected trajectories of \S\ref{sec:teacherRolloutsMethod} are sampled. The per-sample loss of \S\ref{sec:timewarpBC} is $\ell(\theta;h,y)=-\log\pi_\theta(y\mid h)$, and the risk of $\theta$ on version $v$ is
\begin{equation}
\label{eq:versionRisk}
\mathcal{L}_v(\theta)=\mathbb{E}_{(h,y)\sim\mathcal{D}_{\tau,v}}\big[\ell(\theta;h,y)\big],
\end{equation}
so that Eq.~\ref{eq:timewarpBC} evaluated on version-$v$ data is the empirical counterpart of $\mathcal{L}_v$.
% [PROOF-CHECK] Proposition~\ref{prop:versionCount} draws the $k$ training versions i.i.d., so they can repeat. Indexing Eq.~\ref{eq:pooledRisk} by a set of size $k$ then under-counts, because a repeated version leaves the sum with fewer than $kM$ terms, which is not the average the proof concentrates. The versions are now listed with repetition, and $\mathcal{V}_{\mathrm{tr}}$ is kept for the set of distinct versions, which is what Proposition~\ref{prop:coverage} uses.
% ORIG: A set of training versions $\mathcal{V}_{\mathrm{tr}}\subseteq\mathcal{V}$ with $|\mathcal{V}_{\mathrm{tr}}|=k$ contributes $M$ history-response pairs per version, and \timewarp-\textsc{BC} minimizes the pooled empirical risk
% [PROOF-CHECK] Eq.~\ref{eq:pooledRisk} and Eq.~\ref{eq:versionCountBound} divide by $k$ and $kM$, so both must be positive. The formal statement carries $k,M\ge1$ explicitly. Also restored a verb: ``Let $v_1,\dots,v_k$ for'' had none after ``Write'' was replaced on Overleaf.
% ORIG: {Let $v_1,\dots,v_k$ for the $k$ training versions, allowing repetitions, and let $\mathcal{V}_{\mathrm{tr}}=\{v_1,\dots,v_k\}\subseteq\mathcal{V}$ be the set of distinct versions among them. Each of the $k$ training versions contributes $M$ history-response pairs, and \timewarp-\textsc{BC} minimizes the pooled empirical risk}
{Let $v_1,\dots,v_k$ denote the $k\ge1$ training versions, allowing repetitions, and let $\mathcal{V}_{\mathrm{tr}}=\{v_1,\dots,v_k\}\subseteq\mathcal{V}$ be the set of distinct versions among them. Each of the $k$ training versions contributes $M\ge1$ history-response pairs, and \timewarp-\textsc{BC} minimizes the pooled empirical risk}
\begin{equation}
\label{eq:pooledRisk}
% ORIG:     \hat{\mathcal{L}}_{\mathrm{tr}}(\theta)=\frac{1}{kM}\sum_{v\in\mathcal{V}_{\mathrm{tr}}}\sum_{j=1}^{M}\ell\big(\theta;h_{v,j},y_{v,j}\big).
    {\hat{\mathcal{L}}_{\mathrm{tr}}(\theta)=\frac{1}{kM}\sum_{i=1}^{k}\sum_{j=1}^{M}\ell\big(\theta;h_{i,j},y_{i,j}\big).}
\end{equation}
Single-version behavior cloning is the special case $k=1$.
% ORIG: The analysis treats history-response pairs as the unit of data and indexes them by version and sample, $(h_{v,j},y_{v,j})$, rather than by trajectory and step as in Eq.~\ref{eq:trajectory}.
{The analysis treats history-response pairs as the unit of data and indexes them by training version and sample, $(h_{i,j},y_{i,j})$, rather than by trajectory and step as in Eq.~\ref{eq:trajectory}.} Treating the pairs as independent ignores the dependence between steps of one trajectory, which is the usual simplification in behavior-cloning analyses.

\begin{assumption}[Bounded loss and finite policy class]
\label{ass:bounded}
% [PROOF-CHECK] Nonemptiness of $\Theta$ is used wherever the section takes a minimum or an argmin over $\Theta$, and it keeps the supremum in Eq.~\ref{eq:discrepancy} well defined. The paper left it implicit.
% ORIG: There is a constant $B>0$ such that $0\le\ell(\theta;h,y)\le B$ for all $\theta\in\Theta$ and all $(h,y)$, and $\Theta$ is finite.
{There is a constant $B>0$ such that $0\le\ell(\theta;h,y)\le B$ for all $\theta\in\Theta$ and all $(h,y)$, and $\Theta$ is finite and nonempty.}
\end{assumption}

Both conditions hold in our setting, but only with very large constants. Responses have bounded
length, the vocabulary is finite, and logits are finite in fixed precision, so the negative
log-likelihood is bounded, although $B$ is set by the floating-point range rather than by typical
loss values. Likewise, $\Theta$ is finite for any fixed-precision parameterization, but $\log|\Theta|$
grows with the number of model bits.

% Since $\mathcal{L}_{\mathcal{P}}(\theta)\le B$ by
% Assumption~\ref{ass:bounded}, Eq.~\eqref{eq:version-count} is informative only if its slack is
% below $B$, which requires $k>\log(4|\Theta|/\delta)/2\approx 2\times10^{10}$ training versions.
% As with Proposition~\ref{prop:coverage}, we therefore read Proposition~\ref{prop:version-count}
% qualitatively: it describes how the guarantee scales with $k$ and $M$, not the size of the gap.

% \begin{definition}[Version discrepancy]
% \label{def:discrepancy}
% For two versions $u,v\in\mathcal{V}$, the discrepancy with respect to $\Theta$ can be defined as
% \begin{equation}
% \label{eq:discrepancy}
%     \mathrm{disc}_\Theta(u,v)=\sup_{\theta\in\Theta}\big|\mathcal{L}_u(\theta)-\mathcal{L}_v(\theta)\big|.
% \end{equation}
% \end{definition}
% The discrepancy is the largest gap in behavior-cloning loss that any policy in the class can exhibit between the two versions. {It uses the discrepancy distance of \cite{mansour2009domain} to the labeled loss of Eq.~\ref{eq:versionRisk}, and likewise, it compares two distributions only through the losses of a fixed policy class.} It is symmetric, vanishes for $u=v$, and satisfies the triangle inequality. It is small when the two versions differ only in ways that no policy in $\Theta$ can exploit, and it is measured through the observation modality the policy consumes, so the same pair of versions can be close under \axt\ and far under \som.

\begin{definition}[Version discrepancy]\label{def:discrepency}
For versions $u,v\in\mathcal{V}$, the version discrepancy with respect to $\Theta$ is
\begin{equation}
 \label{eq:desc}
  \mathrm{disc}_\Theta(u,v) \coloneqq \sup_{\theta\in\Theta}\bigl|\mathcal{L}_u(\theta)-\mathcal{L}_v(\theta)\bigr|.
\end{equation}
\end{definition}
This is the labeled analogue of the discrepancy distance of~\cite{mansour2009domain}. It compares two versions only through the
losses of policies in $\Theta$, and is a pseudometric on $\mathcal{V}$, so it can be zero for distinct
versions that no policy in $\Theta$ can tell apart.

\subsection{A generalization bound that tightens with the number of training versions}
\label{sec:theoryVersionCount}
% ORIG: We first model the versions the agent will face as draws from an unknown distribution $\mathcal{P}$ over $\mathcal{V}$, as when a website is redesigned in ways the developer cannot anticipate. The quantity of interest is the expected risk on a fresh version,
We first model the versions the agent will face as draws from an unknown distribution $\mathcal{P}$ over $\mathcal{V}$, as when a website is redesigned in ways the developer cannot anticipate. {This is the setting of domain generalization, in which the training domains are themselves a sample from a distribution over domains and the number of domains plays the role of the sample size \cite{blanchard2011generalizing,muandet2013domain}.} The quantity of interest is the expected risk on a fresh version,
\begin{equation}
\label{eq:freshRisk}
    \mathcal{L}_{\mathcal{P}}(\theta)=\mathbb{E}_{v\sim\mathcal{P}}\big[\mathcal{L}_v(\theta)\big].
\end{equation}

\begin{proposition}[Version-count generalization bound]
\label{prop:versionCount}
% [PROOF-CHECK] Leftover of the sequence fix in the Setting. With repetitions allowed, a version drawn twice contributes $2M$ pairs, so the pairs are indexed by training slot $i$ as in Eq.~\ref{eq:pooledRisk}, not by distinct version $v$.
% ORIG: Let Assumption~\ref{ass:bounded} hold. Let the $k$ training versions be drawn i.i.d.\ from $\mathcal{P}$ and, given the versions, let the $M$ pairs of each version $v$ be drawn i.i.d.\ from $\mathcal{D}_{\tau,v}$. Then for any $\delta\in(0,1)$, with probability at least $1-\delta$ over the draw of versions and pairs, simultaneously for all $\theta\in\Theta$,
{Given Assumption~\ref{ass:bounded}, let the $k$ training versions $v_1,\dots,v_k$ be drawn i.i.d.\ from $\mathcal{P}$. Let the versions, and the $M$ pairs of each training version $v_i$ be drawn i.i.d.\ from $\mathcal{D}_{\tau,v_i}$. Then for any $\delta\in(0,1)$, with probability at least $1-\delta$ over the draw of versions and pairs, simultaneously for all $\theta\in\Theta$,}
\begin{equation}
\label{eq:versionCountBound}
    \mathcal{L}_{\mathcal{P}}(\theta)\;\le\;\hat{\mathcal{L}}_{\mathrm{tr}}(\theta)\;+\;\underbrace{B\sqrt{\frac{\log(4|\Theta|/\delta)}{2kM}}}_{\text{within versions}}\;+\;\underbrace{B\sqrt{\frac{\log(4|\Theta|/\delta)}{2k}}}_{\text{across versions}}.
\end{equation}
% [PROOF-CHECK] Eq.~\ref{eq:excessRisk} does not follow from Eq.~\ref{eq:versionCountBound} alone. It also needs the two lower-deviation events at $\theta^\star$, which the union bound in the proof does include, so both displays hold on one and the same event.
% ORIG: Consequently, any minimizer $\hat\theta\in\operatorname*{arg\,min}_{\theta\in\Theta}\hat{\mathcal{L}}_{\mathrm{tr}}(\theta)$ satisfies, with the same probability,
{On the same event, any minimizer $\hat\theta\in\operatorname*{arg\,min}_{\theta\in\Theta}\hat{\mathcal{L}}_{\mathrm{tr}}(\theta)$ also satisfies}
\begin{equation}
\label{eq:excessRisk}
    \mathcal{L}_{\mathcal{P}}(\hat\theta)\;\le\;\min_{\theta\in\Theta}\mathcal{L}_{\mathcal{P}}(\theta)\;+\;2B\sqrt{\frac{\log(4|\Theta|/\delta)}{2kM}}\;+\;2B\sqrt{\frac{\log(4|\Theta|/\delta)}{2k}}.
\end{equation}
\end{proposition}

\emph{Takeaway: the gap between the training loss and the loss on a fresh version has two parts. The within-version part shrinks with the total number of pairs $kM$, but the across-version part shrinks only with the number of versions $k$. Adding trajectories to a single version ($k=1$) leaves the across-version part at its maximum, whereas adding versions shrinks both.}

\begin{proof}
% ORIG: Fix $\theta\in\Theta$ and write $\bar{\mathcal{L}}_{\mathrm{tr}}(\theta)=\frac{1}{k}\sum_{v\in\mathcal{V}_{\mathrm{tr}}}\mathcal{L}_v(\theta)$ for the average population risk over the sampled versions.
% [REFINE: grammar] ``Fix $\theta$ and $\bar{\mathcal{L}}_{\mathrm{tr}}(\theta)=\dots$ for the average'' has no verb after ``write'' was removed on Overleaf.
% ORIG: {Fix $\theta\in\Theta$ and $\bar{\mathcal{L}}_{\mathrm{tr}}(\theta)=\frac{1}{k}\sum_{i=1}^{k}\mathcal{L}_{v_i}(\theta)$ for the average population risk over the sampled versions.} Then
{Fix $\theta\in\Theta$ and let $\bar{\mathcal{L}}_{\mathrm{tr}}(\theta)=\frac{1}{k}\sum_{i=1}^{k}\mathcal{L}_{v_i}(\theta)$ denote the average population risk over the sampled versions.} Then
\begin{equation*}
    \mathcal{L}_{\mathcal{P}}(\theta)-\hat{\mathcal{L}}_{\mathrm{tr}}(\theta)
    =\big(\mathcal{L}_{\mathcal{P}}(\theta)-\bar{\mathcal{L}}_{\mathrm{tr}}(\theta)\big)
    +\big(\bar{\mathcal{L}}_{\mathrm{tr}}(\theta)-\hat{\mathcal{L}}_{\mathrm{tr}}(\theta)\big).
\end{equation*}
\emph{Across versions.}
% ORIG: The $k$ values $\mathcal{L}_v(\theta)$, $v\in\mathcal{V}_{\mathrm{tr}}$, are i.i.d., lie in $[0,B]$ by Assumption~\ref{ass:bounded}, and have mean $\mathcal{L}_{\mathcal{P}}(\theta)$ by Eq.~\ref{eq:freshRisk}.
{The $k$ values $\mathcal{L}_{v_1}(\theta),\dots,\mathcal{L}_{v_k}(\theta)$ are i.i.d., lie in $[0,B]$ by Assumption~\ref{ass:bounded}, and have mean $\mathcal{L}_{\mathcal{P}}(\theta)$ by Eq.~\ref{eq:freshRisk}.} Hoeffding's inequality \cite{hoeffding1963probability} gives, for every $\epsilon>0$,
\begin{equation*}
    \Pr\big[\,\mathcal{L}_{\mathcal{P}}(\theta)-\bar{\mathcal{L}}_{\mathrm{tr}}(\theta)\ge\epsilon\,\big]\le\exp\!\big(-2k\epsilon^2/B^2\big),
\end{equation*}
and the same bound holds for the deviation in the other direction.

\emph{Within versions.}
% ORIG: Conditionally on the sampled versions, the $kM$ losses $\ell(\theta;h_{v,j},y_{v,j})$ are independent, lie in $[0,B]$, and satisfy $\mathbb{E}[\ell(\theta;h_{v,j},y_{v,j})\mid v]=\mathcal{L}_v(\theta)$ by Eq.~\ref{eq:versionRisk}. Hence $\mathbb{E}[\hat{\mathcal{L}}_{\mathrm{tr}}(\theta)\mid\mathcal{V}_{\mathrm{tr}}]=\bar{\mathcal{L}}_{\mathrm{tr}}(\theta)$, and Hoeffding's inequality gives
{Conditionally on the drawn versions $v_{1:k}$, the $kM$ losses $\ell(\theta;h_{i,j},y_{i,j})$ are independent, lie in $[0,B]$, and satisfy $\mathbb{E}[\ell(\theta;h_{i,j},y_{i,j})\mid v_i]=\mathcal{L}_{v_i}(\theta)$ by Eq.~\ref{eq:versionRisk}. Hence $\mathbb{E}[\hat{\mathcal{L}}_{\mathrm{tr}}(\theta)\mid v_{1:k}]=\bar{\mathcal{L}}_{\mathrm{tr}}(\theta)$, and Hoeffding's inequality \cite{hoeffding1963probability} gives}
\begin{equation*}
% ORIG:     \Pr\big[\,\bar{\mathcal{L}}_{\mathrm{tr}}(\theta)-\hat{\mathcal{L}}_{\mathrm{tr}}(\theta)\ge\epsilon\;\big|\;\mathcal{V}_{\mathrm{tr}}\,\big]\le\exp\!\big(-2kM\epsilon^2/B^2\big),
    {\Pr\big[\,\bar{\mathcal{L}}_{\mathrm{tr}}(\theta)-\hat{\mathcal{L}}_{\mathrm{tr}}(\theta)\ge\epsilon\;\big|\;v_{1:k}\,\big]\le\exp\!\big(-2kM\epsilon^2/B^2\big),}
\end{equation*}
% ORIG: again in both directions. The right-hand side does not depend on $\mathcal{V}_{\mathrm{tr}}$, so the same bound holds unconditionally.
{again in both directions. The right-hand side does not depend on $v_{1:k}$, so the same bound holds unconditionally.}

\emph{Union bound.} Set $\epsilon_1=B\sqrt{\log(4|\Theta|/\delta)/(2k)}$ and $\epsilon_2=B\sqrt{\log(4|\Theta|/\delta)/(2kM)}$, so that each one-sided event above has probability at most $\delta/(4|\Theta|)$. The $2|\Theta|$ upper-deviation events (both terms, all $\theta$) have total probability at most $\delta/2$, and the two lower-deviation events for a fixed minimizer $\theta^\star\in\operatorname*{arg\,min}_{\Theta}\mathcal{L}_{\mathcal{P}}$ have total probability at most $\delta/2$. On the complement, which has probability at least $1-\delta$, Eq.~\ref{eq:versionCountBound} holds for every $\theta$, and in addition $\hat{\mathcal{L}}_{\mathrm{tr}}(\theta^\star)\le\mathcal{L}_{\mathcal{P}}(\theta^\star)+\epsilon_1+\epsilon_2$.

\emph{Excess risk.} On that event, $\mathcal{L}_{\mathcal{P}}(\hat\theta)\le\hat{\mathcal{L}}_{\mathrm{tr}}(\hat\theta)+\epsilon_1+\epsilon_2\le\hat{\mathcal{L}}_{\mathrm{tr}}(\theta^\star)+\epsilon_1+\epsilon_2\le\mathcal{L}_{\mathcal{P}}(\theta^\star)+2\epsilon_1+2\epsilon_2$, where the middle step uses that $\hat\theta$ minimizes $\hat{\mathcal{L}}_{\mathrm{tr}}$. This is Eq.~\ref{eq:excessRisk}.
\end{proof}

\begin{remark}[What Proposition~\ref{prop:versionCount} predicts]
\label{rem:versionCountBridge}
% [PROOF-CHECK] Same unit fix as in (ii) below. The Setting fixes history-response pairs as the unit of data, and $M$ counts pairs.
% ORIG: Two predictions follow, and both are visible in the experiments. (i) With a single training version, the across-version term equals $B\sqrt{\log(4|\Theta|/\delta)/2}$ and does not depend on $M$, so adding trajectories from the same version cannot tighten the guarantee on other versions. Fig.~\ref{fig:trainingSamplesBCTWContextLength}(a) shows single-version \bc\ saturating with more samples while multi-version \tw\ keeps improving, although the two settings also differ in the response tokens they train on (\S\ref{sec:SampleAblationDetails}).
Two predictions follow, and both are visible in the experiments. (i) With a single training version, the across-version term equals $B\sqrt{\log(4|\Theta|/\delta)/2}$ and does not depend on $M$, so {adding history-response pairs from the same version} cannot tighten the guarantee on other versions. Fig.~\ref{fig:trainingSamplesBCTWContextLength}(a) shows single-version \bc\ saturating with more samples while multi-version \tw\ keeps improving, although the two settings also differ in the response tokens they train on (\S\ref{sec:SampleAblationDetails}). 
% [PROOF-CHECK] The Setting fixes history-response pairs, not trajectories, as the unit of data, and the budget in Eq.~\ref{eq:versionCountBound} is the $kM$ pairs.
% ORIG: (ii) Spreading a fixed budget of $kM$ trajectories over more versions tightens the across-version term
{(ii) Spreading a fixed budget of $kM$ history-response pairs over more versions tightens the across-version term} without changing the within-version term, which predicts better held-out performance from multi-version training (Tab.~\ref{tab:generalization}, rows v$_{1:5}$ versus v$_1$ and v$_5$) and more consistent outcomes across versions (the lower flip rate in Tab.~\ref{tab:flipRate}). We note that the bound controls the behavior-cloning loss rather than the success rate, and relating the two requires the standard imitation-learning argument that the task cost grows with the per-step imitation error \cite{ross2011reduction}. The i.i.d.\ assumption on versions is an idealization, since historical versions form a sequence. Proposition~\ref{prop:coverage} covers the case of a fixed, non-random target version.
\end{remark}

\subsection{Risk on a fixed unseen version is controlled by its distance to the training versions}
\label{sec:theoryCoverage}
When the target is a particular version $u\notin\mathcal{V}_{\mathrm{tr}}$, for instance, the next redesign of a site, no distribution over versions is needed.

\begin{proposition}[Coverage bound]
\label{prop:coverage}
% [PROOF-CHECK] Nonemptiness is needed. The right-most expression of Eq.~\ref{eq:coverageBound} divides by $k=|\mathcal{V}_{\mathrm{tr}}|$ and the middle one is a minimum over $\mathcal{V}_{\mathrm{tr}}$, so both are undefined when $\mathcal{V}_{\mathrm{tr}}$ is empty, and the printed inequality fails there.
% ORIG: For every $\theta\in\Theta$, every set of training versions $\mathcal{V}_{\mathrm{tr}}$, and every version $u\in\mathcal{V}$,
% [PROOF-CHECK] Two fixes. The Setting uses $k$ for the number of draws with repetition, while this proposition counts distinct versions, so the two differ whenever a draw repeats and a separate letter $m$ is used. And Assumption~\ref{ass:bounded} is what makes the supremum in Eq.~\ref{eq:discrepancy} a finite number, which the statement needs but did not invoke.
% ORIG: {For every $\theta\in\Theta$, every nonempty set of training versions $\mathcal{V}_{\mathrm{tr}}$ with $k=|\mathcal{V}_{\mathrm{tr}}|$, and every version $u\in\mathcal{V}$,}
{Given Assumption~\ref{ass:bounded}, the discrepancy of Eq.~\ref{eq:desc} is finite. For every $\theta\in\Theta$, every nonempty set of training versions $\mathcal{V}_{\mathrm{tr}}$ with $m=|\mathcal{V}_{\mathrm{tr}}|$, and every version $u\in\mathcal{V}$,}
\begin{equation}
\label{eq:coverageBound}
    \mathcal{L}_u(\theta)\;\le\;\min_{v\in\mathcal{V}_{\mathrm{tr}}}\Big\{\mathcal{L}_v(\theta)+\mathrm{disc}_\Theta(u,v)\Big\}
% ORIG: \;\le\;\frac{1}{k}\sum_{v\in\mathcal{V}_{\mathrm{tr}}}\mathcal{L}_v(\theta)\;+\;\frac{1}{k}\sum_{v\in\mathcal{V}_{\mathrm{tr}}}\mathrm{disc}_\Theta(u,v).
    \;\le\;{\frac{1}{m}}\sum_{v\in\mathcal{V}_{\mathrm{tr}}}\mathcal{L}_v(\theta)\;+\;{\frac{1}{m}}\sum_{v\in\mathcal{V}_{\mathrm{tr}}}\mathrm{disc}_\Theta(u,v).
\end{equation}
The middle expression is non-increasing in $\mathcal{V}_{\mathrm{tr}}$ with respect to set inclusion.
\end{proposition}

{\emph{Takeaway: the loss on an unseen version is at most the loss on the closest training version plus their discrepancy, so adding a training version can only tighten that minimum, and a training version at discrepancy zero from $u$ makes it equal to $\mathcal{L}_u(\theta)$. The right-most expression bounds the same quantity by the average risk over the training versions plus the average discrepancy to them. When the $k$ training versions are distinct, $m=k$ and that average is the population counterpart of the pooled objective $\hat{\mathcal{L}}_{\mathrm{tr}}$ that \timewarp-\textsc{BC} minimizes, and replacing one by the other costs the within-version term of Proposition~\ref{prop:versionCount}.}}

\begin{proof}
By Defn.~\ref{def:discrepency}, $\mathcal{L}_u(\theta)-\mathcal{L}_v(\theta)\le\mathrm{disc}_\Theta(u,v)$ for every $v\in\mathcal{V}_{\mathrm{tr}}$, which gives the first inequality after taking the minimum over $v$. A minimum over a finite set can only decrease when the set grows, which gives monotonicity. The second inequality holds because a minimum is at most an average.
\end{proof}

\begin{remark}[What Proposition~\ref{prop:coverage} predicts]
\label{rem:coverageBridge}
If the discrepancy grows with the amount of change between versions, the bound for a policy trained on v$_1$ alone grows with the distance of the test version from v$_1$. This is the pattern in Tab.~\ref{tab:generalization}, where Qwen-3 4B trained on v$_1$ degrades monotonically from v$_1$ to v$_5$, and it is the sense in which \S\ref{sec:Results} states that failure tracks the amount of change rather than its presence. Training on v$_{1:5}$ places every historical version at discrepancy zero from some training version and leaves only the held-out v$_6$ term. Because the discrepancy is measured through the observation modality, the same pair of versions can be close under \html\ or \axt\ yet far under \sshot\ or \som, which is consistent with the larger version sensitivity of visual agents in Tab.~\ref{tab:mainBenchmarkVisual}. The discrepancy term cannot be estimated without data from $u$, so the bound should be read as directional. Tab.~\ref{tab:versionValidation} provides observable proxies for it, such as the variation in accessibility-tree tokens and interactive elements across versions.
\end{remark}
}

\section{Experiment Details}
\label{sec:experimentDetailsAppendix}

Following \S\ref{sec:Results}, we provide additional details on the baselines selected for our benchmark (\S\ref{sec:baselineSelection}), human baseline (\S\ref{sec:humanBaseline}), training details (\S\ref{sec:trainingDetails}), checkpoints of the models used in our benchmark (\S\ref{sec:checkpoints}), and experimental settings (\S\ref{sec:AdditionalExperimentalSettings}).

\subsection{Baseline Selection}
\label{sec:baselineSelection}
We reviewed existing benchmarks to identify commonly used open-source models and found Qwen 2.5 \cite{qwen2025qwen25technicalreport} and Llama-3.1 \cite{grattafiori2024llama} to be the {most popular} ones. Based on the technical report, Qwen-3 \cite{yang2025qwen3} surpassed Qwen-2.5 across benchmarks, and also performed better in our toy experiments. We hence selected multiple variants of the Qwen-3 family models. Similarly, for Vision Language Models (VLMs), we selected the Qwen-3VL family models and also included Gemma-3 \cite{team2025gemma} due to its strong performance on agentic and visual benchmarks. We deliberately restrict the baselines to open-source models: a central part of our study is training (behavior cloning and its ablations), which requires access to model weights and is therefore not possible with closed-source frontier models. The selected models are recent and widely adopted in web-agent training pipelines.

% We looked into other benchmarks and identified what were the models used there. We mostly found researchers using the Qwen 2.5 and Llama-3.1 models as open-source models. Then we evaluated benchmarks on UI understanding and agentic capabilities of these models. We found that the latter Qwen-3 models surpass their 2.5 counterparts and adopted them for our methods. Given the wide popularity of the Llama-3.1 8B model as a web agent, we used it as a baseline for training as well. Unsurprisingly, the trained llama model far outperformed its zero-shot counterparts in our experiments, indicating that the llama-3.1 family of models is a strong baseline for training web agents.
% For VLMs, we chose the Qwen-3 family of VL models. For another family of models, we chose Gemma-3 due to stronger agentic and grounding benchmark results at the same parameter size. 
\subsection{Human Baseline}
\label{sec:humanBaseline}
We provide details of the human baseline in Tab. \ref{tab:mainBenchmarkVisual}. We evaluate human performance using a two-stage pipeline. First, participants attempt all tasks in the dataset while interacting with a \textit{single version} of the environment, randomly picked for each task. This allows us to evaluate which tasks participants find easier to solve. {Second,} participants retry a subset of tasks across all six versions of the environment, which includes all incorrectly answered tasks from the previous stage and correctly solved tasks that we consider more challenging. Participants are not informed of the correctness of their previous attempt. The final success rate of humans is based on the results from these two stages.

% Based on the results from these two stages, we report the final human success rate by aggregating across environments, website versions, and tasks.

\begin{minipage}[t]{0.48\textwidth}
    \vspace{0pt} % Aligns tops
    \subsection{Training Details}
    \label{sec:trainingDetails}
    The models were trained on a cluster of four Nvidia H200 GPUs. We use Llamafactory \cite{zheng2024llamafactory} to train our models with the hyperparameters reported in Tab \ref{tab:hyperparams}. The remaining hyperparameters are set to their default values. To improve memory efficiency, we use DeepSpeed ZeRO-3 \cite{rajbhandari2020zero} and gradient checkpointing \cite{chen2016training} during training. The sequence length or context window during training and the number of training epochs have been tuned based on the empirical results from \S\ref{sec:ContextLengthAblation}, \ref{sec:epochResults}.

\end{minipage}
\hfill
\begin{minipage}[t]{0.49\textwidth}
    \vspace{0pt} % Aligns tops
    \centering
    \captionof{table}{Training Hyperparameters of the \timewarp-\textsc{BC} models.}
    \label{tab:hyperparams}
    \vspace{-4pt}
    \small
    \begin{tabular}{ll}
    \toprule
    \textbf{Hyperparameter} & \textbf{Value} \\
    \midrule
    Attention Mechanism & Flash Attention 2 \\
    Sequence Length & 65,536 \\
    Max Samples & 10,000 \\
    Per-device Batch Size & 1 \\
    Grad. Accum. Steps & 4 \\
    Learning Rate & $1.0 \times 10^{-5}$ \\
    LR Scheduler & Cosine \\
    Warmup Ratio & 0.1 \\
    Training Epochs & 3 \\
    Precision & BF16 \\
    \bottomrule
    \end{tabular}
\end{minipage}

\subsection{Model Checkpoints}
\label{sec:checkpoints}
The checkpoints of the models used for trajectory collection, evaluation, and benchmarking are reported in Tab \ref{tab:modelCheckpoints}. 

\begin{table}[h]
\centering
\caption{Checkpoints of various models used in our benchmark.}
\label{tab:modelCheckpoints}
       \vspace{4pt}
\small
\begin{tabular}{lll}
\toprule
\textbf{Module} & \textbf{Model} & \textbf{Checkpoint (OpenAI/HuggingFace)} \\
\midrule
Judge           & GPT-5.1       & \texttt{openai/gpt-5.1-2025-11-13}                  \\
Judge (Open-Source) & Gemma-4 12B & \texttt{google/gemma-4-12b-it}                  \\
Executor        & GPT-5        & \texttt{openai/gpt-5-2025-08-07}                  \\ \midrule
\multirow{3}{*}{LLM Web Agent}   & Qwen-3 4B       & \texttt{Qwen/Qwen3-4B-Instruct-2507} \\
 & Qwen-3 4B Thinking & \texttt{Qwen/Qwen3-4B-Thinking-2507} \\
& Llama-3.1 8B & \texttt{meta-llama/Llama-3.1-8B-Instruct} \\ \midrule
\multirow{3}{*}{VLM Web Agent}   & Qwen-3 VL 8B       & \texttt{Qwen/Qwen3-VL-8B-Instruct} \\
& Qwen-3 VL 8B Thinking      & \texttt{Qwen/Qwen3-VL-8B-Thinking} \\
& Gemma-3 12B & \texttt{google/gemma-3-12b-it} \\
\bottomrule
\end{tabular}

\end{table}

\subsection{Additional Experimental Settings}
\label{sec:AdditionalExperimentalSettings}
% ORIG: We provide details to the \S\ref{sec:experimentalSetup}'s experimental setting on \timetraj\ (\S\ref{sec:timeTrajExperimentalSetting}), experiments on version generalization (\S\ref{sec:versionGeneralizationExpSetting}),
{We provide details on the experimental settings of \S\ref{sec:experimentalSetup} for \timetraj\ (\S\ref{sec:timeTrajExperimentalSetting}), version generalization (\S\ref{sec:versionGeneralizationExpSetting}),} cross-dataset learning (\S\ref{sec:datasetContinualLearningExpSetting}), version continual learning (\S\ref{sec:versionCLExpSetting}), sample count (\S\ref{sec:SampleAblationDetails}), and ablations (\S\ref{sec:atmpAblationExperimentalSetting}-\ref{sec:trainingEpochExpSetting}).

% For this section we introduce new notation.
% We denote the trajectory dataset corresponding to version $v$ as $\mathcal{D}_{\tau,v}$, and across all versions as $\mathcal{D}_{\tau,1:6}$ (or simply $\mathcal{D}_{\tau}$, as used previously). 

\subsubsection{Trajectory Collection (\timetraj)}
\label{sec:timeTrajExperimentalSetting}
The planner was implemented as an LLM council\footnote{\url{https://github.com/karpathy/llm-council}}, using GPT-5, Gemini-3 Pro, and Grok 4.1. As the plan generated by the planner is subsequently refined by humans, the choice of the planner model is not a crucial design decision but simply a human preference. The executor uses the GPT-5 model. During inference, the executor uses thinking, planning, memory, response history, and the accessibility tree as its observation mode. % ORIG: The inferences were similarly done on BrowserGym, and all the observation modes
{Inference was likewise run on BrowserGym, and all the observation modes} (HTML, accessibility tree, UI Screenshot, and Set of Marks) were collected along with the full response (action, thinking, planning, and memory tokens) of the executor agent. 

\subsubsection{Version Generalization}
\label{sec:versionGeneralizationExpSetting}
To evaluate the generalization capabilities of the web agents on {held-out} versions, we consider the performance of the zero-shot and \timewarp-\textsc{BC} agents, trained on trajectories from version v$_1$, version v$_5$, and versions v$_1$ to v$_5$, denoted as $\mathcal{D}_{\tau,1}$, $\mathcal{D}_{\tau,5}$, and $\mathcal{D}_{\tau,1:5}$, respectively. 

\subsubsection{Cross-Dataset Learning}
\label{sec:datasetContinualLearningExpSetting}
We analyze the effect of additional training samples from other datasets during fine-tuning. We first finetune our \timewarp-\textsc{BC} agents on the training samples from WebArena-Lite \cite{Liu2024VisualAgentBenchTLBG}, then continually finetune on the \timewarp\ training trajectories. The trained agents are evaluated on version v$_6$ across all environments and tasks. 

\subsubsection{Version Continual Learning}
\label{sec:versionCLExpSetting}
We assess whether our \timewarp-\textsc{BC} agents can be continually fine-tuned on new versions. We first train the models on v$_6$ training trajectories, $\mathcal{D}_{\tau,6}$, and evaluate their performance on the corresponding v$_6$ environments. We then continually finetune on v$_1$ training trajectories, $\mathcal{D}_{\tau,1}$, and re-evaluate the performance on the v$_6$ environments. The models are fully fine-tuned with $64$k training context, but for a single epoch to prevent overfitting on the smaller training dataset.

\subsubsection{Sample Count}
\label{sec:SampleAblationDetails}
%%% Finalized by Farhan
For this experiment, we consider two settings: sampling $k\%$ of the total trajectories $|\mathcal{D}_{\tau,1:6}|$ from (i) a single version v$_6$, \textit{i.e.}, $\{\tau_i\}_{i=1}^{k\%\cdot|\mathcal{D}_{\tau,1:6}|} \sim \mathcal{D}_{\tau,6}$  and (ii) from all versions v$_{1:6}$, \textit{i.e.}, $\{\tau_i\}_{i=1}^{k\%\cdot|\mathcal{D}_{\tau,1:6}|} \sim \mathcal{D}_{\tau,1:6}$. We fully fine-tune the Qwen-3 4B model's behavior cloning and \timewarp-\textsc{BC} variants with a $64$k context window for $3$ epochs. The models are evaluated on v$_6$ environments and tasks.

\begin{figure}[ht]
\begin{minipage}{0.99\textwidth}
\begin{center}
    {\NoHyper\ref{sharedlegend}}
\end{center}
\vspace{-0.4cm}
    \begin{tikzpicture}
\begin{axis}[
    title={\textbf{Wiki}},
    width=0.49\columnwidth,
    height=4.5cm,
    xlabel={Version},
    ylabel={Avg. Success Rate (\%)},
    ymin=0,
    ymax=26,
    ytick={0,5,10,15,20,25},
    grid=both,
    major grid style={dashed, gray!55},
    minor grid style={gray!20},
    minor tick num=1,
    symbolic x coords={1,2,3,4,5,6},
    xtick=data,
    tick label style={font=\small},
    legend style={
        at={(0.5,0.5)},
        anchor=north,
        legend columns=3,
        font=\small,
        draw=none,
        /tikz/every even column/.append style={column sep=10pt}
    },
    legend image post style={line width=1pt},
    legend to name=sharedlegend,
]
\addplot[
    color=teal!80!black,
    solid,
    line width=1pt,
    mark=o,
    mark size=2.25pt,
    mark options={solid, fill=teal!80!black}
] coordinates { (1,15.5) (2,15.9) (3,16.9) (4,17.6) (5,15.9) (6,17.0) };
\addlegendentry{AXT}

% =========================
% HTML (orange triangles)
% =========================
\addplot[
    color=orange!90!black,
    solid,
    line width=1pt,
    mark=triangle,
    mark size=3pt,
    mark options={fill=none, line width=0.6pt}
] coordinates { (1,16.1) (2,15.1) (3,16.9) (4,15.4) (5,13.6) (6,15.4) };
\addlegendentry{HTML}

% =========================
% BCv6 (blue diamonds)
% =========================
\addplot[
    color=blue!80!black,
    solid,
    line width=1pt,
    mark=diamond,
    mark size=3pt,
    mark options={fill=none, line width=0.6pt}
] coordinates { (1,3.2) (2,3.2) (3,5.4) (4,4.7) (5,7.5) (6,8.3) };
\addlegendentry{BCv$_6$}

% =========================
% BCv1:6 (red plus)
% =========================
\addplot[
    color=red!70!black,
    solid,
    line width=1pt,
    mark=+,
    mark size=3pt
] coordinates { (1,5.8) (2,5.0) (3,3.2) (4,3.6) (5,5.4) (6,6.1) };
\addlegendentry{BCv$_{1:6}$}

\addplot[
    color=cyan!70!black,
    solid,
    line width=1pt,
    mark=asterisk,
    mark size=3pt
] coordinates { (1,13.6) (2,13.3) (3,11.9) (4,13.6) (5,10.0) (6,17.6) };
\addlegendentry{BCTv$_{1:6}$}

% =========================
% TW (violet squares)
% =========================
\addplot[
    color=violet!85!black,
    solid,
    line width=1pt,
    mark=square,
    mark size=2.25pt,
    mark options={fill=none, line width=0.6pt}
] coordinates { (1,20.8) (2,19.7) (3,24.7) (4,17.6) (5,16.1) (6,20.8) };
\addlegendentry{TW}

\end{axis}
\end{tikzpicture}
\hfill
   \begin{tikzpicture}
\begin{axis}[
    title={\textbf{News}},
    width=0.49\columnwidth,
    height=4.5cm,
    xlabel={Version},
    ylabel={Avg. Success Rate (\%)},
    ymin=2,
    ymax=47,
    ytick={0,10,20,30,40},
    grid=both,
    major grid style={dashed, gray!55},
    minor grid style={gray!20},
    minor tick num=1,
    symbolic x coords={1,2,3,4,5,6},
    xtick=data,
    tick label style={font=\small},
]

% Series 1 (teal circles)
\addplot[
    color=teal!80!black,
    solid,
    line width=1pt,
    mark=o,
    mark size=2.25pt,
    mark options={solid, fill=teal!80!black}
] coordinates { (1,20.2) (2,17.2) (3,17.2) (4,19.7) (5,17.2) (6,21.2) };

% Series 2 (orange triangles)
\addplot[
    color=orange!90!black,
    solid,
    line width=1pt,
    mark=triangle,
    mark size=3pt,
    mark options={fill=none, line width=0.6pt}
] coordinates { (1,21.2) (2,19.2) (3,18.2) (4,21.2) (5,16.2) (6,22.2) };

% Series 3 (blue diamonds)
\addplot[
    color=blue!80!black,
    solid,
    line width=1pt,
    mark=diamond,
    mark size=3pt,
    mark options={fill=none, line width=0.6pt}
] coordinates { (1,23.8) (2,24.7) (3,27.3) (4,24.3) (5,5.5) (6,26.2) };

% Series 4 (red plus)
\addplot[
    color=red!70!black,
    solid,
    line width=1pt,
    mark=+,
    mark size=3pt
] coordinates { (1,21.2) (2,24.2) (3,20.2) (4,22.7) (5,23.7) (6,22.2) };

\addplot[
    color=cyan!70!black,
    solid,
    line width=1pt,
    mark=asterisk,
    mark size=3pt
] coordinates { (1,31.3) (2,31.3) (3,31.3) (4,29.8) (5,34.9) (6,31.3) };

% Series 5 (violet squares)
\addplot[
    color=violet!85!black,
    solid,
    line width=1pt,
    mark=square,
    mark size=2.25pt,
    mark options={fill=none, line width=0.6pt}
] coordinates { (1,41.9) (2,42.4) (3,43.4) (4,40.4) (5,35.8) (6,43.4) };

\end{axis}
\end{tikzpicture}

\vspace{1em}
    \begin{tikzpicture}
\begin{axis}[
    title={\textbf{Shop}},
    width=0.49\columnwidth,
    height=4.5cm,
    xlabel={Version},
    ylabel={Avg. Success Rate (\%)},
    ymin=2,
    ymax=32,
    ytick={5,10,15,20,25,30},
    grid=both,
    major grid style={dashed, gray!55},
    minor grid style={gray!20},
    minor tick num=1,
    symbolic x coords={1,2,3,4,5,6},
    xtick=data,
    tick label style={font=\small},
]

% Series 1 (teal circles)
\addplot[
    color=teal!80!black,
    solid,
    line width=1pt,
    mark=o,
    mark size=2.25pt,
    mark options={solid, fill=teal!80!black}
] coordinates { (1,15.2) (2,14.4) (3,16.9) (4,17.3) (5,13.6) (6,16.5) };

% Series 2 (orange triangles)
\addplot[
    color=orange!90!black,
    solid,
    line width=1pt,
    mark=triangle,
    mark size=3pt,
    mark options={fill=none, line width=0.6pt}
] coordinates { (1,14.4) (2,14.4) (3,16.9) (4,17.3) (5,15.2) (6,15.6) };

% Series 3 (blue diamonds)
\addplot[
    color=blue!80!black,
    solid,
    line width=1pt,
    mark=diamond,
    mark size=3pt,
    mark options={fill=none, line width=0.6pt}
] coordinates { (1,5.8) (2,11.1) (3,10.7) (4,16.0) (5,15.2) (6,8.2) };

% Series 4 (red plus)
\addplot[
    color=red!70!black,
    solid,
    line width=1pt,
    mark=+,
    mark size=3pt
] coordinates { (1,9.9) (2,10.3) (3,9.9) (4,16.9) (5,13.6) (6,4.5) };

\addplot[
    color=cyan!70!black,
    solid,
    line width=1pt,
    mark=asterisk,
    mark size=3pt
] coordinates { (1,24.3) (2,27.6) (3,25.5) (4,26.3) (5,28.0) (6,23.9) };

% Series 5 (violet squares)
\addplot[
    color=violet!85!black,
    solid,
    line width=1pt,
    mark=square,
    mark size=2.25pt,
    mark options={fill=none, line width=0.6pt}
] coordinates { (1,29.6) (2,22.2) (3,25.9) (4,28.8) (5,25.1) (6,24.7) };
\end{axis}
\end{tikzpicture}
\hfill
    \begin{tikzpicture}
\begin{axis}[
    title={\textbf{Multi}},
    width=0.49\columnwidth,
    height=4.5cm,
    xlabel={Version},
    ylabel={Avg. Success Rate (\%)},
    ymin=0,
    ymax=47,
    ytick={0,10,20,30,40},
    grid=both,
    major grid style={dashed, gray!55},
    minor grid style={gray!20},
    minor tick num=1,
    symbolic x coords={1,2,3,4,5,6},
    xtick=data,
    tick label style={font=\small},
]

% Series 1 (teal circles)
\addplot[
    color=teal!80!black,
    solid,
    line width=1pt,
    mark=o,
    mark size=2.25pt,
    mark options={solid, fill=teal!80!black}
] coordinates { (1,18.8) (2,16.4) (3,13.5) (4,16.9) (5,12.6) (6,16.9) };

% Series 2 (orange triangles)
\addplot[
    color=orange!90!black,
    solid,
    line width=1pt,
    mark=triangle,
    mark size=3pt,
    mark options={fill=none, line width=0.6pt}
] coordinates { (1,16.4) (2,16.9) (3,15.5) (4,17.9) (5,13.0) (6,18.3) };

% Series 3 (blue diamonds)
\addplot[
    color=blue!80!black,
    solid,
    line width=1pt,
    mark=diamond,
    mark size=3pt,
    mark options={fill=none, line width=0.6pt}
] coordinates { (1,5.8) (2,1.9) (3,2.9) (4,7.7) (5,7.2) (6,7.2) };

% Series 4 (red plus)
\addplot[
    color=red!70!black,
    solid,
    line width=1pt,
    mark=+,
    mark size=3pt
] coordinates { (1,8.2) (2,4.3) (3,7.2) (4,5.8) (5,4.8) (6,6.7) };

\addplot[
    color=cyan!70!black,
    solid,
    line width=1pt,
    mark=asterisk,
    mark size=3pt
] coordinates { (1,30.4) (2,36.2) (3,34.8) (4,33.8) (5,29.4) (6,32.4) };

% Series 5 (violet squares)
\addplot[
    color=violet!85!black,
    solid,
    line width=1pt,
    mark=square,
    mark size=2.25pt,
    mark options={fill=none, line width=0.6pt}
] coordinates { (1,41.1) (2,44.9) (3,38.2) (4,38.6) (5,39.6) (6,40.1) };

\end{axis}
\end{tikzpicture}
\end{minipage}
\caption{Average success rate (\%) of the text models (Qwen-3 4B, Qwen-3 4B Thinking, and Llama-3.1 8B) across different versions, environments, observations: \html, and Accessibility Tree (\axt), and training settings: Behavior Cloning (\bc) on v$_1$ and v$_{1:6}$, Behavior Cloning using Thinking traces on v$_{1:6}$ (\bct), and \timewarp-\textsc{BC}(\tw). }
\label{fig:envPerformanceTextModels}
\end{figure}
% \vspace{-0.5cm}

\subsubsection{Training Token Ablation}
\label{sec:atmpAblationExperimentalSetting}
%% Wriute this formally

We analyze the contribution of different types of tokens in the agent's response and ablate across combinations of the non-action tokens, \textit{i.e.}, thinking, memory, and planning. Following \S\ref{sec:teacherRolloutsMethod}, we denote the agent's response as a 4-tuple $y=\langle a,c,p,m\rangle$, where $a$, $c$, $p$, and $m$ {represent} the action, chain-of-thought reasoning, planning, and memory tokens, respectively. We consider the modified response $y'$ across eight settings, each a combination of the non-action tokens. % ORIG: In the settings where the full-response $y$ isn't used, a parser will extract the modified response
{In the settings where the full response $y$ is not used, a parser extracts the modified response}, \textit{i.e.}, $y'=\phi(y)$, and the agent is trained using the behavior cloning loss on the modified response, following Eq.~\ref{eq:timewarpBC}:

\begin{equation}
\label{eq:timewarpBCATMP}
    \mathcal{L}_{\text{BC}}(\theta)
    = - \mathbb{E}_{(h,\phi(y)) \sim D} \big[ \log \pi_\theta(y' \mid h) \big],
\end{equation}
where $y'\in\{\langle a\rangle,\langle a,c\rangle,\langle a,p\rangle,\langle a,m\rangle,\langle a,c,p\rangle,\langle a,c,m\rangle,\langle a,p,m\rangle\}$. The models are fully finetuned with $64$k training context and evaluated on v$_6$ across all environments and tasks.

\subsubsection{Training Context Window Length Ablation}
\label{sec:trainingContextExpSetting}
To analyze the impact of the training context window length on the model's performance, we fully fine-tune our best-performing agent, Qwen-3 4B, using \timewarp-\textsc{BC} for $3$ epochs % [REFINE: story] The listed values start at 4096, so the range is 4k to 64k, not 8k.
% ORIG: by varying the \texttt{Sequence Length} hyperparameter in LlamaFactory from 8k to 64k.
{by varying the \texttt{Sequence Length} hyperparameter in LlamaFactory from 4k to 64k.} The hyperparameter is tuned in multiples of $2$, specifically, these values: $4096$, $8192$, $16384$, $32768$, and $65536$, corresponding to $4$k, $8$k, $16$k, $32$k, and $64$k, respectively. The evaluation is performed on v$_6$ environments and tasks.

\begin{figure}[ht]
\begin{minipage}{0.99\textwidth}
\begin{center}
    {\NoHyper\ref{sharedlegend2}}
\end{center}
\vspace{-0.3cm}
    \begin{tikzpicture}
\begin{axis}[
    title={\textbf{Wiki}},
    width=0.49\columnwidth,
    height=4.5cm,
    xlabel={Version},
    ylabel={Avg. Success Rate (\%)},
    ymin=0,
    ymax=20,
    ytick={0,5,10,15,20},
    grid=both,
    major grid style={dashed, gray!55},
    minor grid style={gray!20},
    minor tick num=1,
    symbolic x coords={1,2,3,4,5,6},
    xtick=data,
    tick label style={font=\small},
    legend style={
        at={(0.4,0.5)},
        anchor=north,
        legend columns=4,
        font=\small,
        draw=none,
        /tikz/every even column/.append style={column sep=10pt}
    },
    legend image post style={line width=1pt},
    legend to name=sharedlegend2,
]

% =========================
% AXT (teal circles)
% =========================
\addplot[
    color=teal!80!black,
    solid,
    line width=1pt,
    mark=o,
    mark size=2.25pt,
    mark options={solid, fill=teal!80!black}
] coordinates { (1,14.0) (2,11.8) (3,15.8) (4,16.9) (5,16.1) (6,16.5) };
\addlegendentry{AXT}

% =========================
% HTML (orange triangles)
% =========================
\addplot[
    color=orange!90!black,
    solid,
    line width=1pt,
    mark=triangle,
    mark size=3pt,
    mark options={fill=none, line width=0.6pt}
] coordinates { (1,14.0) (2,12.9) (3,13.3) (4,15.4) (5,13.6) (6,17.9) };
\addlegendentry{HTML}

% =========================
% SS (blue diamonds)
% =========================
\addplot[
    color=blue!80!black,
    solid,
    line width=1pt,
    mark=diamond,
    mark size=3pt,
    mark options={fill=none, line width=0.6pt}
] coordinates { (1,3.6) (2,2.9) (3,1.8) (4,1.8) (5,3.6) (6,1.8) };
\addlegendentry{SS}

% =========================
% SoM (red plus)
% =========================
\addplot[
    color=red!70!black,
    solid,
    line width=1pt,
    mark=+,
    mark size=3pt
] coordinates { (1,3.2) (2,2.5) (3,4.7) (4,6.8) (5,7.2) (6,12.2) };
\addlegendentry{SoM}

\end{axis}
\end{tikzpicture}
\hfill
\begin{tikzpicture}
\begin{axis}[
    title={\textbf{News}},
    width=0.49\columnwidth,
    height=4.5cm,
    xlabel={Version},
    ylabel={Avg. Success Rate (\%)},
    ymin=2,
    ymax=40,
    ytick={5, 15, 25, 35},
    grid=both,
    major grid style={dashed, gray!55},
    minor grid style={gray!20},
    minor tick num=1,
    symbolic x coords={1,2,3,4,5,6},
    xtick=data,
    tick label style={font=\small},
]

% =========================
% AXT (teal circles)
% =========================
\addplot[
    color=teal!80!black,
    solid,
    line width=1pt,
    mark=o,
    mark size=2.25pt,
    mark options={solid, fill=teal!80!black}
] coordinates { (1,31.0) (2,25.2) (3,27.3) (4,31.3) (5,24.8) (6,35.3) };

% =========================
% HTML (orange triangles)
% =========================
\addplot[
    color=orange!90!black,
    solid,
    line width=1pt,
    mark=triangle,
    mark size=3pt,
    mark options={fill=none, line width=0.6pt}
] coordinates { (1,29.3) (2,27.8) (3,27.3) (4,28.8) (5,26.3) (6,36.4) };

% =========================
% SS (blue diamonds)
% =========================
\addplot[
    color=blue!80!black,
    solid,
    line width=1pt,
    mark=diamond,
    mark size=3pt,
    mark options={fill=none, line width=0.6pt}
] coordinates { (1,6.6) (2,5.0) (3,7.6) (4,6.0) (5,5.6) (6,6.6) };

% =========================
% SoM (red plus)
% =========================
\addplot[
    color=red!70!black,
    solid,
    line width=1pt,
    mark=+,
    mark size=3pt
] coordinates { (1,8.1) (2,6.0) (3,7.0) (4,22.7) (5,9.6) (6,26.3) };

\end{axis}
\end{tikzpicture}

\vspace{1em}
\begin{tikzpicture}
\begin{axis}[
    title={\textbf{Shop}},
    width=0.49\columnwidth,
    height=4.5cm,
    xlabel={Version},
    ylabel={Avg. Success Rate (\%)},
    ymin=0,
    ymax=21,
    ytick={0,5,10,15,20},
    grid=both,
    major grid style={dashed, gray!55},
    minor grid style={gray!20},
    minor tick num=1,
    symbolic x coords={1,2,3,4,5,6},
    xtick=data,
    tick label style={font=\small},
]

% =========================
% AXT (teal circles)
% =========================
\addplot[
    color=teal!80!black,
    solid,
    line width=1pt,
    mark=o,
    mark size=2.25pt,
    mark options={solid, fill=teal!80!black}
] coordinates { (1,15.6) (2,10.3) (3,14.4) (4,15.2) (5,11.1) (6,13.2) };

% =========================
% HTML (orange triangles)
% =========================
\addplot[
    color=orange!90!black,
    solid,
    line width=1pt,
    mark=triangle,
    mark size=3pt,
    mark options={fill=none, line width=0.6pt}
] coordinates { (1,14.8) (2,11.5) (3,11.5) (4,16.9) (5,11.9) (6,14.4) };

% =========================
% SS (blue diamonds)
% =========================
\addplot[
    color=blue!80!black,
    solid,
    line width=1pt,
    mark=diamond,
    mark size=3pt,
    mark options={fill=none, line width=0.6pt}
] coordinates { (1,4.1) (2,3.7) (3,2.9) (4,4.1) (5,2.5) (6,3.3) };

% =========================
% SoM (red plus)
% =========================
\addplot[
    color=red!70!black,
    solid,
    line width=1pt,
    mark=+,
    mark size=3pt
] coordinates { (1,2.9) (2,7.4) (3,13.2) (4,19.8) (5,12.3) (6,16.5) };

\end{axis}
\end{tikzpicture}
\hfill
\begin{tikzpicture}
\begin{axis}[
    title={\textbf{Multi}},
    width=0.49\columnwidth,
    height=4.5cm,
    xlabel={Version},
    ylabel={Avg. Success Rate (\%)},
    ymin=-1,
    ymax=22,
    ytick={0,5,10,15,20},
    grid=both,
    major grid style={dashed, gray!55},
    minor grid style={gray!20},
    minor tick num=1,
    symbolic x coords={1,2,3,4,5,6},
    xtick=data,
    tick label style={font=\small},
]

% =========================
% AXT (teal circles)
% =========================
\addplot[
    color=teal!80!black,
    solid,
    line width=1pt,
    mark=o,
    mark size=2.25pt,
    mark options={solid, fill=teal!80!black}
] coordinates { (1,14.9) (2,17.4) (3,17.4) (4,18.8) (5,12.5) (6,20.3) };

% =========================
% HTML (orange triangles)
% =========================
\addplot[
    color=orange!90!black,
    solid,
    line width=1pt,
    mark=triangle,
    mark size=3pt,
    mark options={fill=none, line width=0.6pt}
] coordinates { (1,15.9) (2,15.0) (3,13.5) (4,15.4) (5,10.1) (6,16.9) };

% =========================
% SS (blue diamonds)
% =========================
\addplot[
    color=blue!80!black,
    solid,
    line width=1pt,
    mark=diamond,
    mark size=3pt,
    mark options={fill=none, line width=0.6pt}
] coordinates { (1,1.0) (2,0.5) (3,1.9) (4,2.4) (5,3.4) (6,1.4) };

% =========================
% SoM (red plus)
% =========================
\addplot[
    color=red!70!black,
    solid,
    line width=1pt,
    mark=+,
    mark size=3pt
] coordinates { (1,2.4) (2,4.3) (3,3.8) (4,9.5) (5,7.2) (6,10.6) };

\end{axis}
\end{tikzpicture}

\end{minipage}
\caption{Average success rate (\%) of the vision-language models {(Qwen-3 VL 8B, Qwen-3 VL 8B Thinking, and Gemma-3 12B)} across different versions, environments, and observations: \html, Accessibility Tree (\axt), Screenshots (\sshot), and Set-of-Marks (\som). }
\label{fig:envPerformanceVisionModels}
\end{figure}

\subsubsection{LoRA Ablation}
\label{sec:loraExpSetting}
We investigate the performance of training models using LoRA \cite{hu2022lora} and verify whether it can provide benefits over full finetuning. We train the Qwen 3 4B model with a $64$k training context window using \timewarp-\textsc{BC} for varying adapter ranks and epochs. The agents are similarly evaluated on v$_6$ across all environments and tasks.

\subsubsection{Training Epoch Ablation}
\label{sec:trainingEpochExpSetting}
While many of the previous ablations have been conducted at varying epochs (\S\ref{sec:trainingContextExpSetting},\ref{sec:loraExpSetting}), in this section, we further assess the performance of our best model, Qwen-3 4B, fully fine-tuned on the full response using $64$k training context. The model is trained for {$1$ to $7$} epochs and also evaluated on v$_6$ across all environments and tasks.

% We denote the trajectory dataset for a version $v$ as $D_{\tau,v}$ and the full trajectory dataset as $D_{\tau,1:6}$ or simply $D_{\tau}$ as seen previously in the paper. For the experiment of the success rate against the number of samples, we sample $k\%$ of the original trajectory size \textit{i.e.}, $k\%$ of $|D_{\tau}|$ from the full dataset and a single version v$_6$, \textit{i.e.} , where $v=6$ for the first setting and $v=1:6$ for the second. We then fully fine-tune the Qwen-3 4B model in the standard setting, \textit{i.e.}, using $64k$ context for $3$ epochs.

%%%%%%%%% REVISE FROM HERE!!!!!!

\section{Additional Result Analysis}
\label{sec:additionalResultAnalysis}
We extend the discussion in \S\ref{sec:Results} by providing additional discussion on the flip rate (\S\ref{sec:flipRate}), the performance of the agents across the environments (\S\ref{sec:performanceEnvironment}), efficiency of the models (\S\ref{sec:modelEfficiency}), and ablations on training tokens (\S\ref{sec:trainingToken}), LoRA finetuning (\S\ref{sec:LoRAResults}), training context length (\S\ref{sec:ContextLengthAblation}), and number of training epochs (\S\ref{sec:epochResults}).

\subsection{Flip Rate: Isolating Version-Induced Failures}
\label{sec:flipRate}
While the performance range $\Delta$v captures how much success rates vary across versions, it does not separate version-induced failures from goal difficulty. We quantify this with the \textit{flip rate} (FR): the fraction of goals that an agent solves in at least one version but fails in at least one other. Since a goal solvable in one version should, in principle, be solvable in every other version, a flipped goal indicates that the version, rather than the goal, introduced the failure. We also define the per-version flip rate (FR$_{\text{v}}$) as the fraction of goals the agent solves in some other version but fails in version v. For both metrics, lower values are better, indicating that outcomes stay consistent across versions.

Table~\ref{tab:flipRate} reports the flip rates for the Qwen-3 4B agent. \timewarp-BC training lowers the flip rate overall (0.753 to 0.638) and in five of the six versions, i.e., multi-version training makes outcomes more consistent across versions rather than merely raising average success. For the zero-shot AXT baseline, flips concentrate on v$_1$ and v$_5$: v$_1$ lacks quality-of-life features such as a table of contents and standard search placement (Figs.~\ref{fig:tableofcontent}, \ref{fig:searchDifferences}), while v$_5$ introduces modern complexity and bloat such as icon-gated search and pop-up advertisements (\S\ref{sec:variationsVersions}, Fig.~\ref{fig:popUpAdsErrorQualAnalysis}).

\begin{table}[ht]
\centering
\caption{\textbf{Flip rates of the Qwen-3 4B agent.} FR is the fraction of goals solved in at least one version but failed in at least one other. FR$_{\text{v}}$ is the fraction of goals solved in some other version but failed in version v. Lower is better (best per column in bold). \timewarp-BC training (TW) lowers the flip rate overall and in five of the six versions relative to the zero-shot AXT baseline.}
\label{tab:flipRate}
\vspace{4pt}
\small
\setlength{\tabcolsep}{5pt}
\begin{tabular}{lccccccc}
\toprule
\textbf{Model} & \textbf{FR} & \textbf{FR$_{\text{v}_1}$} & \textbf{FR$_{\text{v}_2}$} & \textbf{FR$_{\text{v}_3}$} & \textbf{FR$_{\text{v}_4}$} & \textbf{FR$_{\text{v}_5}$} & \textbf{FR$_{\text{v}_6}$} \\
\midrule
Qwen-3 4B (AXT, zero-shot) & 0.753 & 0.459 & 0.435 & 0.412 & \textbf{0.412} & 0.494 & 0.365 \\
Qwen-3 4B (TW) & \textbf{0.638} & \textbf{0.362} & \textbf{0.406} & \textbf{0.304} & 0.420 & \textbf{0.333} & \textbf{0.333} \\
\bottomrule
\end{tabular}
\end{table}

\subsection{Performance across Environments}
\label{sec:performanceEnvironment}
In this section, we provide a detailed breakdown of agent performance by environment category. 
% [REFINE: plain] Two "respectively" in one sentence.
% ORIG: We extend performance averages across environments and versions from Tab \ref{tab:mainBenchmarkTextual} and Tab \ref{tab:mainBenchmarkVisual} in \S\ref{sec:textEnvVersionsResults} and \S\ref{sec:vlmEnvVersionResults}, respectively, by providing a version-wise breakdown for each environment category aggregated by the textual and visual models, respectively.
{We extend the averages in Tab.~\ref{tab:mainBenchmarkTextual} and Tab.~\ref{tab:mainBenchmarkVisual} with a version-wise breakdown per environment category, aggregated over the textual models (\S\ref{sec:textEnvVersionsResults}) and the visual models (\S\ref{sec:vlmEnvVersionResults}).}

% The Qwen-3 4B and Llama-3.1 8B \tw agents produces strong gain for across the versions. For Llama-3.1, the gains are consistent across versions, while the Qwen-3 4B model achieves more signinifcant gains in the multi-site tasks. The \bc models also achived noticeable uplifts in the News tasks for the Qwen-3 4B model. Training generally degraded the performance for the Qwen-3 4B Thinking model with minor improvements in the News and Multi-site tasks. The performance of the visual models remain somewhat consistent across the environments when using the textual observations (\html and \axt) but varies more when using the visual observation (\sshot and \som).

\begin{figure}[ht]
    \centering
\begin{minipage}{0.53\textwidth}
    \begin{tikzpicture}
\begin{axis}[
    width=0.77\textwidth, height=8cm,
    xlabel={Version},
    ylabel={Model},
    title={\textbf{Success Rate (\%) (Bubble Size) and}\\
    \textbf{Trajectory Length $|\tau|$ across Versions}},
    title style={align=center,at={(0.475,1.00)}, anchor=south},
    xmin=0.5, xmax=6.5,
    ymin=0.5, ymax=6.5,
    xtick={1,2,3,4,5,6},
    ytick={1,2,3,4,5,6},
    yticklabels={
        \html, 
        \axt, 
        \colorbox{red!15}{BCv$_6$}, 
        \colorbox{red!15}{BCv$_{1:6}$}, 
        \bct,
        \tw
    },
    scatter,
    only marks,
    % Colorbar settings for trajectory
    colorbar,
    colorbar style={
        xlabel={$|\tau|$}, % Changed title to xlabel
        xlabel style={at={(0.5, -0.02)}, anchor=north}, % Positioned at the bottom
        at={(1.05,0.5)}, 
        anchor=west,
    },
    colormap name=viridis,
    visualization depends on={value \thisrow{perf} \as \perfsiz},
    scatter/use mapped color={draw=black, fill=mapped color},
    scatter/@pre marker code/.append style={
        /tikz/mark size=\perfsiz*0.27 
    },
]

\addplot[
    scatter,
    point meta=\thisrow{traj}, 
] table [x=x, y=y] {
x y perf traj
1 1 16.8 7.398
2 1 16.2 7.828
3 1 16.8 7.582
4 1 17.7 6.958
5 1 14.4 6.997
6 1 17.6 6.932
1 2 17.2 6.841
2 2 15.9 7.608
3 2 16.2 6.359
4 2 17.8 7.188
5 2 14.8 6.358
6 2 17.7 6.843
1 3 8.9 14.848
2 3 9.6 16.812
3 3 10.9 14.356
4 3 12.5 12.912
5 3 9.1 16.550
6 3 11.9 16.052
1 4 10.7 14.450
2 4 10.3 16.508
3 4 9.5 16.660
4 4 11.6 15.686
5 4 11.3 14.023
6 4 9.3 17.566
1 5 23.0 13.045
2 5 27.2 13.285
3 5 25.3 13.081
4 5 25.5 13.252
5 5 30.4 12.333
6 5 32.7 13.806
1 6 32.1 9.078
2 6 30.9 10.463
3 6 32.0 9.243
4 6 30.1 8.987
5 6 27.9 9.938
6 6 31.0 9.680
};

\end{axis}
\end{tikzpicture}
    \caption*{(a)}
\end{minipage}
\begin{minipage}{0.46\textwidth}
    \begin{tikzpicture}
\begin{axis}[
    title={\textbf{Success Rate (\%) (Bubble Size) and}\\
    \textbf{Trajectory Length $|\tau|$ across Env}},
    title style={align=center,at={(0.45,1.00)}, anchor=south},
    width=0.78\textwidth, height=8cm,
    xlabel={Environment},
    ylabel={Model},
    xmin=0.5, xmax=4.5,
    ymin=0.5, ymax=6.5,
    xtick={1,2,3,4},
    xticklabels={Wiki, News, Shop, Multi},
    ytick={1,2,3,4,5,6},
    yticklabels={
        \html, 
        \axt, 
        \colorbox{red!10}{BCv$_6$}, 
        \colorbox{red!10}{BCv$_{1:6}$}, 
        \bct,
        \tw
    },
    scatter,
    only marks,
    colorbar,
    colorbar style={
        xlabel={$|\tau|$}, % Changed title to xlabel
        xlabel style={at={(0.5, -0.02)}, anchor=north}, % Positioned at the bottom
        at={(1.05,0.5)}, 
        anchor=west,
    },
    colormap name=viridis,
    visualization depends on={value \thisrow{perf} \as \perfsiz},
    scatter/use mapped color={draw=black, fill=mapped color},
    scatter/@pre marker code/.append style={
        /tikz/mark size=\perfsiz*0.27 
    },
]

\addplot[
    scatter,
    point meta=\thisrow{traj}, 
] table [x=x, y=y] {
x y perf traj
1 1 15.4 9.24
2 1 19.7 6.02
3 1 15.6 5.65
4 1 16.3 7.77
1 2 16.5 8.61
2 2 18.8 6.08
3 2 15.6 5.21
4 2 15.9 7.21
1 3 5.4 17.01
2 3 22.0 10.15
3 3 11.2 14.46
4 3 5.5 18.71
1 4 4.9 17.45
2 4 22.4 10.93
3 4 10.8 14.72
4 4 6.2 19.58
1 5 14.2 13.22
2 5 36.6 9.43
3 5 25.5 14.49
4 5 38.4 14.96
1 6 20.0 9.03
2 6 41.2 6.99
3 6 26.1 10.08
4 6 40.4 12.14
};
\end{axis}
\end{tikzpicture}
    \caption*{(b)}
\end{minipage}
\caption{Performance of the Qwen-3 4B model (success rate (\%) corresponds to the bubble size) against the trajectory length $|\tau|$, varying observation and training settings: HTML (\html), Accessibility Tree (\axt), Behavior Cloning (\bc) on v$_6$ and v$_{1:6}$, reasoning traces on v$_{1:6}$ (\bct), and \textsc{TimeWarp-BC} (\tw), averaged across (a) versions and (b) environments.}
\label{fig:trajectoryPerformanceDiagram}

\end{figure}

\subsubsection{Textual Models across Env Versions}
\label{sec:textEnvVersionsResults}
Fig.~\ref{fig:envPerformanceTextModels} presents a version-wise breakdown of performance across environments, averaged over the textual models: Qwen-3 4B, Llama-3.1 8B, and Qwen-3 4B Thinking. The \tw\ method consistently outperforms the other approaches across both environments and versions. Performance across versions is also relatively stable, indicating high robustness. {In contrast,} \bc\ shows significant variations across versions, \textit{e.g.}, \bc v$_6$ in News v$_5$ gets only a $\sim5\%$ success rate, compared to $>20\%$ in other versions. Similar inconsistencies can be observed in the Shop environment.

\subsubsection{Vision-Language Models across Env Versions}
\label{sec:vlmEnvVersionResults}
Fig. \ref{fig:envPerformanceVisionModels} provides a similar breakdown for the % ORIG: Vision Language Models (VLMs): Qwen-3 8B, Qwen-3 8B Thinking, and Gemma-3 12B.
{Vision Language Models (VLMs): Qwen-3 VL 8B, Qwen-3 VL 8B Thinking, and Gemma-3 12B.} The VLMs appear more robust across versions when using the textual observations (\html\ and \axt). For visual observations, the performance using \sshot\ remains relatively stable across versions but is generally outperformed by \som. However, \som\ shows substantial variations across versions \textit{e.g.}, a success rate of $<3\%$ in Shop v$_1$ \textit{vs.} $\sim20\%$ in Shop v$_4$. This inconsistency is observed across all environments, making \som\ unsuitable for deployment.

\subsection{Efficiency of Models}
\label{sec:modelEfficiency}
\textit{Is there a correlation between performance and trajectory length?} For the Qwen-3 4B model, higher performance tends to {coincide with} shorter trajectories across versions (Fig. \ref{fig:trajectoryPerformanceDiagram}(a)) and environments (Fig. \ref{fig:trajectoryPerformanceDiagram}(b)). This is consistent with previous research findings, where models get stuck in the middle \cite{qi2024webrl} and agents {fail} slowly \cite{yoran2407assistantbench, Song2025BEARCUBSABU}. When agents are unable to make meaningful progress toward solving the task, they tend to take redundant or irrelevant actions until the maximum step limit of $30$ is reached. The trajectory length also varies significantly across environments more than versions, \textit{e.g.}, in Fig. \ref{fig:trajectoryPerformanceDiagram}(b), % [REFINE: symbol] Category is "Multi" in Fig.~\ref{fig:trajectoryPerformanceDiagram}(b); "\tw performance" also lost its space.
% ORIG: while the \tw performance in News and Multi shop is similar, the multi-shop tasks require substantially more steps than the news tasks.
{while the \tw\ performance in News and Multi is similar, the multi-site tasks require substantially more steps than the news tasks.}

\begin{figure}[ht]
    \centering
    \subfigure[\textbf{LoRA Ablation.} LoRA significantly underperforms full fine-tuning (FFT) at adapter ranks $r \in \{4,8,16\}$ for the Qwen-3 4B model trained with a $64$k context window, across varying epochs on the v$_6$ environments.]{
        \begin{tikzpicture}
\begin{axis}[
    width=0.48\columnwidth,
    height=5.25cm,
    xlabel={Epoch},
    ylabel={Success Rate (\%)},
    ymin=0,
    ymax=43,
    ytick={0,10,20,30,40},
    grid=both,
    major grid style={dashed, gray!55},
    minor grid style={gray!20},
    minor tick num=1, % adds one minor grid between majors
    symbolic x coords={1,3,5,7},
    xtick=data,
    tick label style={font=\small},
    legend style={
        at={(0.4,.0)},
        anchor=south,
        legend columns=4,
        font=\small,
        line width=0.5pt    
    },
    line width=1pt,
    axis line style={line width=0.5pt},
    mark options={solid},
]
% 4
\addplot[
    color=teal!90,
    mark=*,
    solid,
    line width=1.25pt,
    mark options={fill=white}
] coordinates {
    (1,17.48)
    (3,22.00)
    (5,19.43)
    (7,21.36)
};

% 8
\addplot[
    color=blue!70,
    mark=square*,
    solid,
    line width=1.25pt,
    mark options={fill=white}
] coordinates {
    (1,15.21)
    (3,21.03)
    (5,20.39)
    (7,17.46)
};

% 16
\addplot[
    color=orange!70,
    mark=triangle*,
    solid,
    line width=1.25pt,
    mark size=2.5pt,
    mark options={fill=white}
] coordinates {
    (1,22.00)
    (3,20.71)
    (5,17.16)
    (7,14.88)
};

% FFT (highlighted)
\addplot[
    color=purple!95!black,
    line width=1.5pt,
    dashed,
    mark=star,
    mark size=3.2pt
] coordinates {
    (1,31.38)
    (3,38.54)
    (5,30.10)
    (7,30.75)
};

\legend{$r=4$, $r=8$, $r=16$, FFT}

\end{axis}
\end{tikzpicture}
    }
    \hfill
    \subfigure[Average trajectory length of the Qwen-3 4B model on the v$_6$ environment when trained using (i) full fine-tuning (FFT) with context lengths in $\{4,8,16,32,64\}$k and (ii) LoRA adapters with rank $r\in\{4,8,16\}$ and a $64$k context window. The trendlines show LoRA producing shorter average trajectories than FFT.]{
        \begin{tikzpicture}
\begin{axis}[
    width=0.48\columnwidth,
    height=5cm,
    xlabel={Epoch},
    ylabel={Avg. Trajectory Length},
    ymin=0,
    ymax=16,
    ytick={0,3,6,9,12,15},
    grid=both,
    major grid style={dashed, gray!55},
    minor grid style={gray!20},
    minor tick num=1,
    symbolic x coords={1,3,5,7},
    xtick=data,
    tick label style={font=\small},
    legend style={
        at={(0.5,1.18)},
        anchor=north,
        legend columns=4,
        font=\small,
        draw=none,
        /tikz/every even column/.append style={column sep=5pt}
    }
]

% -------------------------
% LoRA points (individual)
% -------------------------
\addplot[
    only marks,
    color=teal!60!black,
    mark=o,
    mark size=2.2pt,
    line width=0.6pt,
    xshift=-3pt
] coordinates { (1,3.94) (3,4.07) (5,4.11) (7,4.52) };
\addlegendentry{LoRA}

\addplot[forget plot, only marks, color=teal!60!black, mark=o, mark size=2.2pt, line width=0.6pt, xshift=0pt] 
coordinates { (1,4.07) (3,3.83) (5,3.97) (7,4.34) };

\addplot[forget plot, only marks, color=teal!60!black, mark=o, mark size=2.2pt, line width=0.6pt, xshift=3pt] 
coordinates { (1,3.74) (3,4.76) (5,4.38) (7,4.89) };

% LoRA average line (Solid per caption)
\addplot[
    color=teal!70!black,
    solid, % Fixed: was dashed
    line width=1.2pt,
    mark=none
] coordinates { (1,3.92) (3,4.22) (5,4.15) (7,4.58) };
\addlegendentry{LoRA Avg}

% -------------------------
% FFT points (individual)
% -------------------------
\addplot[
    only marks,
    color=violet!55!black,
    mark=square,
    mark size=2.4pt,
    mark options={fill=none, line width=0.6pt},
    xshift=0pt
] coordinates { (1,9.39) (3,13.27) (5,14.90) (7,13.13) };
\addlegendentry{FFT}

% Additional FFT points with forget plot
\addplot[forget plot, only marks, color=violet!55!black, mark=square, mark size=2.4pt, xshift=-6pt] coordinates {(1,11.26) (3,11.37) (5,11.12) (7,10.90)};
\addplot[forget plot, only marks, color=violet!55!black, mark=square, mark size=2.4pt, xshift=-3pt] coordinates {(1,9.80) (3,10.66) (5,10.61) (7,10.65)};
\addplot[forget plot, only marks, color=violet!55!black, mark=square, mark size=2.4pt, xshift=3pt] coordinates {(1,9.82) (3,10.47) (5,11.42) (7,10.21)};
\addplot[forget plot, only marks, color=violet!55!black, mark=square, mark size=2.4pt, xshift=6pt] coordinates {(1,10.33) (3,9.93) (5,9.59) (7,10.90)};

% FFT average line (Dashed per caption)
\addplot[
    color=violet!75!black,
    dashed,
    line width=1.2pt,
    mark=none
] coordinates { (1,10.12) (3,11.14) (5,11.53) (7,11.16) };
\addlegendentry{FFT Avg}

\end{axis}
\end{tikzpicture}
    }
    \caption{Success rate (a) and average trajectory length (b) of LoRA fine-tuning compared to full-fine-tuning.}
    \label{fig:loraAblation}
\end{figure}

\begin{figure}[h]
\centering

\begin{minipage}{0.48\columnwidth}
\centering
\begin{tikzpicture}
\begin{axis}[
    width=\linewidth,
    height=4.75cm,
    xlabel={$\%$ of Training Data},
    ylabel={Success Rate (\%)},
    ymin=0,
    ymax=40,
    xmin=-5,
    grid=both,
    major grid style={dashed, gray!50},
    minor tick num=1,
    minor grid style={gray!20},
    tick label style={font=\small},
    legend style={
        at={(0.7,0.05)},
        anchor=south,
        legend columns=2
    },
    xtick={0,20,40,60,80,100},
]
\addplot[semithick, blue!60, mark=*, mark size=2.5pt] coordinates {
    (0,21.0) (20,25.6) (40,27.2) (60,28.5) (80,31.6) (100,36.3)
};
\addplot[semithick, red!60, mark=square*, mark size=2.5pt] coordinates {
    (0,21.0) (20,5.5) (40,18.4) (60,21.7) (80,23.6) (100,21.4)
};
\legend{TW, BCv$_6$}
\end{axis}
\end{tikzpicture}
% \vspace{-0.2cm}
\caption*{(a)}
\end{minipage}
\hfill
\begin{minipage}{0.48\columnwidth}
\centering
\begin{tikzpicture}
\begin{axis}[
    width=\columnwidth,
    height=5cm,
    xlabel={Epoch},
    ylabel={Success Rate (\%)},
    ymin=15,
    ytick={10,15,20,25,30,35},
    grid=both,
    major grid style={dashed, gray!50},
    minor tick num=0,
    minor grid style={gray!20},
    symbolic x coords={1,3,5,7},
    xtick=data,
    tick label style={font=\small},
    legend style={
        at={(0.5,.05)},
        anchor=south,
        legend columns=5,
        font=\small,
        line width=0.5pt    
    },
    line width=1pt,
     axis line style={line width=0.5pt},
    mark options={solid},
]

% 4k
\addplot[
    color=cyan!70,
    mark=o
] coordinates {
    (1,22.67)
    (3,24.22)
    (5,23.61)
    (7,21.33)
};

% 8k
\addplot[
    color=blue!40,
    mark=square*
] coordinates {
    (1,22.31)
    (3,26.52)
    (5,30.13)
    (7,26.53)
};

% 16k
\addplot[
    color=violet!60,
    mark=triangle*
] coordinates {
    (1,26.21)
    (3,25.24)
    (5,30.75)
    (7,25.88)
};

% 32k
\addplot[
    color=red!60,
    mark=diamond*
] coordinates {
    (1,25.23)
    (3,28.16)
    (5,26.46)
    (7,22.34)
};

% 64k
\addplot[
    color=orange!70,
    mark=*
] coordinates {
    (1,29.44)
    (3,36.92)
    (5,28.81)
    (7,30.75)
};

\legend{4k, 8k, 16k, 32k, 64k}

\end{axis}
\end{tikzpicture}

% \vspace{-0.2cm}
\caption*{(b)}
\end{minipage}

\vspace{0cm}
\caption{(a) Success rate (\%) of the Qwen-3 4B model trained using \tw\ \textit{vs.} vanilla behavior cloning (\bc v$_6$), across different training data sizes. (b) \textbf{Context Length Ablation (Overall):} The success rate (\%) for a fully fine-tuned Qwen-3 4B model against its trainable context length on the v$_6$ environments and tasks.}
\label{fig:trainingSamplesBCTWContextLength}
\end{figure}

\begin{figure*}[ht]
    \centering
    \includegraphics[width=0.95\linewidth]{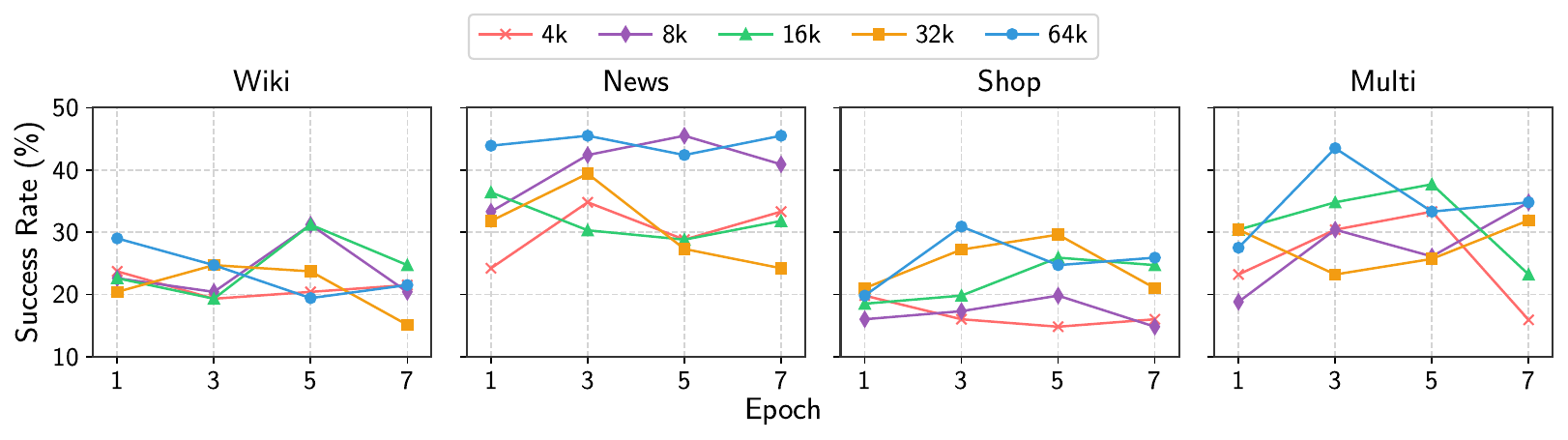}
    \caption{\textbf{Context Length Ablation:} The success rate (\%) of Qwen-3 4B \tw\ against its trainable context length on the $v_6$ environments. While there is no clear winner across all four environments, $64$k context at epoch $3$ performs relatively better.}
    \label{fig:contextLengthAblationAllEnv}
\end{figure*}

\begin{table}[ht]
\centering
\caption{Success rate (\%) of models fully-finetuned under \textsc{TimeWarp-BC} (\tw) and $64$k training context across varying epochs and environments. The best performance is achieved by all the models when trained for $3$ epochs.}
       \vspace{8pt}
\label{tab:epochTable}
\begin{tabular}{lcccccc}
\toprule
\textbf{Model} & \textbf{Epoch} & \textbf{Wiki} & \textbf{News} & \textbf{Shop} & \textbf{Multi} & \textbf{Overall} \\
\midrule
\multirow{3}{*}{Qwen3-4B}
 & 1 & 29.0\% & 43.9\% & 19.8\% & 36.2\% & 31.4\% \\
 & 3 & \textbf{31.2\%} & \textbf{47.0\%} & \textbf{30.9\%} & \textbf{49.3\%} & \textbf{38.5\%} \\
 & 5 & 19.4\% & 42.4\% & 24.7\% & 39.1\% & 30.1\% \\
\midrule
\multirow{3}{*}{Qwen3-4B-Th}
 & 1 & 12.9\% & 28.8\% & 21.0\% & 20.3\% & 20.1\% \\
 & 3 & \textbf{14.0\%} & \textbf{40.9\%} & 21.0\% & \textbf{37.7\%} & \textbf{26.9\%} \\
 & 5 & 12.9\% & 37.9\% & \textbf{22.2\%} & 33.3\% & 25.2\% \\
\midrule
\multirow{3}{*}{Llama-3.1 8B}
 & 1 & 11.8\% & 34.8\% & 01.2\%  & 31.9\% & 18.4\% \\
 & 3 & \textbf{17.2\%} & \textbf{42.4\%} & \textbf{22.2\%} & \textbf{33.3\%} & \textbf{27.5\%} \\
 & 5 & \textbf{17.2\%} & 36.4\% & 21.0\% & 30.4\% & 25.2\% \\
\bottomrule
\end{tabular}
\end{table}

\subsection{Training Tokens}
\label{sec:trainingToken}
The response to each observation in the training data is a tuple comprising action, thinking, memory, and planning tokens. In this section, we conduct an ablation study of all possible combinations of non-action tokens used during training. Tab. \ref{tab:ablationTMP} shows that different training token combinations have their strengths in different v$_6$ environments for the Qwen-3 4B model. However, some overarching themes are prevalent. For instance, the inclusion of thinking tokens and similar non-action tokens improves performance, consistent with previous findings \cite{wei-etal-2025-webagent, hu-etal-2025-webcot}. The performance gains from training on non-action tokens are relatively higher for multi-site tasks and lower for news tasks. Overall, including all tokens, \textit{i.e.}, training on the full response, produced the best results.

% As the answer response is a tuple containing action, thinking, memory, and planning tokens, we ablate on each possible combination of the non-action tokens provided during training. Tab. \ref{tab:ablationATMPactionTokens} shows that different training settings might be suitable for different environments. We observe that the inclusion of thinking tokens during training generally improves performance, which is consistent with the previous works . However, including both planning and action tokens often result in an accuracy drop. This might be due to the fact that these tokens often have similar responses, and can often confuse the agent in taking the right action \red{refer to an error analysis sub diagram where the thinking and planning tokens are contradicting each other}. Overall, we observe the ATMP setting, which includes all response tokens, perform marginally better than the ATM setting, which excludes the planning tokens only. 

\subsection{Full \textit{vs.} LoRA Finetuning}
\label{sec:LoRAResults}
LoRA finetuning \cite{hu2022lora} has provided significant performance gains for models across several scenarios \cite{mao2025survey}, but falls short when finetuning web agents. % [REFINE: jargon] FFT defined at first use.
% ORIG: Fig. \ref{fig:loraAblation}(a) shows that LoRA underperforms FFT significantly across different adapter ranks.
{Fig.~\ref{fig:loraAblation}(a) shows that LoRA significantly underperforms full fine-tuning (FFT) across different adapter ranks.} This is consistent with previous literature, which shows that LoRA underperforms FFT baselines \cite{biderman2024lora}. % ORIG: We, however, observe several interesting properties of LoRA finetuning, including making the model extremely trajectory efficient (Fig. \ref{fig:loraAblation}(b)).
{We do, however, observe several interesting properties of LoRA fine-tuning, including that it makes the model far more trajectory-efficient (Fig.~\ref{fig:loraAblation}(b)).} However, the low average trajectory length is partly due to the model tending to stop earlier than its FFT counterpart, with incorrect results.

\subsection{Training Context Length Ablation}
\label{sec:ContextLengthAblation}
Fig.~\ref{fig:trainingSamplesBCTWContextLength}(b) shows that increasing the training context length consistently improves performance across most training epochs. The best overall performance is achieved with a $64$k context length after $3$ training epochs. Additionally, in Fig. \ref{fig:contextLengthAblationAllEnv}, the benefits of higher context length are most pronounced in the multi-site tasks. Very low training context, such as $4$k, generally underperforms across all environments and training epochs. The substantial performance variance across context lengths emphasizes the training context window as a critical hyperparameter, which requires careful tuning.

% In Fig. \ref{fig:trainingSamplesBCTWContextLength} (b), we observe that higher context generally improves performance across most of the training epochs, and the highest performance is achieved when trained for three epochs with $64$k context. In Fig. \ref{fig:contextLengthAblationAllEnv}, we observe that the gains for 64k context are more prominent in the multi-site setting and least prominent in the Wiki. Lower training context, such as $4$k, generally underperforms across all environments and training epochs, indicating that increasing context length up to the maximum possible length for web agents could produce significant gains in performance. 

\subsection{Impact of \#Epochs}
\label{sec:epochResults}
In this section, we dive deeper into the impact of the number of training epochs on the agents' performance by focusing on our best-performing configuration for each model, \textit{i.e.}, fully fine-tuned with the complete response and a $64$k training context. As shown in Tab.~\ref{tab:epochTable} and consistent with previous findings, all three models achieve their highest average performance across the environments after training for $3$ epochs. Training for fewer epochs leads to underfitting, whereas increasing the number of epochs beyond three results in mild overfitting, as reflected by the performance plateau observed at epoch five.

\vspace{0.4cm}

\begin{minipage}{\linewidth}
  \begin{minipage}[t]{0.52\linewidth}
    \subsection{Filtered \textit{vs.} Unfiltered Training Trajectories}
\label{sec:filterVsUnfilteredTrainingTrajectories}
\vspace{0.2cm}
    The student agent's performance is inherently tied to the quality of the teacher agent's trajectories. To verify whether failed trajectories harm student performance, we retrain Qwen-3 4B, Qwen-3 4B Thinking, and Llama-3.1 8B on both filtered and unfiltered data. As shown in Table~\ref{tab:filtered_vs_unfiltered}, performance is nearly identical across both settings, confirming that student models are robust to failed trajectories in their training data when using \timetraj.
  \end{minipage}%
  \hfill
  \begin{minipage}[t]{0.46\linewidth}
    \vspace{0pt}
    \centering
    \captionof{table}{Success rate (\%) of models trained using filtered and unfiltered \timetraj\ trajectories, showing near-identical performance.}
    \label{tab:filtered_vs_unfiltered}
    \begin{tabular}{@{}lcc@{}}
      \toprule
      \textbf{Model} & \textbf{Filtered} & \textbf{Unfiltered} \\
      \midrule
      Qwen-3 4B          & 38.5\% & 38.1\% \\
      Qwen-3 4B Th & 26.9\% & 26.0\% \\
      Llama-3.1 8B       & 27.5\% & 28.1\% \\
      \bottomrule
    \end{tabular}
  \end{minipage}
\end{minipage}

\begin{figure}[ht]
    \centering
    \includegraphics[width=0.995\linewidth]{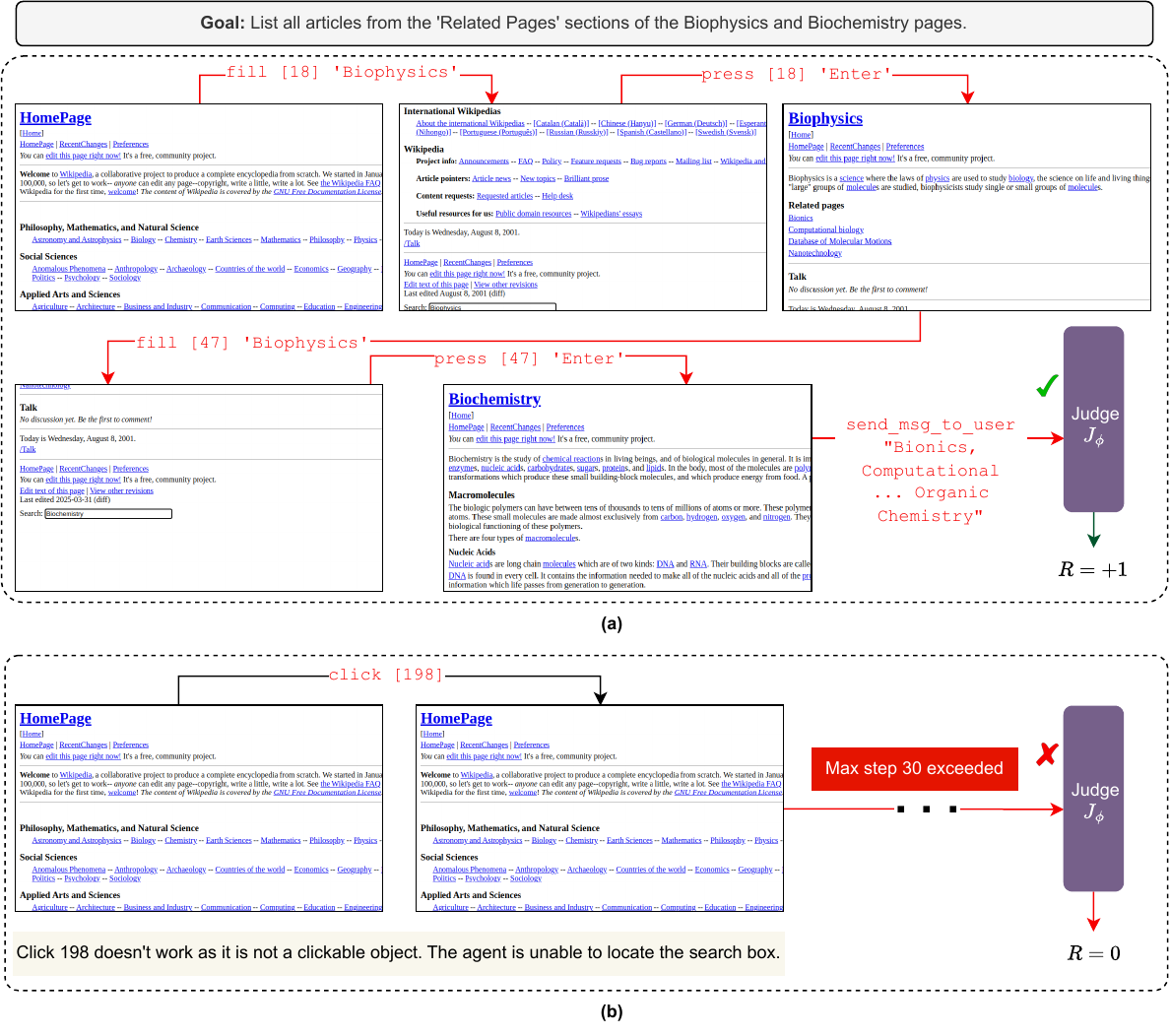}
    \caption{Example of a (a) positive trajectory by the Qwen-3 4B \tw\ model in the Wiki v$_1$ environment, and (b) negative trajectory by several other models that are unable to locate the search box in the dated Wiki v$_1$ environment. }
    \label{fig:postiveTrajectory}
\end{figure}

\subsection{Standard Error}
\label{subsec:consistencyAndStandardError}
To assess the reliability of our evaluation results, we report the mean reward alongside standard error over 3 runs of the Qwen-3 4B trained using \textsc{TimeWarp-BC} in Tab.~\ref{tab:qwen3_4b_mean_se}. Both mean and standard error values are computed for each experiment using the implementation of AgentLab \cite{tmlr25dechezellesBrowserGym}.
{The standard error values follow similar trends for the other baselines and are omitted elsewhere for conciseness.}

\begin{table}[ht]
\centering
\small
\caption{Mean reward with standard error across environments and versions for Qwen3-4B TW.}
\label{tab:qwen3_4b_mean_se}
       \vspace{4pt}
\begin{tabular}{lcccccc}
\toprule
\textbf{Env} & v1 & v2 & v3 & v4 & v5 & v6 \\
\midrule
Wiki  & $0.32 \pm 0.05$ & $0.31 \pm 0.05$ & $0.38 \pm 0.05$ & $0.29 \pm 0.05$ & $0.25 \pm 0.04$ & $0.31 \pm 0.05$ \\
News  & $0.46 \pm 0.06$ & $0.38 \pm 0.06$ & $0.44 \pm 0.06$ & $0.41 \pm 0.06$ & $0.44 \pm 0.06$ & $0.47 \pm 0.06$ \\
Shop  & $0.33 \pm 0.05$ & $0.27 \pm 0.05$ & $0.36 \pm 0.05$ & $0.31 \pm 0.05$ & $0.26 \pm 0.05$ & $0.31 \pm 0.05$ \\
Multi & $0.52 \pm 0.06$ & $0.49 \pm 0.06$ & $0.48 \pm 0.06$ & $0.49 \pm 0.06$ & $0.51 \pm 0.06$ & $0.49 \pm 0.06$ \\
\bottomrule
\end{tabular}
\end{table}

% While the context and LoRA ablations have analyzed the impact of the number of epochs on the performance of the model, in this section, we dive deeper by analyzing how the number of epochs affects our best-performing model \textit{i.e.}, fully-finetuned using the whole action response at $64$k training context across all versions of the environment. Based on Tab. \ref{tab:epochTable}, for all three trained models we observe the highest average performance when trained for $3$ epochs across all versions of the environments. The only exception is the Qwen-3 4B Thinking model with slighlty more performance gains in the Shop env when trained for 5 epochs. Hence, lower epochs indicate underfitting while higher number of epochs are indicating overfitting, which is why the performance saturates a little bit at epoch 5. 

\begin{figure}[h!]
  \centering
  \begin{tcolorbox}[
    width=\textwidth,
        colframe=black!100,
    title={Comparison of responses produced by thinking models.},
    fonttitle=\bfseries
  ]
  \small
  \textbf{Goal:} Which of these two plane crashes appears to be deadlier, based on the number of casualties: (1) ``Airliner crashes in China, 42 killed" and (2) ``Ethiopian airliner crashes into Mediterranean Sea"? Mention the title of the article related to the deadlier incident.\\
  
  \textbf{Desired Answer Output:} (2) Ethiopian airliner crashes into the Mediterranean Sea.
  \end{tcolorbox}
  
  \begin{minipage}[t]{0.48\textwidth}
    \begin{tcolorbox}[colback=orange!11,
        colframe=orange!80!black,colbacktitle=orange!80!black,
    title={Qwen-3 4B Thinking Zero-Shot using AXT}]
      \small
      $<$think$>$ Okay, let's see. The user wants to know which plane crash is deadlier between two options: (1) ``Airliner crashes in China, 42 killed" and (2) ``Ethiopian airliner crashes into the Mediterranean Sea". They need the title of the article related to the deadlier incident.

First, I need to check the current page. The active tab is the News site at http://localhost:5001/. The AXTree shows a list of news articles. Let me scan through the latest news to find mentions of these plane crashes.\\

...\\

But the current step's observation shows that the active tab is the News site. The agent needs to check if there are any relevant articles here. Since there's no mention of the two crashes, the agent should inform the user.

So the next action is send$\_$msg$\_$to$\_$user with the answer.
$<$/think$>$\\

$<$action$>$
send$\_$msg$\_$to$\_$user(`Based on the given information, the deadlier incident is ``Airliner crashes in China, 42 killed" as it specifies 42 casualties, while the Ethiopian airliner crash into the Mediterranean Sea does not provide a casualty count. Therefore, the title of the article related to the deadlier incident is ``Airliner crashes in China, 42 killed".')
$<$/action$>$
    \end{tcolorbox}
  \end{minipage}
  \hfill
  \begin{minipage}[t]{0.495\textwidth}
    \begin{tcolorbox}[colback=cyan!10,
     coltext=black,colbacktitle=cyan!80!black,
    colframe=cyan!80!black,
    title={Qwen-3 4B Thinking \timewarp-trained using AXT}]
      \small
      $<$think$>$
I need to search the News site using the search box, then open the first result and confirm the casualty count before sending the final answer.
$<$/think$>$\\

$<$plan$>$\\
1. Enter the first query into the search box and press Enter.\\
2. Open the first result ``Airliner crashes in China, 42 killed".\\
3. Confirm casualty count (42) and remember it.\\
4. Enter the second query into the search box and press Enter.\\
5. Open the first result ``Ethiopian airliner crashes into Mediterranean Sea".\\
6. Confirm casualty count and remember it.\\
7. Compare counts and send the final message indicating the deadlier incident.\\
$<$/plan$>$\\

$<$step$>$1$<$/step$>$\\

$<$memory$>$
Preparing to search for the first article by filling the search box with the query text.
$<$/memory$>$\\

$<$action$>$\\
fill(`23', `Airliner crashes in China, 42 killed')\\
$<$/action$>$
    \end{tcolorbox}
  \end{minipage}

  \caption{\textbf{Training improves the structure of Thinking Tokens.} Side-by-side comparison of the reasoning tokens produced by the Qwen-3 4B Thinking model in zero-shot and trained using \textsc{TimeWarp-BC} (\tw). The thinking tokens are more structured after training, and the response includes planning and memory tokens.}
  \label{fig:qwen4BThinkingImproved}
\end{figure}

\begin{figure}[ht]
  \centering
\begin{tcolorbox}[title={Overthinking Problem of Qwen-3 VL 8B Thinking}]
$<$think$>$\\
Okay, let's see. The user is trying to find the most recent sports event in Brazil's Wikipedia article and then check a news article about it. They're using Simple Wikipedia, which is a local instance.\\

First, they tried to fill the search bar with ``Brazil" but got an error because they used the wrong bid. The correct bid for the search input is `26', not `6' or `53'. Then they clicked on `53', which was a script element, not the search result. That's why it didn't work.\\

Now, looking at the current screenshot, the search bar (bid `26') has ``Brazil" typed in, and there's a dropdown with suggestions. The first suggestion is ``Brazil" (bid `53'), which is the actual article link. The previous attempt to click '53' was probably on the script element, but maybe the correct bid for the search result is different.\\

Wait, the user's last action was clicking `53', which was a script tag. That's not the right element. The actual search result for Brazil should be a link with a different bid. Let me check the screenshot again.\\

In the dropdown, the first item is "Brazil" with bid `53', but that's the script. Wait, no—the dropdown shows "Brazil" as a link. Maybe the bid for the actual link is `53' but the script is also bid `53'? That's confusing. Wait, the error message says "Locator.click: Element is not visible" when they tried to click `53' before. Maybe the correct bid for the search result is `53' but it's not visible yet.
$\cdots$
\end{tcolorbox}
\caption{\textbf{Overthinking Problem of the Qwen-3 VL 8B Thinking model.} The thinking token exceeds the maximum number of tokens that can be produced in a single step, resulting in an error. The total text exceeds $8$k tokens and has been truncated for conciseness.}
\label{fig:overthikingQwen3VLthinking}
\end{figure}

\begin{figure}[ht]
  \centering
\begin{tcolorbox}[title={Misalignment between action and non-action tokens.}]
\textbf{Goal:} How many articles on Morocco were published in 2009?\\
\textbf{Desired outcome:} \texttt{send\_msg\_to\_user(`One')}
\medskip
\hrule
\medskip

$<$think$>$\\
I verified pages 1–6 of the search results and found exactly one article dated 2009-12-06.\\
$<$/think$>$\\

$<$plan$>$\\
1. Confirm total 2009 articles across pages 1–6.\\
2. Send the result to the user.\\
$<$/plan$>$\\

$<$step$>$2$<$/step$>$\\

$<$memory$>$\\
I verified across pages 1–6 and found one 2009 article: ``Western Sahara campaigner still blocked from returning to Morocco; on twentieth day of hunger strike" (2009-12-06). Proceeding to report the count as per instructions.\\
$<$/memory$>$\\

$<$action$>$
send$\_$msg$\_$to$\_$user(`Two')
$<$/action$>$
\end{tcolorbox}
\caption{\textbf{Misalignment between action and non-action tokens.} Although the non-action tokens (thinking, memory, and planning) may \textit{suggest} the correct action, web agents, especially vision language models (VLMs) and untrained ones, often generate actions that don't align with their reasoning. Training mitigates the issue, but it remains one of the common pitfalls among web agents.}
\label{fig:actionMisalignment}
\end{figure}

\section{Qualitative and Error Analysis}
\label{sec:qualitativeErrorAnalysisAppendix}
In this section, we examine several qualitative examples of web agents in the \timewarp\ environment and analyze failure cases, including action misalignment, performance variations across versions, and the performance of the thinking models. 

\subsection{Action Misalignment}
\label{sec:actionMisalignmentError}
% ORIG: One of the most jarring limitations of modern web agents
{One of the most striking limitations of modern web agents} is the misalignment between the action and non-action tokens. Specifically, we often see that the agent's reasoning is correct, but it ends up picking the wrong action. We observe this in Fig. \ref{fig:actionMisalignment}, where the agent's reasoning should send the user the message \textit{One}, but instead sends \textit{Two}. This type of misalignment is more common for terminating actions and untrained (zero-shot) agents.

\subsection{Version Variations}
\label{sec:versionErrors}
Fig.~\ref{fig:postiveTrajectory} illustrates a successful trajectory by the Qwen-3 4B \tw\ model in the Wiki v$_1$ environment, where the search box is located at the bottom of the page. The model correctly identifies and interacts with the search box despite its atypical placement. On the contrary, several other models, particularly VLMs, are unable to locate the search box and get stuck {clicking} random non-interactable elements.

\subsection{Thinking Models}
\label{sec:thinkingError}
Thinking models do not get performance gains from any form of training (\S\ref{sec:Results}). Nevertheless, we observe that the thinking tokens of these models become more structured after training (Fig. \ref{fig:qwen4BThinkingImproved}). One possible explanation is that imitating the teacher model’s reasoning style may inadvertently constrain the model’s original reasoning process. Another factor may be the reduced training context length ($64$k) relative to the models’ original pretraining context length (256k), potentially limiting their reasoning abilities. We also observe in Fig. \ref{fig:overthikingQwen3VLthinking} that the Qwen-3 VL 8B Thinking model tends to \textit{overthink} the visual observation (\sshot\ and \som) and ends up with the wrong reasoning. 

\section{Reproducibility, Dataset Access, and Licensing}
\label{sec:reporducibilityDatasetAccessLicensing}
To support open science, we make all resources associated with this work publicly available. Our codebase includes documentation covering environment setup, preprocessing pipelines, training configurations, and evaluation procedures, enabling full reproduction of the results reported in this paper. % ORIG: Hyperparameter (\S\ref{tab:hyperparams}), experimental details
{Hyperparameters (Tab.~\ref{tab:hyperparams}), experimental details} (\S\ref{sec:experimentalSetup},\ref{sec:experimentDetailsAppendix}), and model checkpoints (\S\ref{sec:checkpoints}) are additionally provided in the Supplementary Material.
The environment is additionally released as a pip-installable package, \texttt{browsergym-timewarp} (v0.2.0), which registers all $231$ goals as Gym tasks within the BrowserGym \cite{tmlr25dechezellesBrowserGym} ecosystem. As $229/231$ goals are scored by deterministic verifiers (\S\ref{sec:deterministicVerifierDetails}), the reported success rates can be reproduced without access to any proprietary LLM API and without sampling variance from the evaluator. The dataset is also publicly hosted on Hugging Face, including the visual-information labels of \S\ref{sec:visualInfoAnnotation}. The released dataset stores goal-level entries (goal text, reference answers, category, split, and refined plan), which the environment instantiates into tasks across the six versions. The dataset card documents this goal-to-task instantiation with a worked example. The frontend code of every version, the evaluation scripts, and the judge prompt (\S\ref{sec:prompt}) are all included in the codebase. Both the code and dataset are released under the \textbf{MIT License}, permitting free use, modification, and redistribution for academic and commercial purposes.

\section{Declaration of LLM Usage}
\label{sec:llmUsageDeclaration}
The authors acknowledge moderate use of LLM-based coding agents in the development of the frontend components. All LLM-generated code was reviewed and verified by the authors prior to integration into the web environment. LLMs have also been used in refining prompts and providing iterative feedback during prompt engineering. The final prompts were likewise reviewed and approved by the authors. In addition, LLM agents were used to draft and adversarially validate the deterministic verifier specifications (\S\ref{sec:verifierAnnotation}).

% \clearpage
\newpage

\section{Prompts}
\label{sec:prompt}
We provide all the prompts used in our experiments.

% Preamble:
% \usepackage[most]{tcolorbox}

\begin{tcolorbox}[title={Judge Evaluator $J_\phi$ Prompt},breakable]
Help a teacher grade the answer of a student given a question. Keep in mind that the student may use different phrasing or wording to answer the question. The goal is to evaluate whether the answer is semantically equivalent to the reference answer.\\

\textbf{\#Input:}
\begin{itemize}
  \item question: $\langle \texttt{Question}\rangle$
  \item reference answer: $\langle \texttt{Reference\_Answer}\rangle$
  \item student answer: $\langle \texttt{Generated\_Answer}\rangle$
\end{itemize}

\textbf{\# Special Sequence:} The string `N/A' that you see is a special sequence that means `not achievable'.\\

\textbf{\# Output:} You must respond with EXACTLY one of the following words (nothing else):
\begin{enumerate}
  \item `correct': If the answer is semantically equivalent to the reference.
  \begin{itemize}
    \item Numeric values must match exactly (including units, signs, and scale) unless the question/reference clearly allows an approximation or rounding.
    \item If an estimate is allowed, the student must still be reasonably close and not contradict the reference.
    \item Ordered lists/steps/rankings must match exactly in both the elements' values and order.
    \item Unordered lists/sets must contain the same elements; the order of the elements does not matter.
    \item Extra information is allowed only if it does not contradict or change the meaning.
  \end{itemize}
  \item `partially correct': If the answer is somewhat related but incomplete or inaccurate.
  \item `incorrect': If the answer is wrong or unrelated.
\end{enumerate}

Do not include any additional text, explanation, or formatting. Only respond with one of the three words above.
\end{tcolorbox}

\begin{tcolorbox}[title={Planning Agent $\Pi_\text{plan}$ Prompt}, breakable]
You are a planning agent who creates step-by-step execution instructions for web-based tasks. Your instructions will guide another agent to complete tasks by interacting with the web elements. Generate a precise, executable plan that:
\begin{itemize}
    \item Breaks down the task into clear, sequential steps.
    \item Specifies exact web interactions or action verbs (type, click, navigate, check, etc.).
    \item Includes verification steps where needed.
    \item Concludes with the final action (e.g., reporting results to the user). Include the reasoning behind generating the final output.
\end{itemize}

\#\textbf{Input Format:}
\begin{itemize}
    \item \textbf{Task Goal:} \texttt{<Task\_Goal>}
    \item \textbf{Reference Answer:} \texttt{<Reference\_Answer>}
\end{itemize}

\#\textbf{Example:}\\

\textbf{Task Goal:} Is biophysics mentioned as a related page or branch in the Biology article, the Physics article, both, or neither of them?\\

\textbf{Reference Answer:} Biology\\

\textbf{Execution Plan:}
\begin{enumerate}
    \item Type `Biology' in the search box and press Enter.
    \item Check the related branches and related pages sections to see if Biophysics is present.
    \item Type `Physics' in the search box and press Enter.
    \item Check the branches and related pages sections to see if Biophysics is present.
    \item As Biophysics was only present in the Biology article's branches section, send a message to the user: `Biology'.
\end{enumerate}
\end{tcolorbox}
\begin{tcolorbox}[title={Executor/Teacher Agent $\Pi_T$ Prompt}, breakable]
$\langle\texttt{Web\_Agent}\_\texttt{\#Instructions\_Prompt}\rangle$\\

\# \textbf{Execution Instructions:} Follow these instructions step by step to successfully complete the task.\\

\hspace*{1cm}$\langle\texttt{Execution\_Plan}\rangle$\\

$\langle\texttt{Web\_Agent}\_\texttt{Remaining\_Prompt}\rangle$\\
\end{tcolorbox}
\begin{tcolorbox}[title={Web Agent $\pi_\theta$ Prompt}, breakable]
You are an agent trying to solve a web task based on the content of the page and user instructions. You can interact with the page and explore, and send messages to the user. Each time you submit an action, it will be sent to the browser, and you will receive a new page.\\

\# \textbf{Instructions:}
Review the current state of the page and all other information to find the best possible next action to accomplish your goal. Your answer will be interpreted and executed by a program. Make sure to follow the formatting instructions.\\

\#\# \textbf{Goal:} $\langle\texttt{Task}\_\texttt{Goal}\rangle$\\

\#\# \textbf{Extra instructions:}\\

\textbf{IMPORTANT:} You must only navigate to URLs within the TimeWarp environment. Do NOT navigate to external websites. Only use the URLs provided in the TimeWarp environment. 

\begin{itemize}
    \item WIKI URL: $\langle\texttt{Wiki}\_\texttt{URL}\rangle$
    \item NEWS URL: $\langle\texttt{News}\_\texttt{URL}\rangle$
    \item SHOP URL: $\langle\texttt{Shop}\_\texttt{URL}\rangle$
\end{itemize}

Example: \texttt{goto}(``$\langle\texttt{Wiki}\_\texttt{URL}\rangle$") to navigate to WIKI.\\

\textbf{CRITICAL:}
\begin{itemize}
    \item You MUST output EXACTLY ONE action per response. Do NOT attempt multiple actions at once.
    \item Output the $<\text{action}>$ tag ONLY ONCE, at the end of your response, OUTSIDE of any $<\text{think}>$, $<\text{plan}>$, $<\text{step}>$, and $<\text{memory}>$ tags.
\end{itemize}

\# \textbf{Observation of current step:}\\

\#\# \textbf{Currently open tabs:}\\

\hspace*{0.5cm}$\langle\texttt{Tabs}\_\texttt{List}\rangle$\\
\hspace*{1cm} Tab: $\langle\texttt{Tab\_ID}\rangle$ ($\langle\texttt{Active/Inactive}\rangle$ tab):\\
    \hspace*{2cm} Title: $\langle\texttt{Tab\_Title}\rangle$\\
    \hspace*{2cm} URL: $\langle\texttt{Tab\_URL}\rangle$\\

\#\# \textbf{AXTree:}\\ 

\textbf{Note:} [bid] is the unique alpha-numeric identifier at the beginning of lines for each element in the AXTree. Always use bid to refer to elements in your actions.\\

\textbf{Note:} You can only interact with visible elements. If the ``visible" tag is not present, the element is not visible on the page.\\

\hspace*{1cm}Focused Element: $\langle\texttt{BID}\rangle$\\

\hspace*{1cm}$\langle\texttt{AX\_Tree}\rangle$\\

\#\# \textbf{HTML:} $\langle\texttt{HTML}\rangle$\\

\textbf{\#\# Screenshot:}
Here is a screenshot of the page, it is annotated with bounding boxes and corresponding bids:\\

\hspace*{1cm}$\langle\texttt{Screenshot}\rangle$\\

\textbf{\#\# Error from previous action:} $\langle\texttt{Error\_Log}\rangle$\\

\textbf{\# History of interaction with the task:}\\
\hspace*{1cm}$\langle\texttt{Previous\_Responses}\rangle$\\
    \hspace*{2cm} $\langle\texttt{Thinking\_Tokens}\rangle$\\
    \hspace*{2cm} $\langle\texttt{Planning\_Tokens}\rangle$\\
    \hspace*{2cm} $\langle\texttt{Memory\_Tokens}\rangle$\\
    \hspace*{2cm} $\langle\texttt{Action\_Tokens}\rangle$\\

\textbf{\# Action Space:}\\

\textbf{Note:} This action set allows you to interact with your environment. Most of them are Python functions executing Playwright code. The primary way of referring to elements in the page is through bid, which are specified in your observations.\\

11 different types of actions are available.

\begin{itemize}
\item \texttt{scroll(delta\_x: float, delta\_y: float)}\\
\hspace*{1cm}Examples:\\
        \hspace*{2cm} \texttt{scroll(0, 200)}\\
        \hspace*{2cm} \texttt{scroll(-50.2, -100.5)}

\item \texttt{fill(bid: str, value: str, enable\_autocomplete\_menu: bool = False)}\\
\hspace*{1cm}Examples:\\
\hspace*{2cm}\texttt{fill('45', 'multi-line\textbackslash n example')}\\
\hspace*{2cm}\texttt{fill('a12', 'example with "quotes"')}\\
\hspace*{2cm}\texttt{fill('b534', 'Montre', True)}

\item \texttt{click(bid: str, button: Literal['left', 'middle', 'right'] = 'left', modifiers: list[typing.Literal['Alt', 'Control', 'ControlOrMeta', 'Meta', 'Shift']] = [])}\\
\hspace*{1cm}Examples:\\
\hspace*{2cm}\texttt{click('a51')}\\
\hspace*{2cm}\texttt{click('b22', button='right')}\\
\hspace*{2cm}\texttt{click('48', button='middle', modifiers=['Shift'])}

\item \texttt{press(bid: str, key\_comb: str)}\\
\hspace*{1cm}Examples:\\
\hspace*{2cm}\texttt{press('88', 'Backspace')}\\
\hspace*{2cm}\texttt{press('a26', 'ControlOrMeta+a')}\\
\hspace*{2cm}\texttt{press('a61', 'Meta+Shift+t')}

\item \texttt{go\_back()}\\
\hspace*{1cm}Examples:\\
\hspace*{2cm}\texttt{go\_back()}

\item \texttt{goto(url: str)}\\
\hspace*{1cm}Examples:\\
\hspace*{2cm}\texttt{goto('http://www.example.com')}

\item \texttt{send\_msg\_to\_user(text: str)}\\
\hspace*{1cm}Examples:\\
\hspace*{2cm}\texttt{send\_msg\_to\_user('Based on the results of my search, the city\\\hspace*{2cm}was built in 1751.')}

\item \texttt{report\_infeasible(reason: str)}\\
\hspace*{1cm}Examples:\\
\hspace*{2cm}\texttt{report\_infeasible('I cannot follow these instructions because\\\hspace*{2cm} there is no email field in this form.')}

\item \texttt{new\_tab()}\\
\hspace*{1cm}Examples:\\
\hspace*{2cm}\texttt{new\_tab()}

\item \texttt{tab\_close()}\\
\hspace*{1cm}Examples:\\
\hspace*{2cm}\texttt{tab\_close()}

\item \texttt{tab\_focus(index: int)}\\
\hspace*{1cm}Examples:\\
\hspace*{2cm}\texttt{tab\_focus(2)}

\end{itemize}

Only a single action can be provided at once. Example:
\texttt{fill('b534', 'Montre', True)}.\\

\textbf{\# Plan:} You just executed step $\langle\texttt{step\_number}\rangle$ of the previously proposed plan: $\langle\texttt{plan}\rangle$\\
After reviewing the effect of your previous actions, verify if your plan is still relevant and update it if necessary.\\

\# \textbf{Abstract Example:}\\

Here is an abstract version of the answer with a description of the content of each tag. Make sure you follow this structure, but replace the content with your answer:\\

$<$think$>$
Think step by step. If you need to make calculations, such as coordinates, write them here. Describe the effect
that your previous action had on the current content of the page.
$<$/think$>$\\

$<$plan$>$
Provide a multi-step plan that will guide you to accomplish the goal. There should always be steps to verify if the previous action had an effect. The plan can be revisited at each step. Specifically, if there was something unexpected. The plan should be cautious and favor exploring before submitting.
$<$/plan$>$\\

$<$step$>$Integer specifying the step of the current action.$<$/step$>$\\

$<$memory$>$
Write down anything you need to remember for the next steps. You will be presented with the list of previous memories and past actions. Some tasks require to remember hints from previous steps in order to solve them.
$<$/memory$>$\\

$<$action$>$
One single action to be executed. You can only use one action at a time.
$<$/action$>$\\

\# \textbf{Concrete Example:}

Here is a concrete example of how to format your answer. Make sure to follow the template with proper tags:\\

$<$think$>$
From the previous action, I tried to set the value of year to ``2022",
using select\_option, but it doesn't appear to be in the form. It may be a dynamic dropdown. I will try using click with the bid ``a324" and look at the response from the page.
$<$/think$>$\\

$<$plan$>$
\begin{enumerate}
    \item Fill form (failed)
    \begin{itemize}
        \item Type first name
        \item Type last name
    \end{itemize}
    \item Try to activate the form
    \begin{itemize}
        \item Click on tab 2
    \end{itemize}
    \item Fill form again
    \begin{itemize}
        \item Type first name
        \item Type last name
    \end{itemize}
    \item Verify and submit
    \begin{itemize}
        \item Verify form is filled
        \item Submit if filled, if not, replan
    \end{itemize}
\end{enumerate}
$<$/plan$>$\\

$<$step$>$2$<$/step$>$\\

$<$memory$>$
I clicked on bid "32" to activate tab 2. The accessibility tree should list the form's focusable elements at the next step.
$<$/memory$>$\\

$<$action$>$
\texttt{click('a324')}
$<$/action$>$

\end{tcolorbox}

\section{Snapshots of Version UI}
\label{sec:versionExamples}
We provide snapshots of the homepage UI across all six versions of Wiki in Fig. \ref{fig:wiki_v1}-\ref{fig:wiki_v6}, News in Fig. \ref{fig:news_v1}-\ref{fig:news_v6}, and Shop in Fig. \ref{fig:shop_v1}-\ref{fig:shop_v6}. The snapshots are presented at full resolution, exactly as the observation provided to a visual web agent via BrowserGym \cite{tmlr25dechezellesBrowserGym}.

\begin{figure}[htbp]
  \centering
\subfigure[Snapshot of the Wiki v$_1$ homepage, resembling the Wikipedia UI from $2001$. \label{fig:wiki_v1}]{%
    \fbox{\includegraphics[width=0.99\linewidth]{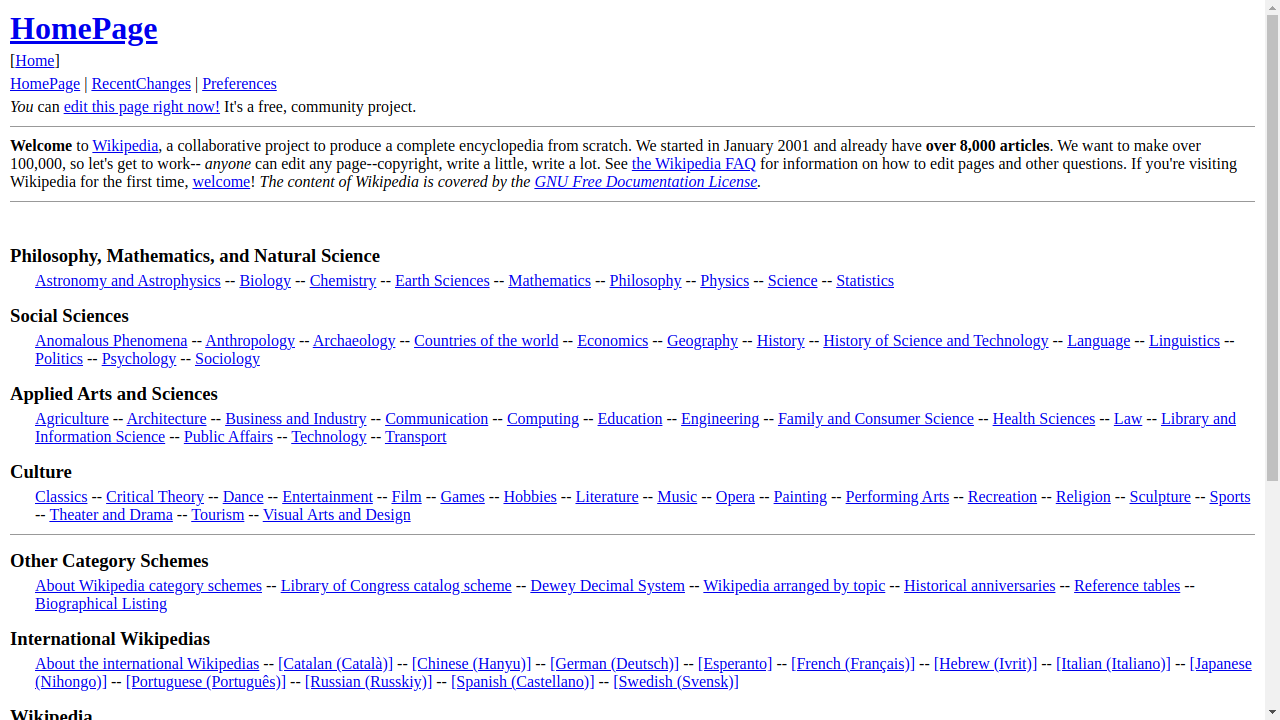}}
}\vspace{1em}
\subfigure[Snapshot of the Wiki v$_2$ homepage, resembling the Wikipedia UI from $2002-03$.\label{fig:wiki_v2}]{%
    \fbox{\includegraphics[width=0.99\linewidth]{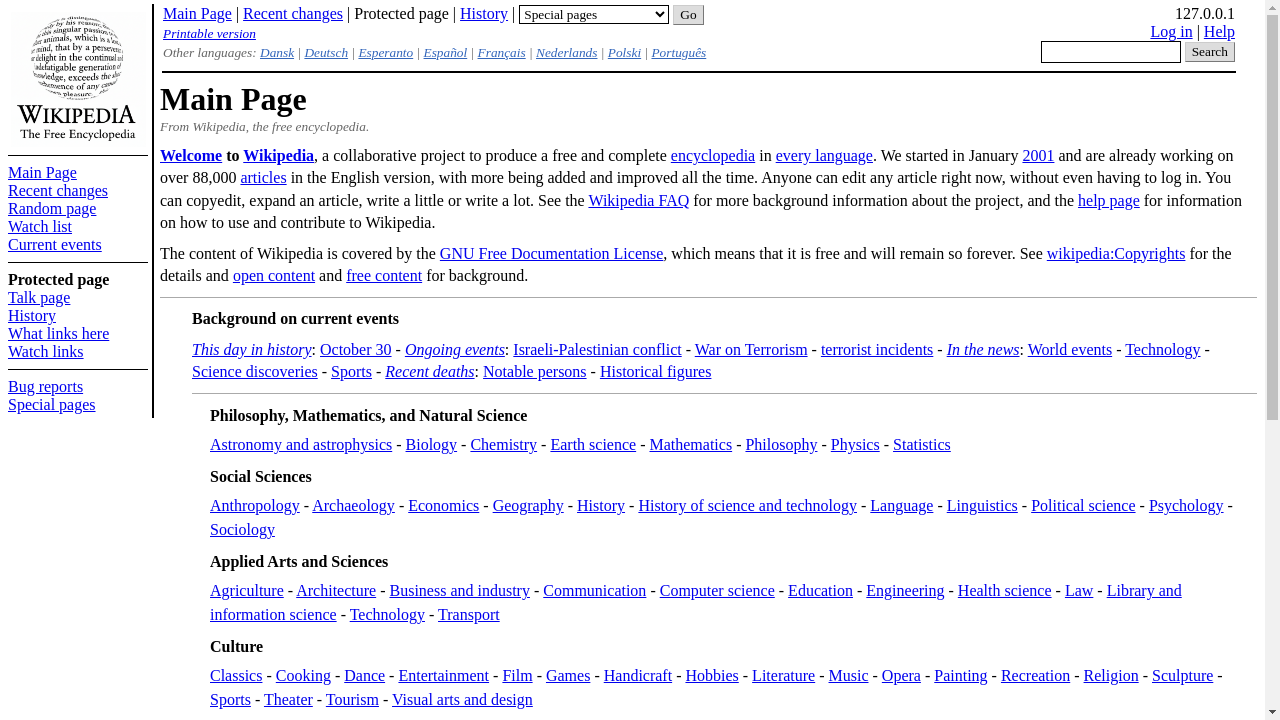}}
}
\end{figure}

\begin{figure}[htbp]
  \centering
\subfigure[Snapshot of the Wiki v$_3$ homepage, resembling the Wikipedia UI from $2003-04$. \label{fig:wiki_v3}]{%
    \fbox{\includegraphics[width=0.99\linewidth]{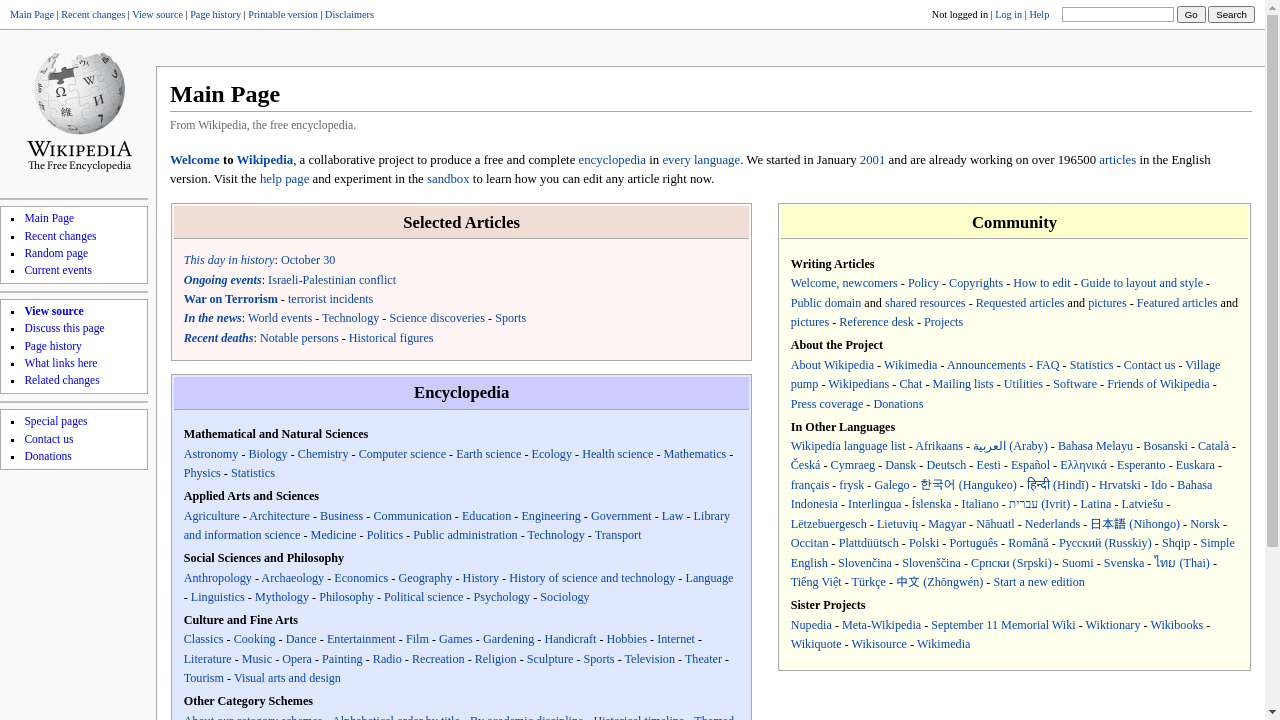}}
}\vspace{1em}
\subfigure[Snapshot of the Wiki v$_4$ homepage, resembling the Wikipedia UI from $2005-22$.\label{fig:wiki_v4}]{%
    \fbox{\includegraphics[width=0.99\linewidth]{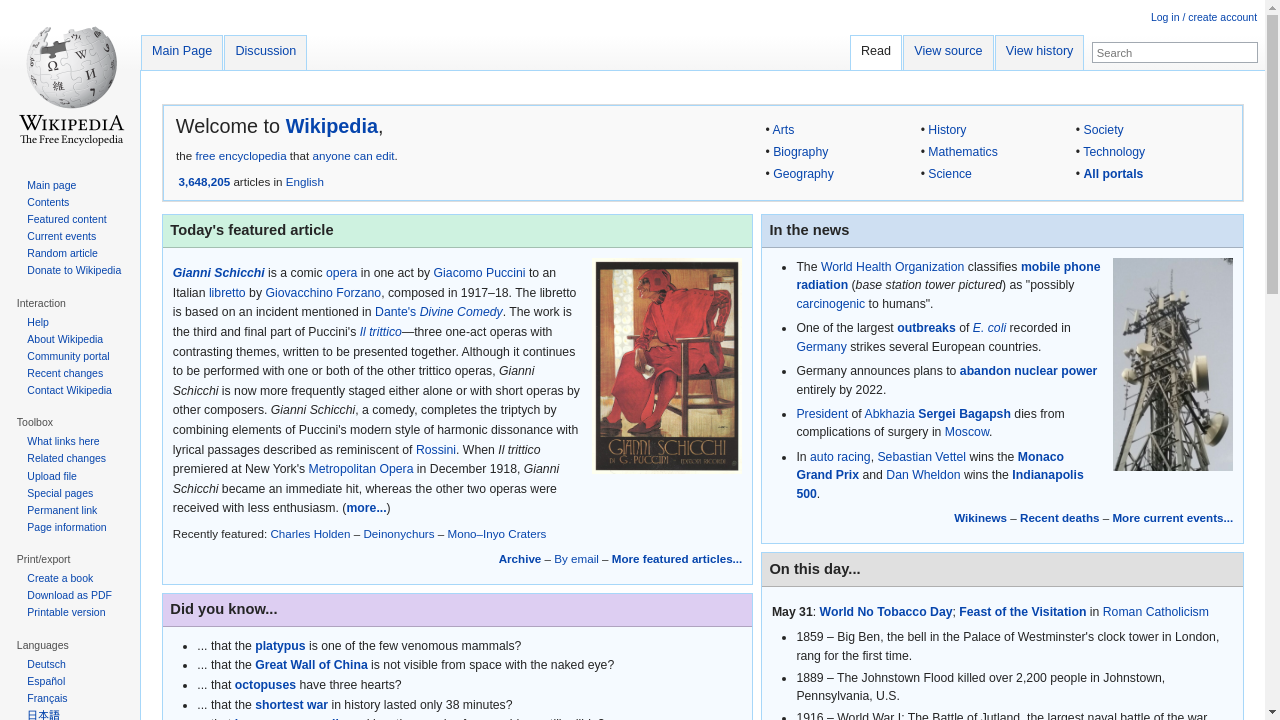}}
}
\end{figure}

\begin{figure}[htbp]
  \centering
\subfigure[Snapshot of the Wiki v$_5$ homepage, resembling the Wikipedia UI from $2023-25$. \label{fig:wiki_v5}]{%
    \fbox{\includegraphics[width=0.99\linewidth]{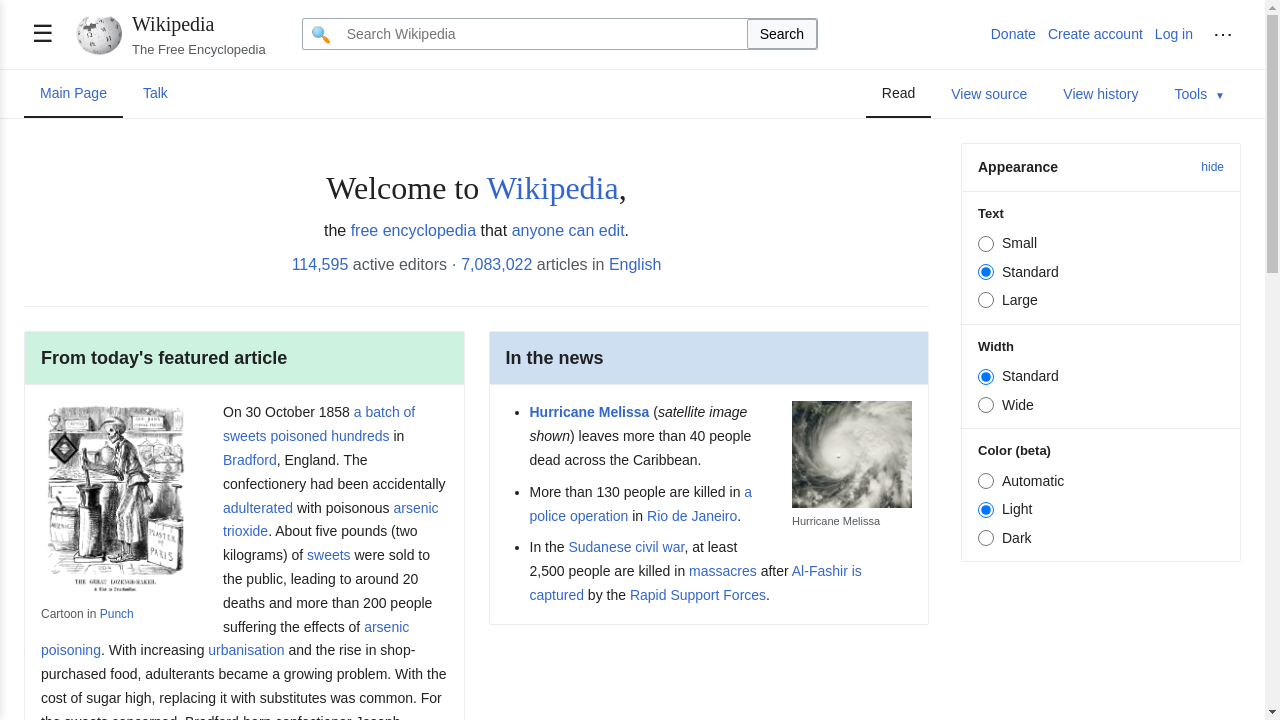}}
}\vspace{1em}
\subfigure[Snapshot of the Wiki v$_6$ homepage, resembling a minimal Wikipedia UI.\label{fig:wiki_v6}]{%
    \fbox{\includegraphics[width=0.99\linewidth]{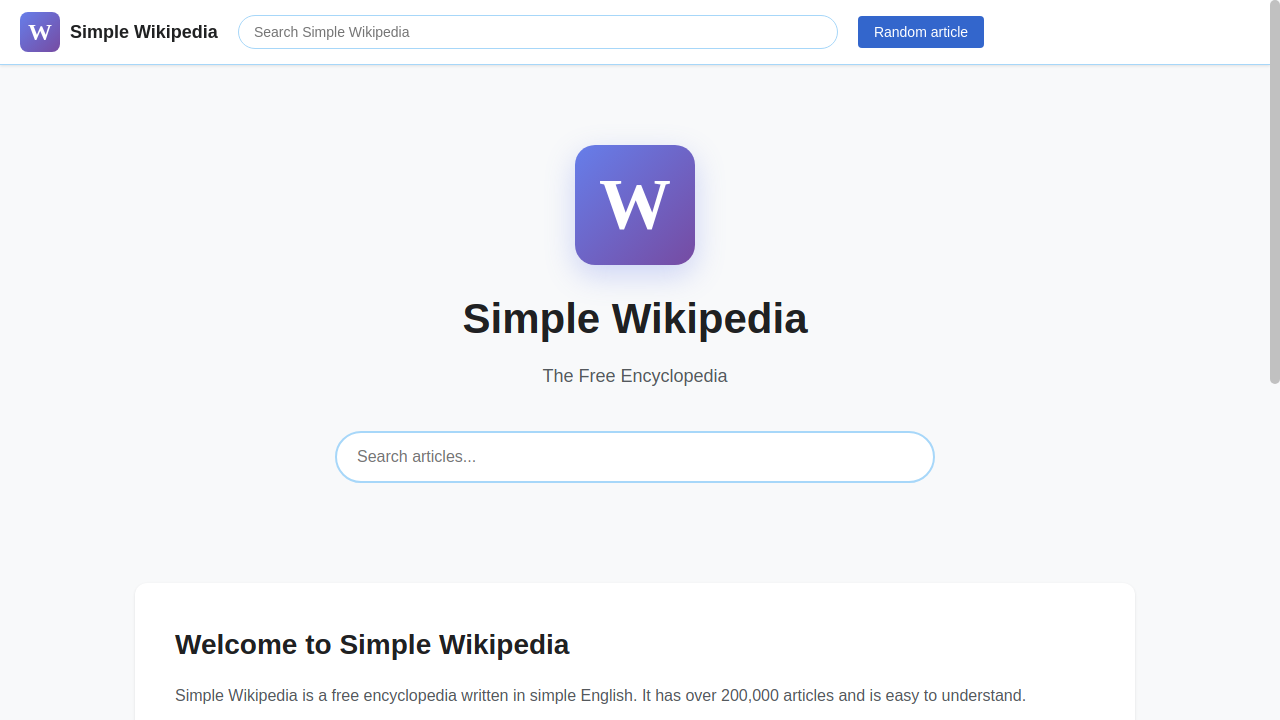}}
}
% [REFINE: story] This is the Wiki figure; "News UI" was a copy-paste slip.
% ORIG: \caption{(a)-(f) Snapshots of different versions of Wiki homepages, resembling different News UI across the eras.}
\caption{(a)-(f) Snapshots of different versions of Wiki homepages, {resembling different Wikipedia UI across the eras.}}
\label{fig:wiki_UI_all}
\end{figure}

\begin{figure}[htbp]
  \centering
\subfigure[Snapshot of the News v$_1$ homepage, resembling the BBC News UI from $1998-2001$. \label{fig:news_v1}]{%
    \fbox{\includegraphics[width=0.99\linewidth]{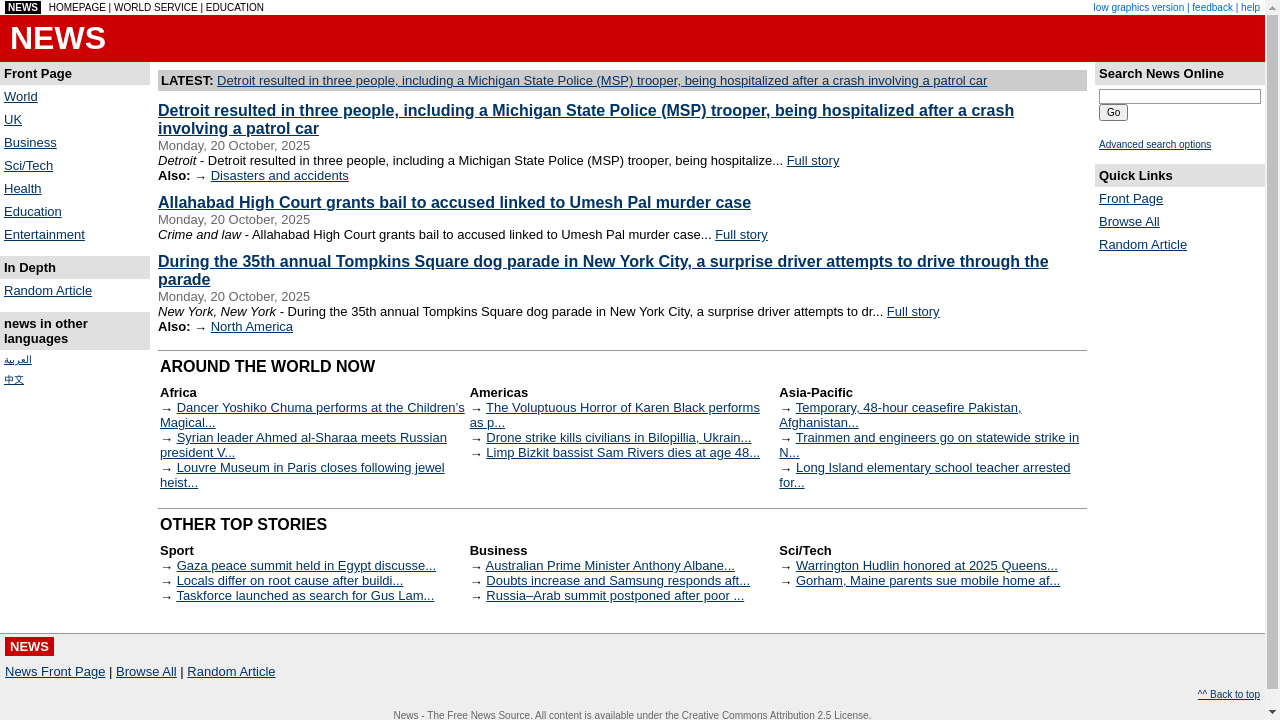}}
}\vspace{1em}
\subfigure[Snapshot of the News v$_2$ homepage, resembling the BBC News UI from $2002-07$.\label{fig:news_v2}]{%
    \fbox{\includegraphics[width=0.99\linewidth]{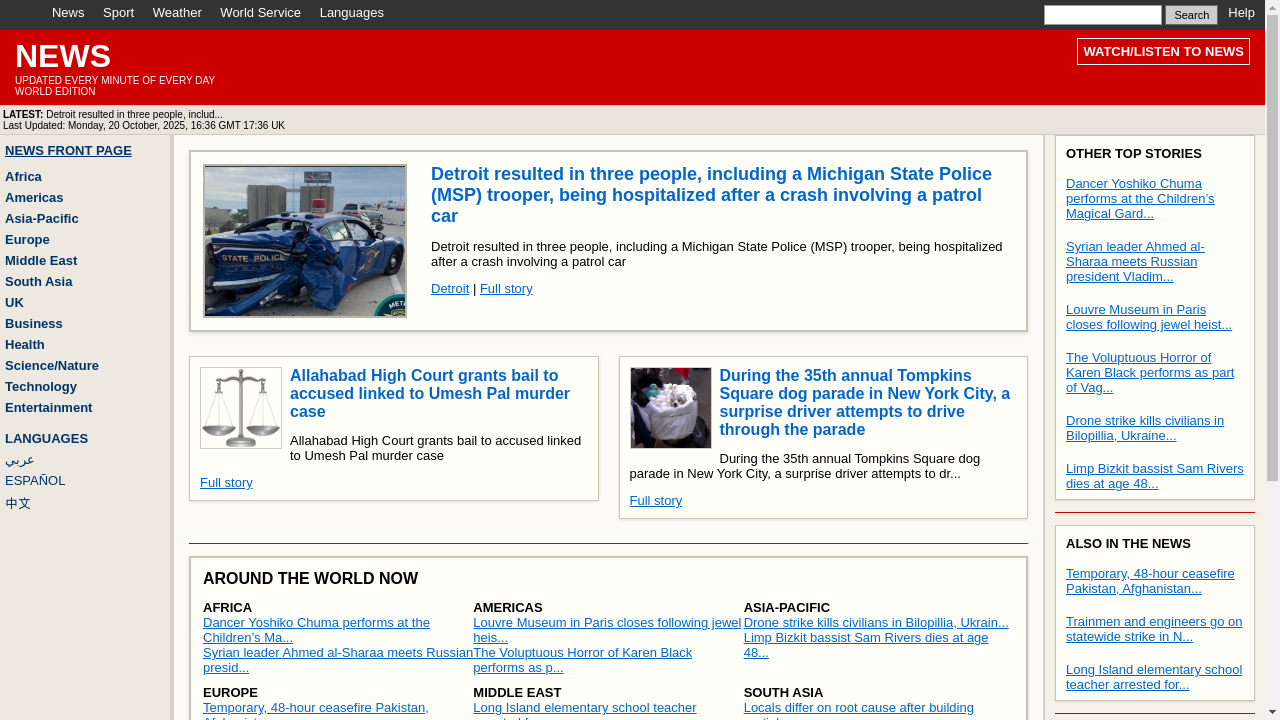}}
}
\end{figure}

\begin{figure}[htbp]
  \centering
\subfigure[Snapshot of the News v$_3$ homepage, resembling the BBC News UI from $2008-2015$. \label{fig:news_v3}]{%
    \fbox{\includegraphics[width=0.99\linewidth]{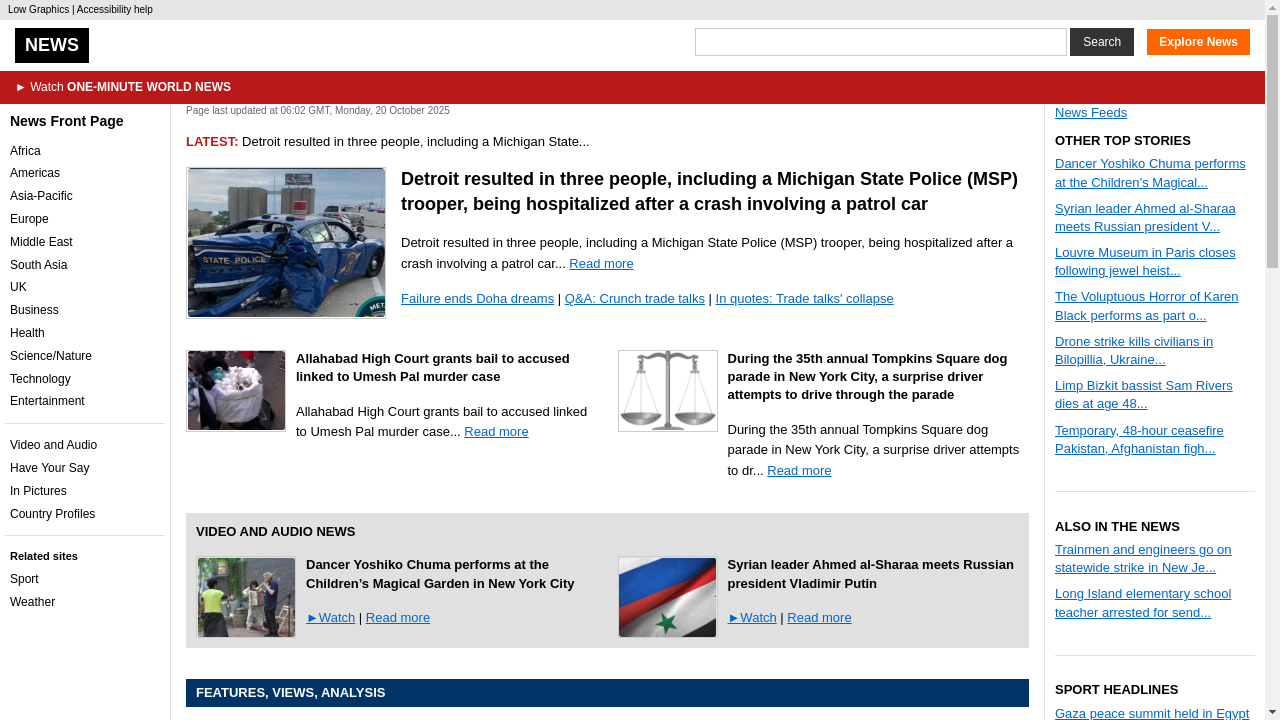}}
}\vspace{1em}
\subfigure[Snapshot of the News v$_4$ homepage, resembling the BBC News UI from $2016-22$.\label{fig:news_v4}]{%
    \fbox{\includegraphics[width=0.99\linewidth]{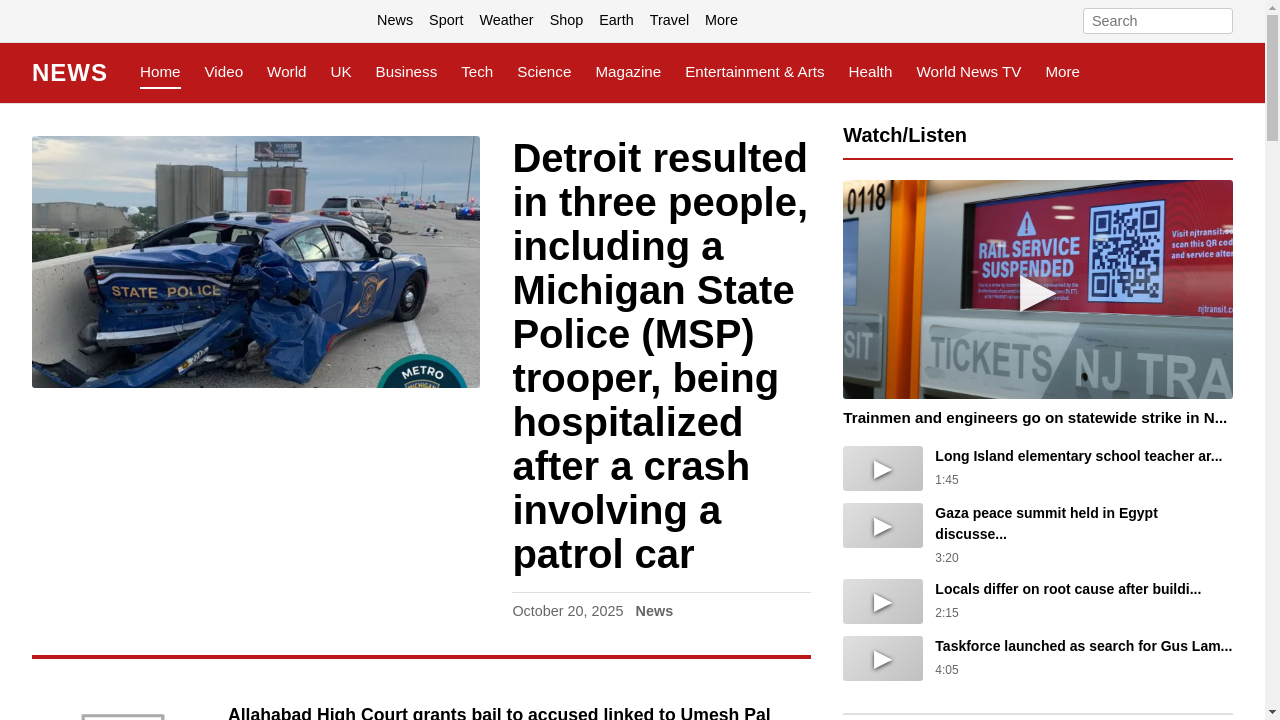}}
}
\end{figure}

\begin{figure}[htbp]
  \centering
\subfigure[Snapshot of the News v$_5$ homepage, resembling the BBC News UI from $2023-2025$. \label{fig:news_v5}]{%
    \fbox{\includegraphics[width=0.99\linewidth]{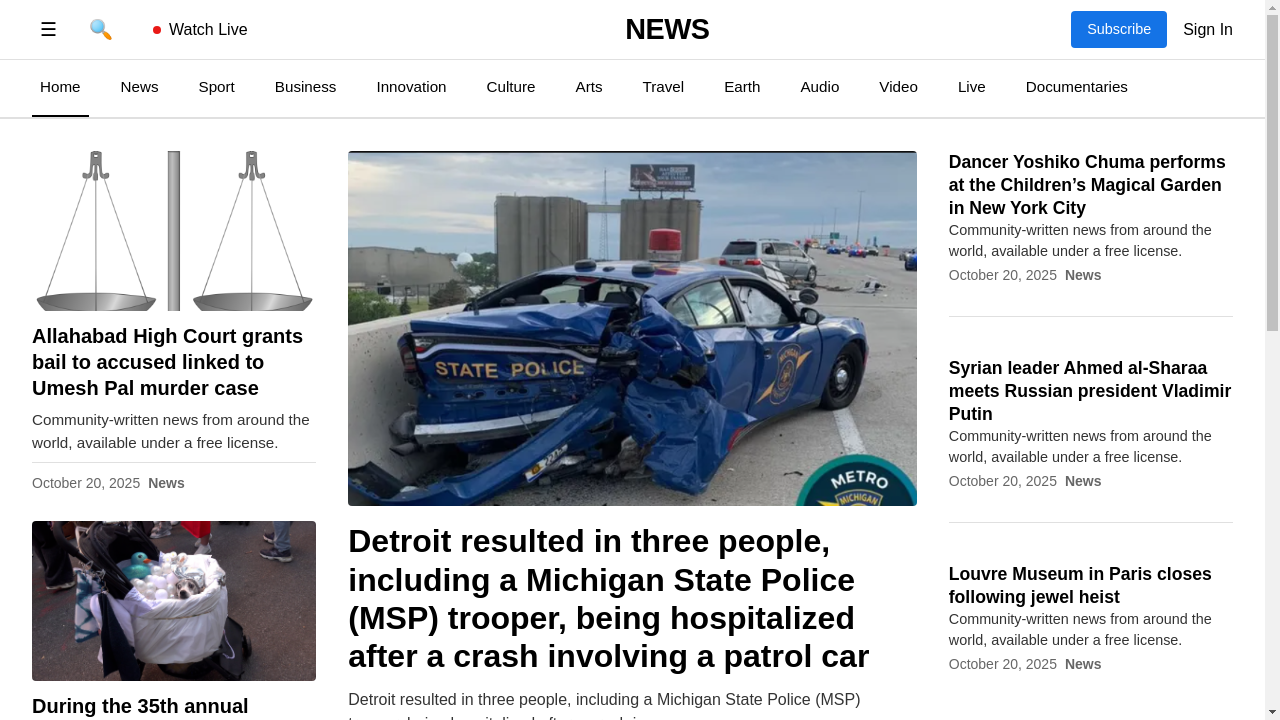}}
}\vspace{1em}
\subfigure[Snapshot of the News v$_6$ homepage, resembling a minimal News UI.\label{fig:news_v6}]{%
    \fbox{\includegraphics[width=0.99\linewidth]{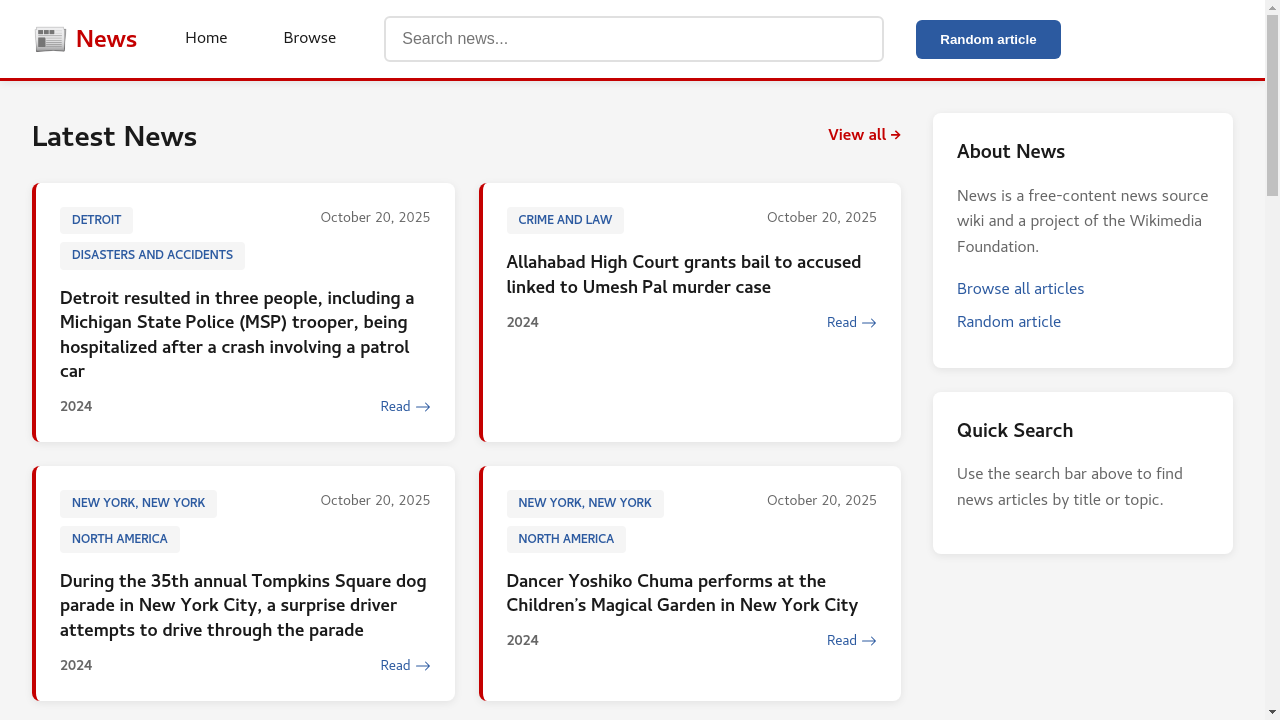}}
}
\caption{(a)-(f) Snapshots of different versions of News homepages, resembling different News UI across the eras.}
\label{fig:news_UI_all}
\end{figure}

\begin{figure}[htbp]
  \centering
\subfigure[Snapshot of the Shop v$_1$ homepage, resembling the Amazon UI from $1999-2004$. \label{fig:shop_v1}]{%
    \fbox{\includegraphics[width=0.99\linewidth]{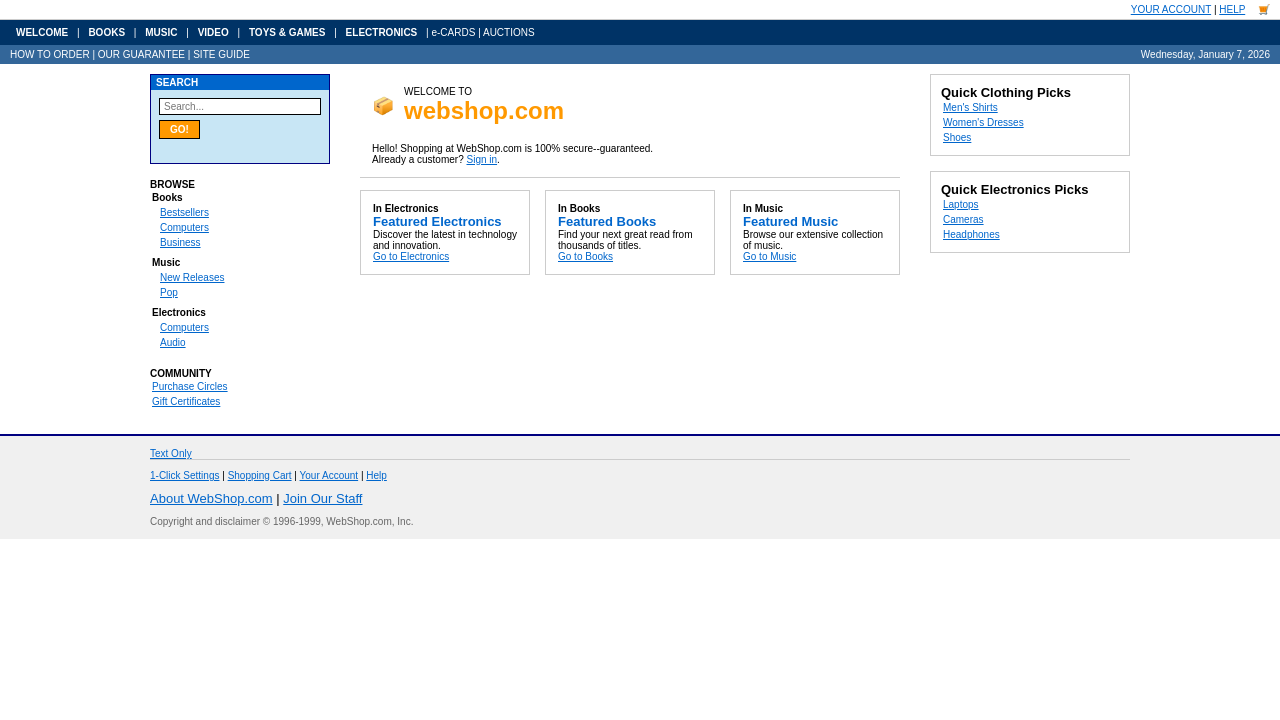}}
}\vspace{1em}
\subfigure[Snapshot of the Shop v$_2$ homepage, resembling the Amazon UI from $2005-11$.\label{fig:shop_v2}]{%
    \fbox{\includegraphics[width=0.99\linewidth]{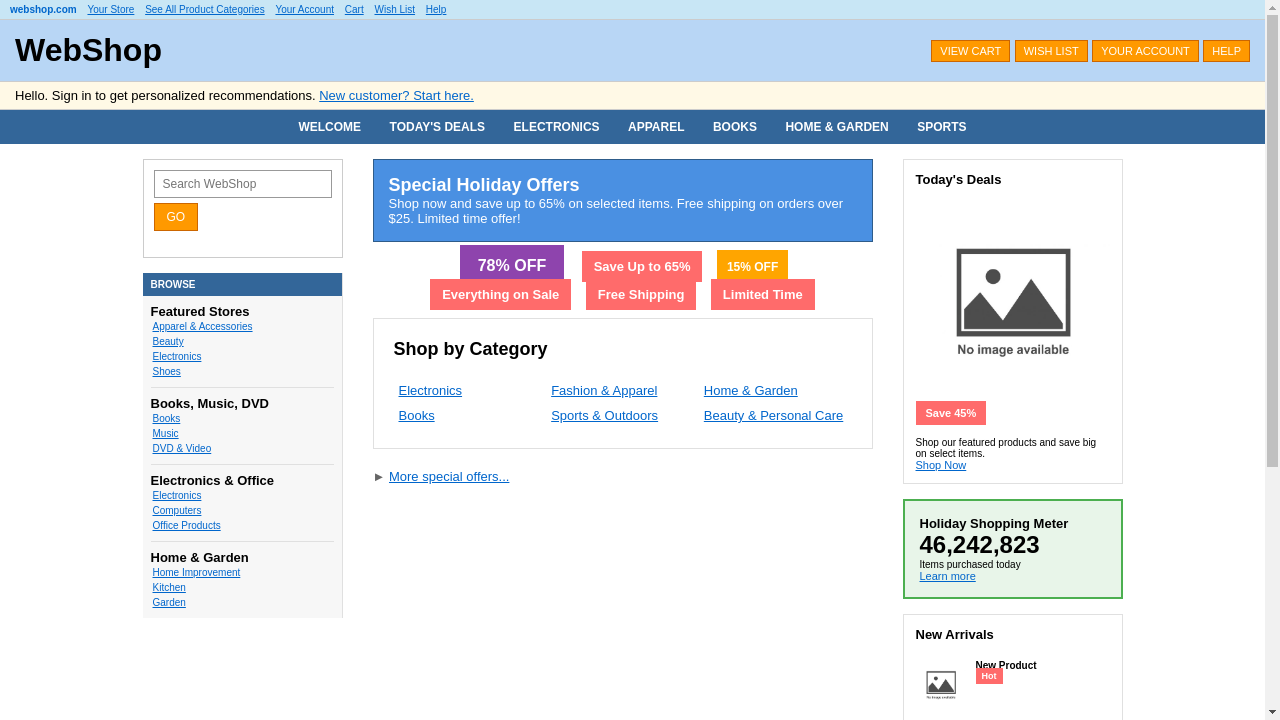}}
}
\end{figure}

\begin{figure}[htbp]
  \centering
\subfigure[Snapshot of the Shop v$_3$ homepage, resembling the Amazon UI from $2012-2014$. \label{fig:shop_v3}]{%
    \fbox{\includegraphics[width=0.99\linewidth]{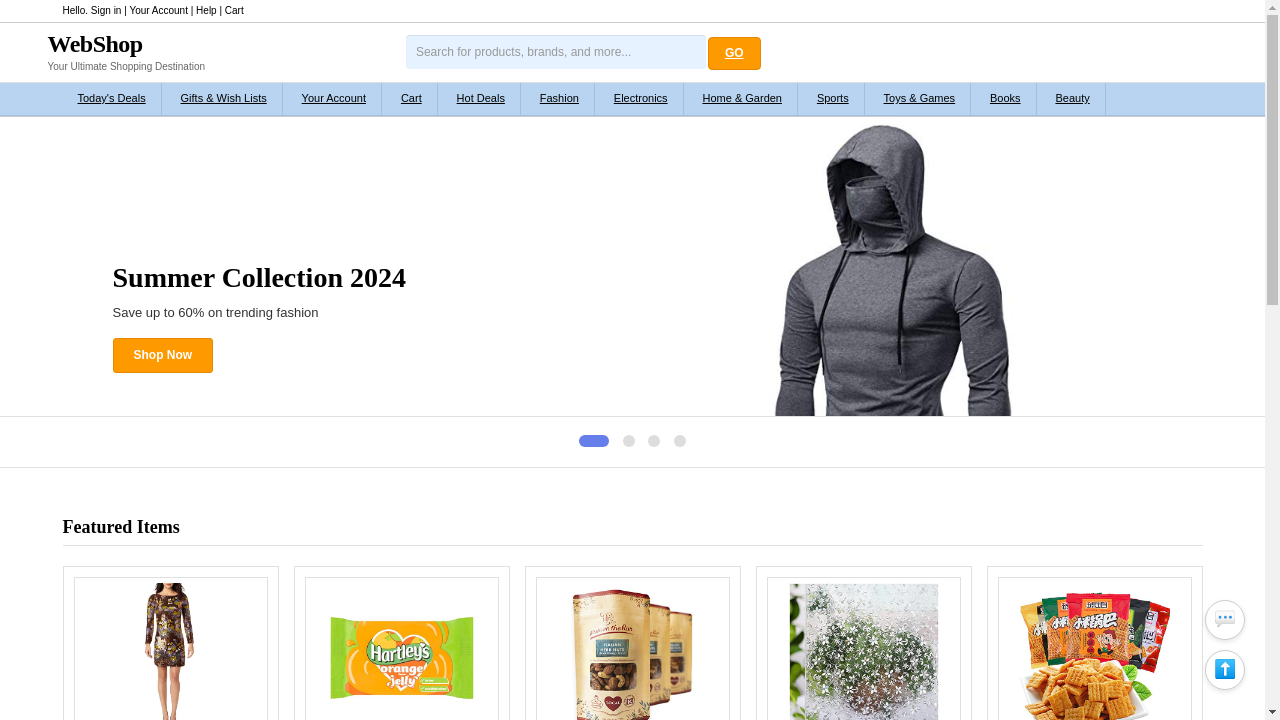}}
}\vspace{1em}
\subfigure[Snapshot of the Shop v$_4$ homepage, resembling the Amazon UI from $2015-25$.\label{fig:shop_v4}]{%
    \fbox{\includegraphics[width=0.99\linewidth]{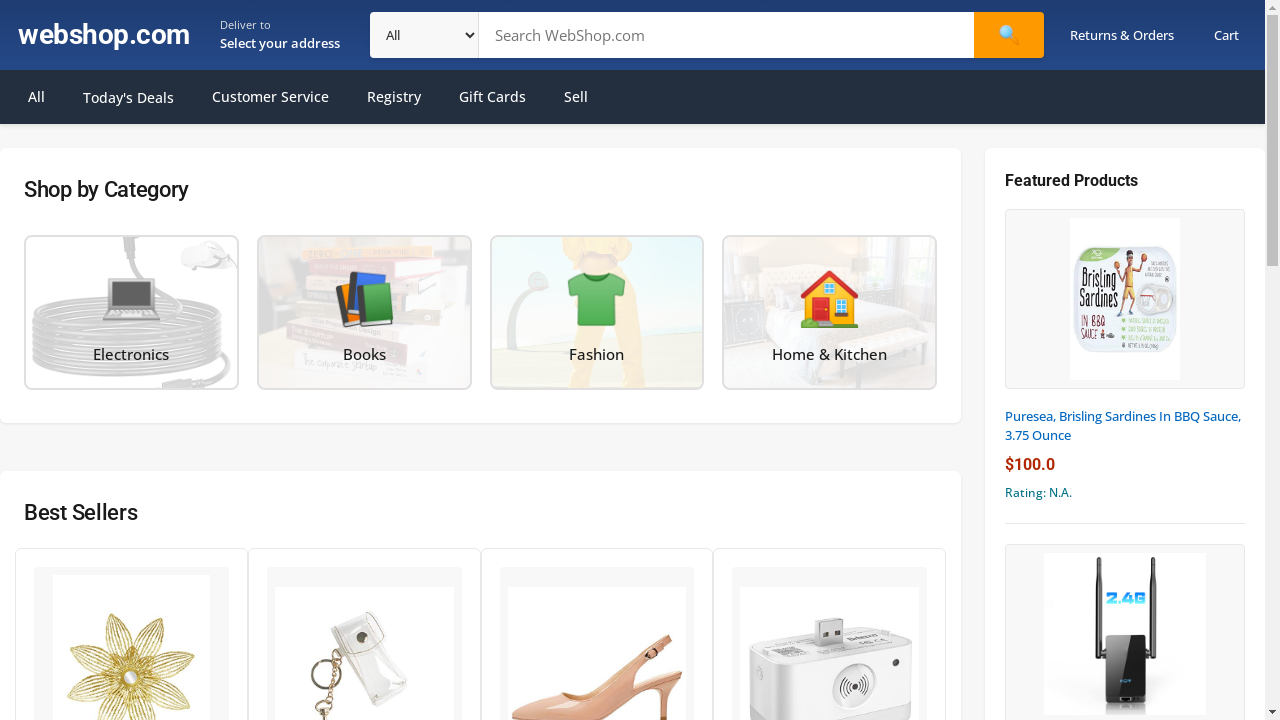}}
}
\end{figure}

\begin{figure}[htbp]
  \centering
\subfigure[Snapshot of the Shop v$_5$ homepage, resembling the Temu UI from $2025$. \label{fig:shop_v5}]{%
    \fbox{\includegraphics[width=0.99\linewidth]{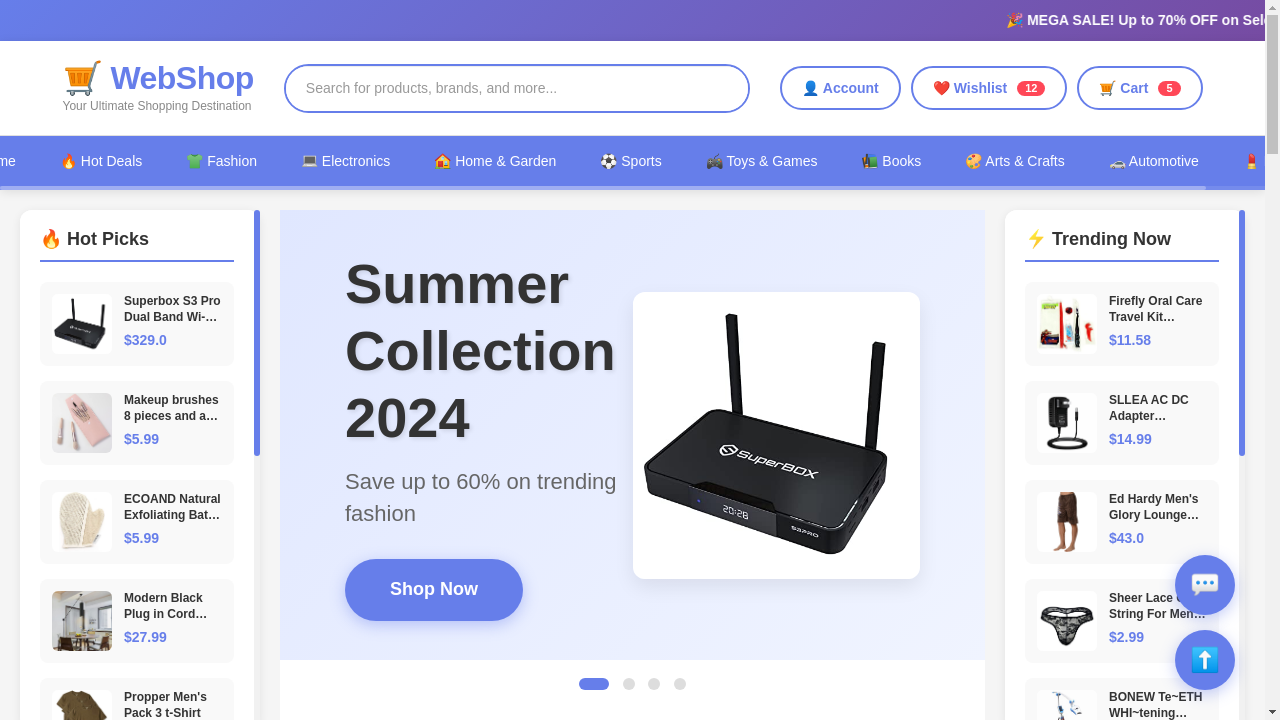}}
}\vspace{1em}
\subfigure[Snapshot of the Shop v$_6$ homepage, resembling the original WebShop UI.\label{fig:shop_v6}]{%
    \fbox{\includegraphics[width=0.99\linewidth]{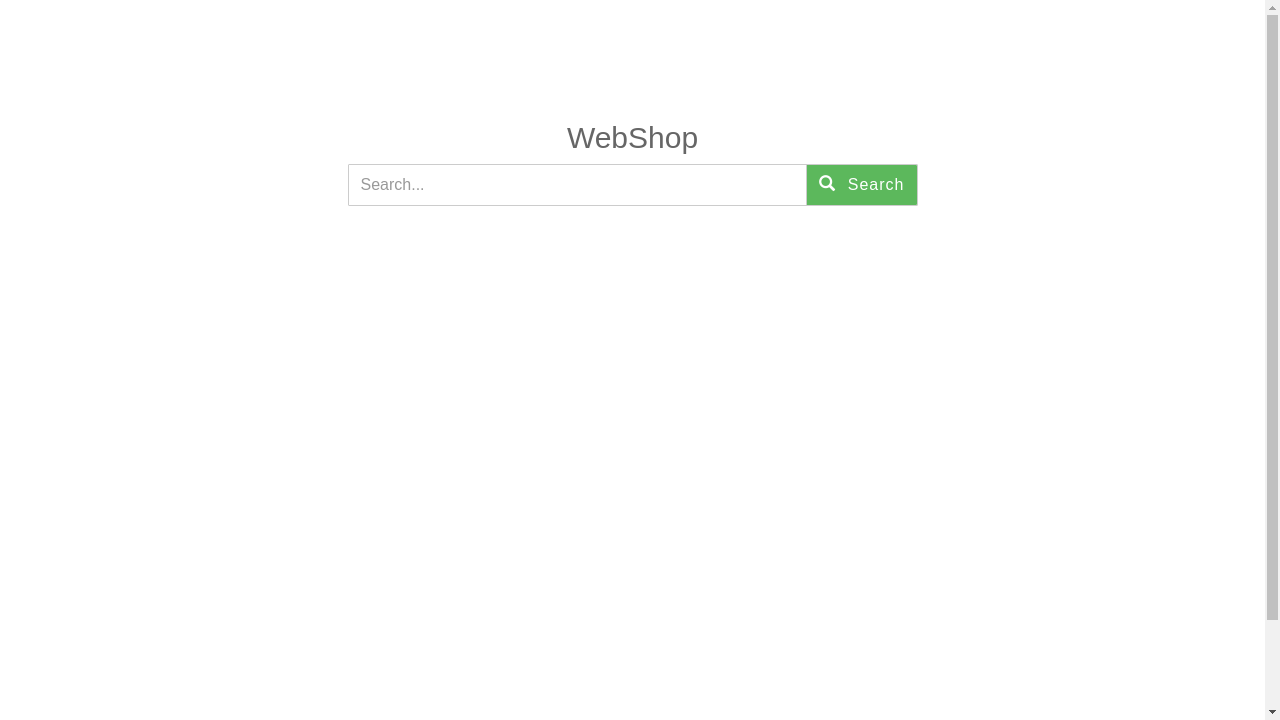}}
}
\caption{(a)-(f) Snapshots of different versions of Shop homepages, resembling different e-commerce UI across the eras.}
\label{fig:shop_UI_all}
\end{figure}

\newpage
\section*{NeurIPS Paper Checklist}

\begin{enumerate}

\item {\bf Claims}
    \item[] Question: Do the main claims made in the abstract and introduction accurately reflect the paper's contributions and scope?
    \item[] Answer: \answerYes{} % Replace by \answerYes{}, \answerNo{}, or \answerNA{}.
    \item[] Justification: The abstract and introduction accurately described the core contribution of our work, which is the \timewarp\ benchmark, the empirical findings, the trajectory collection method \timetraj, and our variant of behavior cloning \timewarp-\textsc{BC}.

\item {\bf Limitations}
    \item[] Question: Does the paper discuss the limitations of the work performed by the authors?
    \item[] Answer: \answerYes{} % Replace by \answerYes{}, \answerNo{}, or \answerNA{}.
    \item[] Justification: The limitations have been discussed in detail in the Appendix Section \ref{sec:limitations}.
    \item[] Guidelines:
    \begin{itemize}
        \item The answer \answerNA{} means that the paper has no limitation while the answer \answerNo{} means that the paper has limitations, but those are not discussed in the paper. 
        \item The authors are encouraged to create a separate ``Limitations'' section in their paper.
        \item The paper should point out any strong assumptions and how robust the results are to violations of these assumptions (e.g., independence assumptions, noiseless settings, model well-specification, asymptotic approximations only holding locally). The authors should reflect on how these assumptions might be violated in practice and what the implications would be.
        \item The authors should reflect on the scope of the claims made, e.g., if the approach was only tested on a few datasets or with a few runs. In general, empirical results often depend on implicit assumptions, which should be articulated.
        \item The authors should reflect on the factors that influence the performance of the approach. For example, a facial recognition algorithm may perform poorly when image resolution is low or images are taken in low lighting. Or a speech-to-text system might not be used reliably to provide closed captions for online lectures because it fails to handle technical jargon.
        \item The authors should discuss the computational efficiency of the proposed algorithms and how they scale with dataset size.
        \item If applicable, the authors should discuss possible limitations of their approach to address problems of privacy and fairness.
        \item While the authors might fear that complete honesty about limitations might be used by reviewers as grounds for rejection, a worse outcome might be that reviewers discover limitations that aren't acknowledged in the paper. The authors should use their best judgment and recognize that individual actions in favor of transparency play an important role in developing norms that preserve the integrity of the community. Reviewers will be specifically instructed not to penalize honesty concerning limitations.
    \end{itemize}

\item {\bf Theory assumptions and proofs}
    \item[] Question: For each theoretical result, does the paper provide the full set of assumptions and a complete (and correct) proof?
    \item[] Answer: \answerYes{} % Replace by \answerYes{}, \answerNo{}, or \answerNA{}.
    \item[] Justification: The full set of assumptions and a complete (and correct) proof is provided in the Appendix \ref{sec:theory}.

    \item {\bf Experimental result reproducibility}
    \item[] Question: Does the paper fully disclose all the information needed to reproduce the main experimental results of the paper to the extent that it affects the main claims and/or conclusions of the paper (regardless of whether the code and data are provided or not)?
    \item[] Answer: \answerYes{} % Replace by \answerYes{}, \answerNo{}, or \answerNA{}.
    \item[] Justification: We have described in detail: the web environment and evaluation in Section \ref{sec:timewarpBenchmark} and Appendix \ref{sec:EnvironmentBenchmarkDetails}, the methodology in Section \ref{sec:methodology} and Appendix \ref{sec:methodologyDetails}, the experimental details in Section \ref{sec:experimentalSetup} and Appendix \ref{sec:experimentDetailsAppendix}. These details including implementation details (\S\ref{sec:envImplementationDetails}), training setup (\S\ref{sec:trainingDetails}), hyperparameters (Tab. \ref{tab:hyperparams}), and checkpoints (\S\ref{sec:checkpoints}), which we believe would be sufficient to reproduce the main experimental results of the paper.

\item {\bf Open access to data and code}
    \item[] Question: Does the paper provide open access to the data and code, with sufficient instructions to faithfully reproduce the main experimental results, as described in supplemental material?
    \item[] Answer: \answerYes{} % Replace by \answerYes{}, \answerNo{}, or \answerNA{}.
    \item[] Justification: The paper provides open access to both data and code, with documentations, as declared in Appendix \ref{sec:reporducibilityDatasetAccessLicensing}.

\item {\bf Experimental setting/details}
    \item[] Question: Does the paper specify all the training and test details (e.g., data splits, hyperparameters, how they were chosen, type of optimizer) necessary to understand the results?
    \item[] Answer: \answerYes{} % Replace by \answerYes{}, \answerNo{}, or \answerNA{}.
    \item[] Justification: The experimental setting has been discussed in Section \ref{sec:experimentalSetup}, with details in Appendix \ref{sec:experimentDetailsAppendix}.

\item {\bf Experiment statistical significance}
    \item[] Question: Does the paper report error bars suitably and correctly defined or other appropriate information about the statistical significance of the experiments?
    \item[] Answer: \answerYes{} % Replace by \answerYes{}, \answerNo{}, or \answerNA{}.
    \item[] Justification: All experimental results in the paper have been conducted over 3 runs and the mean has been reported in every results table. For conciseness, the standard error is reported and discussed separately in Appendix \ref{subsec:consistencyAndStandardError}.

\item {\bf Experiments compute resources}
    \item[] Question: For each experiment, does the paper provide sufficient information on the computer resources (type of compute workers, memory, time of execution) needed to reproduce the experiments?
    \item[] Answer: \answerYes{} % Replace by \answerYes{}, \answerNo{}, or \answerNA{}.
    \item[] Justification: The computation details to reproduce the experiments have been provided in Appendix \ref{sec:trainingDetails}.
    
\item {\bf Code of ethics}
    \item[] Question: Does the research conducted in the paper conform, in every respect, with the NeurIPS Code of Ethics \url{https://neurips.cc/public/EthicsGuidelines}?
    \item[] Answer: \answerYes{} % Replace by \answerYes{}, \answerNo{}, or \answerNA{}.
    \item[] Justification: The research conducted in the paper conforms, in every respect, with the NeurIPS Code of Ethics.

\item {\bf Broader impacts}
    \item[] Question: Does the paper discuss both potential positive societal impacts and negative societal impacts of the work performed?
    \item[] Answer: \answerYes{} % Replace by \answerYes{}, \answerNo{}, or \answerNA{}.
    \item[] Justification: Both positive and negative social impacts have been discussed in Appendix \ref{sec:broaderImpact}.
    
\item {\bf Safeguards}
    \item[] Question: Does the paper describe safeguards that have been put in place for responsible release of data or models that have a high risk for misuse (e.g., pre-trained language models, image generators, or scraped datasets)?
    \item[] Answer: \answerNA{} % Replace by \answerYes{}, \answerNo{}, or \answerNA{}.
    \item[] Justification: To the best of our knowledge, the paper poses no such risks.

\item {\bf Licenses for existing assets}
    \item[] Question: Are the creators or original owners of assets (e.g., code, data, models), used in the paper, properly credited and are the license and terms of use explicitly mentioned and properly respected?
    \item[] Answer: \answerYes{} % Replace by \answerYes{}, \answerNo{}, or \answerNA{}.
    \item[] Justification: Yes, all code, data and models used are appropriately credited and cited.

\item {\bf New assets}
    \item[] Question: Are new assets introduced in the paper well documented and is the documentation provided alongside the assets?
    \item[] Answer: \answerYes{} % Replace by \answerYes{}, \answerNo{}, or \answerNA{}.
    \item[] Justification: The detailed documentation of the new assets has been provided in the \texttt{README.md} files of our GitHub repository.

\item {\bf Crowdsourcing and research with human subjects}
    \item[] Question: For crowdsourcing experiments and research with human subjects, does the paper include the full text of instructions given to participants and screenshots, if applicable, as well as details about compensation (if any)? 
    \item[] Answer: \answerNA{} % Replace by \answerYes{}, \answerNo{}, or \answerNA{}.
    \item[] Justification: The paper does not involve crowd-sourcing nor research with human subjects.

\item {\bf Institutional review board (IRB) approvals or equivalent for research with human subjects}
    \item[] Question: Does the paper describe potential risks incurred by study participants, whether such risks were disclosed to the subjects, and whether Institutional Review Board (IRB) approvals (or an equivalent approval/review based on the requirements of your country or institution) were obtained?
    \item[] Answer: \answerNA{} % Replace by \answerYes{}, \answerNo{}, or \answerNA{}.
    \item[] Justification: The paper does not involve crowd-sourcing nor research with human subjects.

\item {\bf Declaration of LLM usage}
    \item[] Question: Does the paper describe the usage of LLMs if it is an important, original, or non-standard component of the core methods in this research? Note that if the LLM is used only for writing, editing, or formatting purposes and does \emph{not} impact the core methodology, scientific rigor, or originality of the research, declaration is not required.
    %this research? 
    \item[] Answer: \answerYes{} % Replace by \answerYes{}, \answerNo{}, or \answerNA{}.
    \item[] Justification: Declaration of LLM usage has been reported in Appendix \ref{sec:llmUsageDeclaration}.

\end{enumerate}
% \sectionvspace

% \sectionvspace
% \subsection*{Standardization of Web Agent Training Reproducibility}
% \red{Maybe mention something similiar in the introduction section. A one liner or two.}
% While training our models, we found several web agent training pipelines that have significantly weaker baselines, and while investigating, we found several unoptimized hyperparameters in those baselines. For instance, we found instances where the Llama-3.1 8B model has been behavior cloned with only 1k context, which is significantly lesser than our optimal context for fine-tuning. We believe that our extensive expertimental analysis and ablation will shed light into better hyperparameter optimization while training web agents, essentially making stronger baselines.

\end{document}